%% file: example_paper.tex
\documentclass{article}

\usepackage{microtype}
\usepackage{graphicx}

\usepackage{subcaption}
\usepackage{booktabs} 

\usepackage{hyperref}

\usepackage[accepted]{icml2024}

\usepackage{amsmath}
\usepackage{amssymb}
\usepackage{mathtools}
\usepackage{amsthm}

\usepackage[capitalize,noabbrev]{cleveref}

\theoremstyle{plain}
\newtheorem{theorem}{Theorem}[section]

\newtheorem{lemma}[theorem]{Lemma}

\theoremstyle{definition}
\newtheorem{definition}[theorem]{Definition}
\newtheorem{assumption}[theorem]{Assumption}
\theoremstyle{remark}

\usepackage[textsize=tiny]{todonotes}

\usepackage{cuted}
\usepackage{multirow}
\usepackage{multicol}
\usepackage{stfloats}

\usepackage{tikz}
\usepackage{pgfplots}
\pgfplotsset{compat=1.18} 

\usepackage{algorithm}
\usepackage[lined,boxed,commentsnumbered,algo2e]{algorithm2e}

\icmltitlerunning{Triadic-OCD: Asynchronous Online Change Detection with Provable Robustness, Optimality, and Convergence}

\begin{document}

\twocolumn[
\icmltitle{Triadic-OCD: Asynchronous Online Change Detection with Provable Robustness, Optimality, and Convergence}




\begin{icmlauthorlist}
\icmlauthor{Yancheng Huang}{yyy}
\icmlauthor{Kai Yang}{yyy}
\icmlauthor{Zelin Zhu}{yyy}
\icmlauthor{Leian Chen}{sch,note}
\end{icmlauthorlist}

\icmlaffiliation{yyy}{Tongji University}
\icmlaffiliation{sch}{Columbia University}

\icmlaffiliation{note}{This work was done prior to Leian Chen joining Amazon}

\icmlcorrespondingauthor{Kai Yang}{kaiyang@tongji.edu.cn}

\icmlkeywords{}

\icmlkeywords{Machine Learning, ICML}

\vskip 0.3in
]



\printAffiliationsAndNotice{}  

\begin{abstract}
The primary goal of online change detection (OCD) is to promptly identify changes in the data stream. 
OCD problem find a wide variety of applications in diverse areas, e.g., security detection in smart grids and intrusion detection in communication networks.
Prior research usually assumes precise knowledge of the system parameters.
Nevertheless, this presumption often proves unattainable in practical scenarios due to factors such as estimation errors, system updates, etc.
This paper aims to take the first attempt to develop a triadic-OCD framework with certifiable robustness, provable optimality, and guaranteed convergence. In addition, the proposed triadic-OCD algorithm can be realized in a fully asynchronous distributed manner, easing the necessity of transmitting the data to a single server. This asynchronous mechanism  could also mitigate the straggler issue that faced by traditional synchronous algorithm.   
Moreover, the non-asymptotic convergence property of Triadic-OCD is theoretically analyzed, and its iteration complexity to achieve an $\epsilon$-optimal point is derived.
Extensive experiments have been conducted to elucidate the effectiveness of the proposed method. 
\end{abstract}

\section{Introduction}
Detecting distribution changes as quickly as possible while controlling the false alarm rate is the fundamental objective of OCD. And this problem is frequently encountered in a diverse range of fields such as econometrics, climate modeling, and system security \cite{econometrics,climate,network_security,yang2016deep,dou2019pc,adaptive_cusum,methodin14,QCD_in_cyber_attack,QCD_in_cyber_attack_2}. In contrast to general detection problems, OCD problem in dynamic systems are particularly challenging because changes in the distribution of observed data can be caused by the dynamics of the system itself, even without external changes.
Although extensive research has been conducted on this problem, existing studies still face critical issues. The primary challenge is that the parameters associated with systems are assumed to be perfectly estimated. In practice, however, this is often not the case due to factors such as estimation errors, system updates, etc.
For instance, the investigated OCD problem within dynamic systems subsumes the false data injection attacks (FDIA) detection problem in smart grids and the blockage detection in MIMO systems —both critical issues that have received extensive attention \cite{adaptive_cusum, MIMO-blockage}.
In smart grids, fluctuations in the environment can lead to changes in line admittances, causing inaccuracies in system matrix \cite{methodin14}.
Similarly, in MIMO systems, factors such as estimation error, aging, and quantization often prevent the perfect estimation of the channel matrix \cite{error_1}.
Furthermore, current research focuses on centralized setting, which may incur privacy breaches \cite{privacy} and high communication costs. Additionally, the synchronous distributed approach often encounters straggler issues \cite{ADBO} which can cause significant delays in the detection process.
Therefore, in this paper, we aim to address these issues, and our contributions are summarized as follows:
\begin{itemize}
\item \textbf{Certified Robustness:} As opposed to existing works in the literature, Triadic-OCD provides a high degree of confidence across a broad spectrum of system parameter uncertainties.
This certifiable robustness assures the reliability of change point detection in real-world applications.
\item \textbf{Asynchronous Updating:} In order to address concerns related to potential privacy breaches, elevated communication costs, and straggler issues, we propose a novel \textit{asynchronous} distributed algorithm to effectively address the detection problem in the presence of parameter uncertainties. In addition, the proposed approach goes beyond empirical performance by offering theoretical proofs that establish the optimality of Triadic-OCD under certain conditions. This represents a theoretic advancement in ensuring the algorithm's efficiency and effectiveness. 
\item \textbf{Non-asympotic Convergence Analysis:}
We not only proves that Triadic-OCD is guaranteed to converge, but also  undertake non-asymptotic convergence analysis to establish an upper bound to the iteration complexity of the proposed asynchronous algorithm to attain an $\epsilon$-optimal solution.  
\end{itemize}

\section{Related Work}
\label{Related-work}
There have been numerous studies that address the OCD problem with uncertain pre- and post-change distributions.
\cite{ULR} proposes the Uncertain Likelihood Ratio (ULR) test statistic to tackle the parameters within pre- and post-change distributions, which are completely unknown, or known with limited prior knowledge. On the basis of that, \cite{WULR} develops a more efficient method called the Windowed Uncertain Likelihood Ratio (W-ULR) test.
\cite{Minimax} consider the OCD problem when the pre-change and post-change distributions belong to known uncertainty sets. It provides a condition under which the detection rule based on the least favorable distributions (LFDs) are minimax robust. \cite{Misspecified} relaxes this condition and provides the new performance guarantee of misspecified CUSUM rules. \cite{wasserstein} proposes a non-parametric method based on Wasserstein uncertainty sets. However, there’s a significant distinction between these works and our paper.
The general OCD methods considering uncertain distributions assume that the distribution before and after the change remains constant over time, which is completely inconsistent with our problem due to the system's inherent dynamics. As a result, general OCD methods are not applicable to our problem.

The OCD problem in dynamic system has attracted increasing attention recently.
Both \cite{adaptive_cusum} and \cite{methodin14} proposed CUSUM-type algorithms with prior assumptions on the state variables of the systems. And the method proposed in \cite{adaptive_cusum} is inefficient for large or negative attacking vectors injected into the system. In contrast, our approach makes no assumptions for the state variables of the system and imposes no restrictions on the sign of the attacking vectors. \cite{QCD_in_cyber_attack} proposes a real-time detection method based on residuals and constructs the decision statistic with the Rao test statistic. However, in some cases, the decision statistic cannot be evaluated due to the covariance singularity of the residuals. 
All the methods mentioned above, as well as \cite{ICASSP,zhangjiangfan}, assume that the system parameters can be perfectly determined, which is often impractical in reality. 
In addition, some recent work has studied the OCD in certain dynamic systems in the distributed setting.  
A distributed algorithm based on the Kalman filter is proposed in \cite{GLR}. 
And \cite{methodin14} propose distributed sequential detectors based on the generalized likelihood ratio. However, these methods are synchronously distributed and may suffer from straggler problems, which could incur significant delays during the detection.  
As opposed to all previous methods, this paper proposes an asynchronous distributed algorithm to solve the OCD problem in the presence of parameter uncertainties.

\section{Problem Statement}
In this section, we provide a detailed explanation of the OCD problem with parameter uncertainties.   The system of interest can be expressed as follow,
\begin{equation}
    \mathbf{y}^{(t)} =\bar{f}( \boldsymbol{\theta}^{(t)} \mid \mathbf{H})+\mathbf{n}^{(t)}, \mathbf{H} \in \mathcal{U}
    \label{equ-1}
\end{equation}
where $\mathbf{y}^{(t)} \in \mathbb{R}^{M}$ represents the known observation vector and $\boldsymbol{\theta}^{(t)} \in \mathbb{R}^{N} $ represents the unknown time-varying system states. $\bar{f}$ models the relationship within the system and the matrix $\mathbf{H} \in \mathbb{R}^{M \times N}$ incorporates the system parameters. In contrast to previous work, the system matrix is assumed to belong to an uncertainty set $\mathcal{U}$ instead of being perfectly known, i.e. $\mathbf{H} \in \mathcal{U}$. 
When $t \geq t_a$, the time-varying attack vector $\mathbf{a}^{(t)}$ is injected into the system which alters the distribution of $\mathbf{y}^{(t)}$. Our goal is to detect the injected vector as soon as possible.
Note that the complexity of this problem is formidable, and its practical applications remain somewhat under-explored.
However, in numerous real-world scenarios, including the FDIA detection in smart grids and the blockage detection in MIMO systems, the function $\bar{f}$ degenerates into a linear form. 
In such case, prevailing methodologies typically rely on precise knowledge of the system parameters, which often proves unattainable in practical settings. To address this limitation, we propose an asynchronous  OCD method capable of robust and high-performance detection.

Given that the system states $\boldsymbol{\theta}^{(t)}$ and the injected anomaly vector $\mathbf{a}^{(t)}$ are unknown, we estimate them using their maximum likelihood estimates (MLE) \cite{GCUSUM}. This leads to the generalized CUSUM detector for our problem, which can be written as follows.
\begin{equation}
   T_{G}=\mathrm{min} \{  {J}:\operatorname*{max}_{1\leq j\leq J} \Lambda_{j}^{(J)}\geq \zeta \}.
   \label{GCUSUM}
\end{equation}
where $\zeta$ is the predefined threshold.  $\Lambda_{j}^{(J)}$ is given in (\ref{test statistic}), where $ \mathbf{x}^{(t)}$ represent the component of $\mathbf{a}^{(t)}$ orthogonal to the column space of $\mathbf{H}$. And $\rho_H$ is the upper bound for the absolute value of each component of $\mathbf{x}^{(t)}$.
Based on (\ref{GCUSUM}), let $V_{J}$ represents $ \operatorname*{max}_{1\leq j\leq J} \Lambda_{j}^{(J)}$, we can obtain that, 
\begin{figure*}
\begin{equation}
\Lambda_{j}^{(J)} \triangleq  \sup\limits_{\mathbf{H} \in \mathcal{U}}\ln \frac{\sup \limits_{\boldsymbol{\theta}^{(t)}, \mathbf{a}^{(t)}: -\rho_{U}\mathbf{1} \leq \mathbf{x}^{(t)} \leq \rho_{U}\mathbf{1}, 
\mathbf{H}\mathbf{x}^{(t)} =\mathbf{0}    } \prod_{t=1}^{j-1} f_{p}\left(\mathbf{y}^{(t)} \mid \boldsymbol{\theta}^{(t)},\mathbf{H} \right) \prod_{t=j}^{J} f_{q}\left(\mathbf{y}^{(t)} \mid \boldsymbol{\theta}^{(t)}, \mathbf{a}^{(t)},\mathbf{H}\right)}
{\sup \limits_{\boldsymbol{\theta}^{(t)}} 
\prod_{t=1}^{J} f_{p}\left(\mathbf{y}^{(t)} \mid \boldsymbol{\theta}^{(t)},\mathbf{H}\right)}.
\label{test statistic}
\end{equation}
\end{figure*}
\begin{equation}
\begin{aligned}
V_{J}   
      \triangleq &  \operatorname*{max}_{1\leq j\leq J} \Lambda_{j}^{(J)}  = \operatorname*{max}_{1 \leq j \leq J} \sum \nolimits_{t=j}^{J} \frac{v_{t}}{2\sigma_n^2}  \\
      = & \max \left\{V_{J-1}, 0\right\}+\frac{v_{t}}{2\sigma_n^2}, \text { with } V_{0} = 0, 
\end{aligned}
\label{V_K}
\end{equation}
where the value of $v_{t}$ can be obtained by solving problem (\ref{v_t}). The detailed derivation can be found in Appendix \ref{Ad_A}.
It can be seen from (\ref{GCUSUM}) and (\ref{V_K}) that the change is declared when $V_J$ surpasses $\zeta$. And $V_J$ can be calculated in a recursive way, with the primary challenge being to obtain the value of $v_t$.
\begin{equation}
\begin{aligned}
v_t\triangleq  & \sup\limits_{\mathbf{H} \in \mathcal{U}}  \sup_{ -\rho_{U}\mathbf{1} \leq \mathbf{x}^{(t)} \leq \rho_{U}\mathbf{1}, 
\mathbf{H}^{T} \mathbf{x}^{(t)} = \mathbf{0}}
\bigg\{ \\ & 
\bigg[2(\mathbf{x}^{(t)})^{T} \mathbf{y}^{(t)}  - \| \mathbf{x}^{(t)}\|_{2}^{2} \bigg]\bigg\}.
\end{aligned}
\label{v_t}
\end{equation} 
Note that the inherent uncertainties associated with $\mathbf{H}$ makes the constraint $\mathbf{H}^{T} \mathbf{x}^{(t)} = \mathbf{0}$  in (\ref{v_t}) unattainable, which greatly exacerbates the complexity of the problem.

\section{Nested Optimization}
In this section, we elaborate on the uncertainty of the system matrix and reformulate the problem (\ref{v_t}) under the distributed setup as a nested optimization problem.

To address the challenges posed by the inherent uncertainty of $\mathbf{H}$, we adopt the constraint-wise uncertainty model, which decouples the uncertainties between different rows in the matrix. This versatile approach is applicable to a wide range of practical problems \cite{yang2014distributed,priceofrobustness}.
We first represent $\mathbf{h}_{i}$ as the $i$-th column of $\mathbf{H}$, each $\mathbf{h}_{i}$ lies in the uncertainty set $\mathcal{U}_i$.  Denote the $i$-th nominal column as $\mathbf{\bar{h}}_{i}$ .
Subsequently, we relax the constraint $\mathbf{H}^{T} \mathbf{x}^{(t)} = \mathbf{0}$ with $ \mathbf{H}^{T} \in \mathcal{U}$ as follows,
\begin{equation}
\label{4-1}
\begin{aligned}
&  \mathbf{\bar{h}}_{i}^{T}\mathbf{x}^{(t)} + p_i(\mathbf{x}^{(t)}) \leq \delta_{i}, ~1 \leq i \leq 2N,   \\
\end{aligned}
\end{equation}
where 
\begin{equation}
\label{4-2}
\begin{aligned}
& p_i(\mathbf{x}^{(t)}) =\max_{\mathbf{h}_{i}\in \mathcal{U}_i}(\mathbf{h}_i-\mathbf{\bar{h}}_i)^T\mathbf{x}^{(t)}, 1 \leq i \leq N  \\
& p_i(\mathbf{x}^{(t)}) =\max_{-\mathbf{h}_{i}\in \mathcal{U}_i}(\mathbf{h}_i-{\mathbf{\bar{h}}}_i)^T\mathbf{x}^{(t)}, N+1 \leq i \leq 2N
\end{aligned}
\end{equation}
are the protection functions.  
For $i = N+1,\cdots,2N$, 
$\bar{\mathbf{h}}_{i}=\bar{\mathbf{h}}_{i-N}$, $\mathcal{U}_i = \mathcal{U}_{i-N}$, and $\delta_{i} = \delta_{i-N}$.

The choice of different forms for $\mathcal{U}_i$  will yield distinct protection functions $p_i(\mathbf{x}^{(t)})$, consequently affecting the trade-off between robustness and detection performance.
In this paper, we examine a highly versatile uncertain set. Specifically, the uncertainty set $\mathcal{U}_i$ corresponding to each $\mathbf{h}_{i}$ is assumed to be characterized by $U_{i}$ differentiable functions, that is, 
\begin{equation}
\label{4-3}
c_{iu}(\mathbf{h}_{i}) \leq 0, 1 \leq u \leq U_{i}.
\end{equation}
Our framework offers the flexibility to customize the uncertainty set chosen in practical applications based on specific requirements, addressing diverse needs related to complexity and performance.

Now we consider the model (\ref{equ-1}) in a distributed setting, where the system comprises numerous sub-regions (workers) geographically dispersed across a wide area. Each sub-region collects and manages local observation data before communicating with the master node to facilitate collaborative detection.
Suppose there are $L$ sub-regions in the system, we rewrite (\ref{equ-1}) for each sub-region as follows,
\begin{equation}
    \mathbf{y}^{(t)}_l =\mathbf{H}_l \boldsymbol{\theta}^{(t)}_l+\mathbf{n}^{(t)}_l,
    \label{subregion-model}
\end{equation}
where we utilize the subscript $l$ to denote the local components within the $l$-th sub-region.
Given that the state vectors of neighboring sub-regions may share certain parameters, the variable $\boldsymbol{\theta}_l$ for different $l$ may therefore partially overlap.
For clarity, we denote the $\mathbf{x}^{(t)}$ within the $l$-th sub-region as $\boldsymbol{\mu}^{(t)}_l$.
Therefore, we have $\boldsymbol{\mu}^{(t)}_l = \mathbf{B}_{l}\mathbf{x}^{(t)}, 1 \leq l \leq L $
, where $\mathbf{B}_{l}$ is the matrix projecting global attacking vector to the local attacking vector on worker $l$.
Below we omit the time superscript $(t)$ for notational simplicity.
On the basis of (\ref{4-1})-(\ref{subregion-model}), the problem (\ref{v_t}) can be formulated as 
\begin{equation}
\begin{aligned}
\min \quad &   \sum \nolimits_{l=1}^{L} \{ ~ \|\boldsymbol{\mu}_{l}\|_{2}^{2} -   2{\mathbf{y}_{l}}^{T} \boldsymbol{\mu}_{l} ~  \}   \\
\text { s.t. } \quad
&-\rho_{U}\mathbf{1} \leq \boldsymbol{\mu}_{l} \leq \rho_{U}\mathbf{1}, ~~~~~~1\leq l \leq L   \\
&  \boldsymbol{\mu}_{l} = \mathbf{B}_{i}\mathbf{x}, ~~~~~~1\leq l \leq L   \\
&\bar{\mathbf{h}}_{i}^{T}\mathbf{x} + p_i(\mathbf{x}) \leq \delta_{i}, ~1 \leq i \leq 2N\\
\text {var.}  \quad &  \{\boldsymbol{\mu}_{l}\}, \{\mathbf{h}_{i}\},\mathbf{x},
\label{distributed_1}
\end{aligned}
\end{equation}

\section{Asynchronous Distributed Method}
In this section, we provide a detailed explanation of the proposed algorithm named Triadic-OCD, capable of solving our problem in an asynchronous distributed manner while ensuring robustness, and optimality.
Triadic-OCD employs a set of cutting planes to approximate the feasible region constrained by (\ref{4-1}), leveraging their flexibility in adaptation to problem complexity and efficiency in exploration \cite{cut_1,cut_2,yang2014distributed}.
Subsequently, variables undergo asynchronous distributed updates. These two steps iteratively alternate in the proposed algorithm. The cutting planes continuously refine the approximation, while variable updates on the basis of the newly updated cutting plane sets.

Given that computing the exact value of $\{ \mathbf{h}_{i} \}$ for updating outer-level variables at each iteration is unnecessary and inefficient in terms of time and memory usage \cite{estimate_h_1,estimate_h_2}, we opt to employ the estimates of $\{ \mathbf{h}_{i} \}$ instead of their precise values during algorithm execution.
Based on existing methods \cite{K_round_1,K_round_2,ADBO}, Triadic-OCD uses $D$ iterations of gradient descent to approximate the optimal value of $\{ \mathbf{h}_{i} \}$ for the inner-level optimization problem (\ref{4-2}).
We first present the augmented Lagrangian function for each protection function $p_i, i = 1,\cdots,N$ as follows. For $p_i,i = N+1,\cdots,2N$, the results are similar.
\begin{equation}
\begin{aligned}
& L_{\sigma}\left(\mathbf{x}, \mathbf{h}_{i}, \{\phi_{iu}\} \right)=  (\mathbf{h}_i-\mathbf{\bar{h}}_i)^T\mathbf{x}   \\
& +  \frac{\sigma}{2} \sum \nolimits_{u = 1}^{U_{i}}\left(\max \left\{\frac{\phi_{iu}}{\sigma}+c_{iu}(\mathbf{h}_{i}), 0\right\}^{2}-\frac{\phi_{iu}^{2}}{\sigma^{2}}\right) 
\label{ALM}
\end{aligned}
\end{equation}
where $\phi_{iu} \in \mathbb{R}$ is the dual variable, and $\sigma > 0 $ is the penalty parameter. In the $(d + 1)^{th}$ iteration, 
the variables are updated as follows,   
\begin{align}
\mathbf{h}_{i,d+1} =  & ~~\mathbf{h}_{i,d} -\eta_{\mathbf{h}}\nabla_{\mathbf{h}_{i}}L_{\sigma}(\mathbf{x}, \mathbf{h}_{i,d}, \{\phi_{iu,d}\} ),
\nonumber \\
\phi_{iu,d+1} = &~~ ( \phi_{iu,d} + \eta_{\phi}c_{iu}(\mathbf{h}_{i,d+1}) )^{+}, 
\end{align}
where $ (\cdot)^{+} = \max\{0, \cdot \}$, and $\eta_{\mathbf{h}}$, $\eta_{\phi}$ are step-sizes. 
We use the results after $D$ iterations of gradient descent to obtain the estimate of $\mathbf{h}_{i}$, i.e., 
\begin{equation}
\mathbf{h}_{i,D} =  \mathbf{h}_{i,0} - \sum\limits_{d = 0}^{D -1}\eta_{\mathbf{h}}\nabla_{\mathbf{h}_{i}}L_{\sigma}(\mathbf{x}, \mathbf{h}_{i,d}, \{\phi_{iu,d}\} ) .
\label{K_round}
\end{equation}

Based on the estimated solution, we define,
\begin{equation}
{g_i(\mathbf{x})} = 
(\mathbf{h}_{i,0} - \sum_{d = 0}^{D -1}\eta_{\mathbf{h}}\nabla_{\mathbf{h}_{i}}L_{\sigma}(\mathbf{x}, \mathbf{h}_{i,d}, \{\phi_{iu,d}\}) )^{T} 
\mathbf{x}.
\label{function_g}
\end{equation}
As such, the problem (\ref{distributed_1}) can be written as  
\begin{equation}
\begin{aligned}
\min \quad &   \sum \nolimits_{l=1}^{L} \{  ~\|\boldsymbol{\mu}_{l}\|_{2}^{2} -   2{\mathbf{y}_{l}}^{T} \boldsymbol{\mu}_{l}  ~ \} \\
\text { s.t. } \quad
&-\rho_{U}\mathbf{1} \leq \boldsymbol{\mu}_{l} 
  \leq  \rho_{U}\mathbf{1}, ~~~~~~1 \leq  l \leq L \\
&  \boldsymbol{\mu}_{l} = \mathbf{B}_{i}\mathbf{x}, ~~~~~~1\leq l \leq L \\
&   {g_i(\mathbf{x})}  \leq  {\delta}_i, ~~~~~ 1 \leq i \leq 2N \\
\text {var.}  \quad &  \{\boldsymbol{\mu}_{l}\}, \mathbf{x}.
\end{aligned}
\label{distributed_2}
\end{equation}
Since the protection functions are convex \cite{yang2014distributed},
we  employ a set of cutting planes to approximate the feasible region defined by the constraint $g_i(\mathbf{x}) \leq  {\delta}_i, 1 \leq i \leq 2N $. 
In the $(k+1)^{th}$ iteration, let $\mathcal{P}^{k}$ denote the feasible region respect to the set of cutting planes, i.e, 
\begin{equation}
\mathcal{P}^{k}=\{\mathbf{b}_{s}^{T}\mathbf{x}+\kappa_{s} \leq 0,s=1,\cdot\cdot\cdot,|\mathcal{P}^{k}|\},
\label{polytope}
\end{equation}
where $|\mathcal{P}^{k}|$ denotes the number of cutting planes in the $(k+1)^{th}$ iteration.  And $\mathbf{b}_{s} \in \mathbb{R}^{M}$ and $\kappa_{s} \in \mathbb{R}$ represents the parameters in the $s^{th}$ cutting plane. 
Therefore, in the $(k+1)^{th}$ iteration, the approximation problem can be formulated as follows  :
\begin{equation}
\begin{aligned}
\min \quad &   \sum \nolimits_{l=1}^{L} \{  ~ \|\boldsymbol{\mu}_{l}\|_{2}^{2} -   2{\mathbf{y}_{l}}^{T} \boldsymbol{\mu}_{l}   ~\} \\
\text { s.t. } \quad
& -\boldsymbol{\mu}_{l} - \rho_{U}\mathbf{1} + \mathbf{r}_{l} \circ \mathbf{r}_{l} =  \mathbf{0},~~~~1\leq l \leq L  \\ 
& \boldsymbol{\mu}_{l} - \rho_{U}\mathbf{1} + \mathbf{p}_{l} \circ \mathbf{p}_{l} =  \mathbf{0}, ~~~~1\leq l \leq L \\
&  \boldsymbol{\mu}_{l} = \mathbf{B}_{i}\mathbf{x}, ~~~~~~1\leq l \leq L \\
&\mathbf{b}_{s}^{T}\mathbf{x}+\kappa_{s} + q_{s}^2 =  0, ~ s=1,\cdot\cdot\cdot,|\mathcal{P}^{k}| \\
\text {var.}  \quad &  \{\boldsymbol{\mu}_{l}\},\mathbf{x}, \{\mathbf{r}_l\}, \{\mathbf{p}_l\},  \{q_{s}\},
\end{aligned}
\label{distributed_3}
\end{equation}
where $\mathbf{r}_{l}$,$ \mathbf{p}_{l}$,$q_{s}$ are introduced slack variables. $q_{s}\in \mathbb{R}$ and  $\mathbf{r}_{l}$, $ \mathbf{p}_{l}$ have the same dimensions as $\boldsymbol{\mu}_{l}$.
Notation $\circ$ represents the Hadamard product.
As the algorithm iterates, the set of cutting planes will be continuously updated to better approximate the original feasible region.
In Triadic-OCD, we update the set of cutting planes every $w$ iterations to improve the approximation of the original feasible region when $k < K_{1}$. $K_{1}$ and $w$ are the pre-defined positive integers that can be adjusted. 
If $(k+1)\bmod w = 0$, we compute the value of $g_i(\mathbf{x}^{k+1})$ for each $i$ according to (\ref{function_g}). And subsequently, we check if  $g_i(\mathbf{x}^{k+1} )  \leq  {\delta}_i$. 
If $g_i(\mathbf{x}^{k+1} )  >  {\delta}_i$, a new cutting plane will be generated to separate the point from the feasible region defined by $g_i(\mathbf{x})  \leq  {\delta}_i$.  
Following \cite{cut_1}, the generated cutting plane $cp^{k+1}$ can be written as,
\begin{equation}
g_i(\mathbf{x}^{k+1}) + \left( \frac{\partial  g_i(\mathbf{x}^{k+1})}{\partial \mathbf{x}} \right)^{T} ( \mathbf{x} - \mathbf{x}^{k+1}   ) \leq {\delta}_i.
\label{cutting_plane}
\end{equation}

\begin{figure*}[ht]
\begin{align}
\label{mu_update}
& \boldsymbol{\mu}_{l}^{k+1}=
\left\{\begin{array}{l}
{{\boldsymbol{\mu}_{l}^{k}-\eta_{\boldsymbol{\mu}   }\nabla_{\boldsymbol{\mu}_{l}}
\widetilde{L}_{p} (\{\boldsymbol{\mu}^{{\hat{k}}_{l}}_{l}\}, \{ \mathbf{r}_{l}^{{\hat{k}}_{l}}\},\{ \mathbf{p}_{l}^{{\hat{k}}_{l}}\},\{ q_{s}^{{\hat{k}}_{l}}\},\mathbf{x}^{{\hat{k}}_{l}}, \{{\boldsymbol{\lambda}}_{l}^{{\hat{k}}_{l}}\},\{{\boldsymbol{\alpha}}_{l}^{{\hat{k}}_{l}}\},\{{\boldsymbol{\beta}}_{l}^{{\hat{k}}_{l}}\}, \{\gamma_{s}^{{\hat{k}}_{l}}\} ),
 l \in \mathcal{Q}^{k+1}}}\\[1mm]
{{\boldsymbol{\mu}_{l}^{k},
l \notin \mathcal{Q}^{k+1}}}
\end{array}\right.,
\tag{24}
\\
\label{r_update}
& \mathbf{r}_{l}^{k+1}=
\left\{\begin{array}{l}
{{\mathbf{r}_{l}^{k}-\eta_{\mathbf{r}   }\nabla_{\mathbf{r}_{l}}
\widetilde{L}_{p} (\{\boldsymbol{\mu}^{{\hat{k}}_{l}}_{l}\}, \{ \mathbf{r}_{l}^{{\hat{k}}_{l}}\},\{ \mathbf{p}_{l}^{{\hat{k}}_{l}}\},\{ q_{s}^{{\hat{k}}_{l}}\},\mathbf{x}^{{\hat{k}}_{l}}, \{{\boldsymbol{\lambda}}_{l}^{{\hat{k}}_{l}}\},\{{\boldsymbol{\alpha}}_{l}^{{\hat{k}}_{l}}\},\{{\boldsymbol{\beta}}_{l}^{{\hat{k}}_{l}}\}, \{\gamma_{s}^{{\hat{k}}_{l}}\} ),
l \in \mathcal{Q}^{k+1}}}\\[1mm]
{{\mathbf{r}_{l}^{k},
l \notin \mathcal{Q}^{k+1}}}
\end{array}\right.,
\tag{25}
\\
\label{p_update}
& \mathbf{p}_{l}^{k+1}=
\left\{\begin{array}{l}
{{\mathbf{p}_{l}^{k}-\eta_{\mathbf{p}   }\nabla_{\mathbf{p}_{l}}
\widetilde{L}_{p} (\{\boldsymbol{\mu}^{{\hat{k}}_{l}}_{l}\}, \{ \mathbf{r}_{l}^{{\hat{k}}_{l}}\},\{ \mathbf{p}_{l}^{{\hat{k}}_{l}}\}, \{ q_{s}^{{\hat{k}}_{l}}\}, \mathbf{x}^{{\hat{k}}_{l}}, \{{\boldsymbol{\lambda}}_{l}^{{\hat{k}}_{l}}\},\{{\boldsymbol{\alpha}}_{l}^{{\hat{k}}_{l}}\},\{{\boldsymbol{\beta}}_{l}^{{\hat{k}}_{l}}\}, \{\gamma_{s}^{{\hat{k}}_{l}}\} ),
l \in \mathcal{Q}^{k+1}}}\\[1mm]
{{\mathbf{p}_{l}^{k},
l \notin \mathcal{Q}^{k+1}}}
\end{array}\right.,
\tag{26}
\\
& \label{q_update}
q_{s}^{k+1}= {q_{s}^{k}-\eta_{q}\nabla_{q_{s}}
\widetilde{L}_{p} (\{\boldsymbol{\mu}^{k+1}_{l}\}, \{ \mathbf{r}_{l}^{k+1}\},\{ \mathbf{p}_{l}^{k+1}\},\{ q_{s}^{k}\},\mathbf{x}^{k}, \{{\boldsymbol{\lambda}}_{l}^{k}\},\{{\boldsymbol{\alpha}}_{l}^{k}\},\{{\boldsymbol{\beta}}_{l}^{k}\}, \{\gamma_{s}^{k}\} ) },
\tag{27}
\\
& \label{v_update}
\mathbf{x}^{k+1}= {\mathbf{x}^{k}-\eta_{\mathbf{x}   }\nabla_{\mathbf{x}}
\widetilde{L}_{p} (\{\boldsymbol{\mu}^{k+1}_{l}\}, \{ \mathbf{r}_{l}^{k+1}\},\{ \mathbf{p}_{l}^{k+1}\},\{ q_{s}^{k+1}\},\mathbf{x}^{k}, \{{\boldsymbol{\lambda}}_{l}^{k}\},\{{\boldsymbol{\alpha}}_{l}^{k}\},\{{\boldsymbol{\beta}}_{l}^{k}\}, \{\gamma_{s}^{k}\} ) },
\tag{28}
\\ 
& \label{lambda_update}
{\boldsymbol{\lambda}}_{l}^{k+1}=
\left\{\begin{array}{l}
({{{\boldsymbol{\lambda}}_{l}^{k} + \eta_{{\boldsymbol{\lambda}}   }\nabla_{{\boldsymbol{\lambda}}_{l}}
\widetilde{L}_{p} (\{\boldsymbol{\mu}^{k+1}_{l}\}, \{ \mathbf{r}_{l}^{k+1}\},\{ \mathbf{p}_{l}^{k+1}\},\{ q_{s}^{k+1}\},\mathbf{x}^{k+1}, \{{\boldsymbol{\lambda}}_{l}^{k}\},\{{\boldsymbol{\alpha}}_{l}^{k}\},\{{\boldsymbol{\beta}}_{l}^{k}\}, \{\gamma_{s}^{k}\} ))^{+},
l \in \mathcal{Q}^{k+1}}}\\[1mm]
{{{\boldsymbol{\lambda}}_{l}^{k},
l \notin \mathcal{Q}^{k+1}}}
\end{array}\right.,
\tag{29}
\\
& \label{alpha_update}
{\boldsymbol{\alpha}}_{l}^{k+1}=
\left\{\begin{array}{l}
({{{\boldsymbol{\alpha}}_{l}^{k} + \eta_{{\boldsymbol{\alpha}}   }\nabla_{{\boldsymbol{\alpha}}_{l}}
\widetilde{L}_{p} (\{\boldsymbol{\mu}^{k+1}_{l}\}, \{ \mathbf{r}_{l}^{k+1}\},\{ \mathbf{p}_{l}^{k+1}\},\{ q_{s}^{k+1}\},\mathbf{x}^{k+1}, \{{\boldsymbol{\lambda}}_{l}^{k+1}\},\{{\boldsymbol{\alpha}}_{l}^{k}\},\{{\boldsymbol{\beta}}_{l}^{k}\}, \{\gamma_{s}^{k}\} ))^{+},
l \in \mathcal{Q}^{k+1}}}\\[1mm]
{{{\boldsymbol{\alpha}}_{l}^{k},
l \notin \mathcal{Q}^{k+1}}}
\end{array}\right.,
\tag{30}
\\
& \label{beta_update}
{\boldsymbol{\beta}}_{l}^{k+1}=
\left\{\begin{array}{l}
({{{\boldsymbol{\beta}}_{l}^{k} + \eta_{{\boldsymbol{\beta}}   }\nabla_{{\boldsymbol{\beta}}_{l}}
\widetilde{L}_{p} (\{\boldsymbol{\mu}^{k+1}_{l}\}, \{ \mathbf{r}_{l}^{k+1}\},\{ \mathbf{p}_{l}^{k+1}\},\{ q_{s}^{k+1}\},\mathbf{x}^{k+1}, \{{\boldsymbol{\lambda}}_{l}^{k+1}\},\{{\boldsymbol{\alpha}}_{l}^{k+1}\},\{{\boldsymbol{\beta}}_{l}^{k}\}, \{\gamma_{s}^{k}\} ))^{+},
l \in \mathcal{Q}^{k+1}}}\\
{{{\boldsymbol{\beta}}_{l}^{k},
l \notin \mathcal{Q}^{k+1}}}
\end{array}\right.,
\tag{31}
\\
& \label{gamma_update}
\gamma_{s}^{k+1} = (\gamma_{s}^{k} + \eta_{\gamma   }\nabla_{\gamma_{s}}
\widetilde{L}_{p} (\{\boldsymbol{\mu}^{k+1}_{l}\}, \{ \mathbf{r}_{l}^{k+1}\},\{ \mathbf{p}_{l}^{k+1}\},\{ q_{s}^{k+1}\},\mathbf{x}^{k+1}, \{{\boldsymbol{\lambda}}_{l}^{k+1}\},\{{\boldsymbol{\alpha}}_{l}^{k+1}\},\{{\boldsymbol{\beta}}_{l}^{k+1}\}, \{\gamma_{s}^{k}\} ))^{+}.
\tag{32}
\end{align}
\end{figure*}
And the set of cutting planes will be updated as follows,
\begin{equation}
\mathcal{P}^{k+1}=\left\{\begin{array}{l}
\mathcal{P}^{k} \cup \{cp^{k+1}\} , \text { if } g_i(\mathbf{x}^{k+1})>\delta_i \\
\mathcal{P}^{k}, ~~otherwise
\end{array}\right..
\label{cuting_update}
\end{equation}
After the new cutting plane is added, its corresponding dual variable set and slack variable set will be updated accordingly,
\begin{equation}
\left\{q^{k+1}\right\}=\left\{\begin{array}{l}
~ \{q^{k}\} \cup ~ q_{\left|\mathcal{P}^{k}\right|+1}^{k+1}, \text { if }g_i(\mathbf{x}^{k+1})  >\delta_i \\
\left\{q^{k}\right\}, ~~otherwise
\end{array}\right.,
\label{slack_cutting_update}
\end{equation}
\begin{equation}
\left\{\gamma^{k+1}\right\}=\left\{\begin{array}{l}
\left\{\gamma^{k}\right\} \cup ~ \gamma_{\left|\mathcal{P}^{k}\right|+1}^{k+1}, \text { if }g_i(\mathbf{x}^{k+1}) >\delta_i \\
\left\{\gamma^{k}\right\},~~ otherwise
\end{array}\right. .
\label{dual_cutting_update}
\end{equation}
Based on the newly refined cutting plane sets,
Triadic-OCD solves the nested optimization problem in an asynchronous distributed way. 
We first provide the Lagrangian function of (\ref{distributed_3}) as follows, 
\begin{align}
& L_{p}(\{\boldsymbol{\mu}_{l}\}, \{ \mathbf{r}_{l}\},\{ \mathbf{p}_{l}\},\{ q_{s}\}, \mathbf{x}, \{{\boldsymbol{\lambda}}_{l}\},\{{\boldsymbol{\alpha}}_{l}\},\{{\boldsymbol{\beta}}_{l}\}, \{\gamma_{s}\} ) \nonumber \\[-1.5mm]
& =  \sum \limits_{l=1}^{L} \{ \|\boldsymbol{\mu}_{l}\|_{2}^{2} -   2{\mathbf{y}_{l}}^{T} \boldsymbol{\mu}_{l} \} +  \sum_{l=1}^{L} {\boldsymbol{\lambda}}_{l}^{T} (    -\boldsymbol{\mu}_{l} - \rho_{U}\mathbf{1} + \mathbf{r}_{l} \circ \mathbf{r}_{l}    ) \nonumber \\[-1.5mm]
& +\sum_{l=1}^{L} {\boldsymbol{\alpha}}_{l}^{T} (\boldsymbol{\mu}_{l} - \rho_{U}\mathbf{1} + \mathbf{p}_{l} \circ \mathbf{p}_{l} )  + \sum_{l=1}^{L} {\boldsymbol{\beta}}_{l}^{T}(\boldsymbol{\mu}_{l} - \mathbf{B}_{i}\mathbf{x}) \nonumber \\[-1.5mm]
& + \sum_{s=1}^{|\mathcal{P}^{k}|}\gamma_{s}(\mathbf{b}_{s}^{\top}\mathbf{x} + \kappa_{s} + q_{s}^2),
\label{ALM_2}
\end{align}
where ${\boldsymbol{\lambda}}_{l}$, ${\boldsymbol{\alpha}}_{l}$, ${\boldsymbol{\beta}}_{l}$, $\mathbf{\gamma}_{s}$ are dual variables. $\mathbf{\gamma}_{s} \in \mathbb{R}$ and ${\boldsymbol{\lambda}}_{l}$, ${\boldsymbol{\alpha}}_{l}$, ${\boldsymbol{\beta}}_{l}$ have the same dimensions as $\boldsymbol{\mu}_{l}$.
Simplify the Lagrangian function of (\ref{distributed_3}) as $L_p$, we next give the regularized version \cite{regularized} of $L_p$ as follows,
\begin{equation}
\resizebox{\linewidth}{!}{$
\begin{aligned}
& \widetilde{L}_{p} (\{\boldsymbol{\mu}_{l}\}, \{ \mathbf{r}_{l}\},\{ \mathbf{p}_{l}\},\{ q_{s}\}, \mathbf{x}, \{{\boldsymbol{\lambda}}_{l}\},\{{\boldsymbol{\alpha}}_{l}\},\{{\boldsymbol{\beta}}_{l}\}, \{\gamma_{s}\} )  = L_{p}  - \\[-2pt] 
& \sum_{l=1}^{L}\frac{c_{{\boldsymbol{\lambda}}}^{k}}{2}\| {\boldsymbol{\lambda}}_{l} \|^{2}
-\sum_{l=1}^{L}\frac{c_{{\boldsymbol{\alpha}}}^{k}}{2}\|{\boldsymbol{\alpha}}_{l}\|^{2} - \sum_{l=1}^{L}\frac{c_{{\boldsymbol{\beta}}}^{k}}{2}\|{\boldsymbol{\beta}}_{l}\|^{2} 
- \sum_{s=1}^{|\mathcal{P}^{k}|}\frac{c_{\gamma}^{k}}{2}\|{\gamma}_{s}\|^{2},
\end{aligned}
$}
\label{ALM_3}
\end{equation}
where $c_{\lambda}^{k}$, $c_{{\boldsymbol{\alpha}}}^{k}$, $c_{\beta}^{k}$, $c_{\gamma}^{k}$ denote the regularization terms in the $(k + 1)^{th}$ iteration.
We represent the upper bound of $|\mathcal{P}^{k}|$ as $P$. Following the settings in \cite{asynchronous_setting}, 
in each iteration, we update the variables in the master once it receives the local variables from $S$ active workers. 
And we require every worker to communicate with the master at least once every $\tau$ iterations.
Let $\mathcal{Q}^{k+1}$ denote the set of indexes of active workers in the $(k + 1)^{th}$ iteration, the variables in Triadic-OCD are updated as follows,

(1) Local variables in workers are updated according to (\ref{mu_update}), (\ref{r_update}), (\ref{p_update}).  Notations $\eta_{\boldsymbol{\mu}}$, $\eta_{\mathbf{r}}$ and $\eta_{\mathbf{p}}$ are step-sizes. 
And ${\hat{k}}_{l}$ denotes the last iteration in which worker $l$ was active.

(2) After receiving the updates from active workers, the master updates the variables according to (\ref{q_update}), 
(\ref{v_update}), (\ref{lambda_update}), (\ref{alpha_update}), (\ref{beta_update}), (\ref{gamma_update}).  Notations $\eta_{q}$, $\eta_{\mathbf{x}}$, $\eta_{{\boldsymbol{\lambda}}}$, $\eta_{{\boldsymbol{\alpha}}}$, $\eta_{{\boldsymbol{\beta}}}$, $\eta_{\gamma}$ are step-sizes. 

The whole process of Triadic-OCD is summarized in \cref{algorithm_1}.
\begin{figure*}[ht]
\begin{equation}
\tag{34}
\label{ic}
K(\epsilon) \sim \mathcal{O}\left(\max ~\left\{~(16(\frac{Mw_{{\boldsymbol{\lambda}}}^2}{\eta_\lambda^2} + \frac{Mw_{{\boldsymbol{\alpha}}}^2}{\eta_{{\boldsymbol{\alpha}}}^2} +
\frac{Mw_{{\boldsymbol{\beta}}}^2}{\eta_{\boldsymbol{\beta}}^2} + 
\frac{Pw_{\gamma}^2}{\eta_\gamma^2})^2 \frac{1}{\epsilon^2},
(\frac{  d_9(d_7 + k_d\tau)(\tau-1)d_8}{\epsilon}+(K_1+2)^{\frac{1}{2}})^2
\right\}~ 
\right).
\end{equation}
\end{figure*}

\section{Theoretical Analysis}
\begin{theorem}
As the set of cutting planes is continuously updated, the optimal objective value of the approximation problem (\ref{distributed_3}) converges monotonically.
\label{theorem_1}
\end{theorem}
The complete proof is provided in Appendix \ref{AD_C}.

Let $\nabla G^{k}$ denote the gradient of $L_p$ in the  $k$-th iteration.
According to \cite{convex_optimization}, we know that $(\{\boldsymbol{\mu}^{k}_{l}\}$, $\{ \mathbf{r}_{l}^{k}\}$, $\{ \mathbf{p}_{l}^{k}\}$, $\{ q_{s}^{k}\}$, $\mathbf{x}^{k}$, $\{{\boldsymbol{\lambda}}_{l}^{k}\}$, $\{{\boldsymbol{\alpha}}_{l}^{k}\}$, $\{{\boldsymbol{\beta}}_{l}^{k}\}$, $\{\gamma_{s}^{k}\})$ is the optimal solution of problem (\ref{distributed_3}) if and only if $\| \nabla G^{k}\|^{2} = 0$.

\begin{definition}
$(\{\boldsymbol{\mu}^{k}_{l}\}$, $\{ \mathbf{r}_{l}^{k}\}$, $\{ \mathbf{p}_{l}^{k}\}$, $\{ q_{s}^{k}\}$, $\mathbf{x}^{k}$, $\{{\boldsymbol{\lambda}}_{l}^{k}\}$, $\{{\boldsymbol{\alpha}}_{l}^{k}\}$, $\{{\boldsymbol{\beta}}_{l}^{k}\}$, $\{\gamma_{s}^{k}\})$ is an $\epsilon$-optimal point of (\ref{distributed_3}) if  $\| \nabla G^{k}\|^{2} \leq \epsilon$.
Define $K(\epsilon)$ as the first iteration index such that $\| \nabla G^{k}\|^{2} \leq \epsilon$, i.e., $K(\epsilon)=\operatorname*{min}\{k \mid| \| \nabla G^{k}\|^{2}\leq\epsilon\}$.
\label{definition}
\end{definition}

\begin{assumption}
Following \cite{ADBO,boundness}, we assume that variables are
bounded, i.e., $\|\boldsymbol{\mu}_{l}\|_{\infty} \leq w_{\boldsymbol{\mu}}$, 
$\| \mathbf{r}_{l} \|_{\infty} \leq w_{\mathbf{r}}$,
$\|\mathbf{p}_{l}\|_{\infty} \leq w_{\mathbf{p}}$,
$\|q_{s}\|_{\infty} \leq w_{q}$, 
$\|{\boldsymbol{\lambda}}_{l}\|_{\infty} \leq w_{{\boldsymbol{\lambda}}}$,
$\|{\boldsymbol{\alpha}}_{l}\|_{\infty} \leq w_{{\boldsymbol{\alpha}}}$,
$\|{\boldsymbol{\beta}}_{l}\|_{\infty} \leq w_{{\boldsymbol{\beta}}}$, 
$\|\gamma_{s}\|_{\infty} \leq w_{\gamma}$. 
Before obtaining the $\epsilon$-optimal point, we assume variable $\mathbf{x}$ satisfies that $
\| \mathbf{x}^{k+1} - \mathbf{x}^{k} \|^{2}\geq \vartheta, $
where $\vartheta > 0$ is a small constant. And the change of the $\mathbf{x}$ is upper bounded within $\tau$ iterations, i.e., 
$\| \mathbf{x}^{k} - \mathbf{x}^{k-\tau}\|^{2}
\leq \tau k_{1}\vartheta $,where $k_{1} > 0$ is a constant.
\label{assumption_1}
\end{assumption}

\begin{theorem}(Iteration complexity)
Suppose Assumption \ref{assumption_1}  holds, set 
\setcounter{equation}{32}
\begin{equation}
\begin{aligned}
& c_{{\boldsymbol{\lambda}}}^{k} = \frac{1}{  \eta_{{\boldsymbol{\lambda}}}(k+1)^{\frac{1}{4}}}\geq \underline{c_{{\boldsymbol{\lambda}}}} , \quad c_{{\boldsymbol{\alpha}}}^{k} = \frac{1}{  \eta_{{\boldsymbol{\alpha}}}(k+1)^{\frac{1}{4}}}\geq \underline{c_{{\boldsymbol{\alpha}}}}, \\
& c_{{\boldsymbol{\beta}}}^{k} = \frac{1}{  \eta_{{\boldsymbol{\beta}}}(k+1)^{\frac{1}{4}}}\geq \underline{c_{{\boldsymbol{\beta}}}}   ,\quad c_{\gamma}^{k} = \frac{1}{  \eta_{\gamma}(k+1)^{\frac{1}{4}}} \geq \underline{c_{\gamma}}, \\
& \eta_{\boldsymbol{\mu}} = \frac{1}{ \frac{8\xi}{\eta_{{\boldsymbol{\lambda}}}(  \underline{ c_{{\boldsymbol{\lambda}}} }  )^2}   +  \frac{8\xi}{\eta_{{\boldsymbol{\alpha}}}(  \underline{ c_{{\boldsymbol{\alpha}}}  } )^2}  +  \frac{8\xi}{  \eta_{{\boldsymbol{\beta}}}(  \underline{  c_{{\boldsymbol{\beta}}}  })^2}} , \\
&  \eta_{\mathbf{x}} = \frac{1}{ \frac{8\xi}{\eta_{{\boldsymbol{\beta}}}(  \underline{c_{{\boldsymbol{\beta}}}} )^2}    + \frac{8\xi\nu}{\eta_{{\gamma}}(  \underline{c_{{\gamma}}} )^2}   },   \\
& \eta_{\mathbf{r}} = \frac{\eta_{{\boldsymbol{\lambda}}}(\underline{ c_{{\boldsymbol{\lambda}}}  })^2}{  32w^2_{\mathbf{r}}\xi},   \quad  \eta_{\mathbf{p}} = \frac{\eta_{{\boldsymbol{\alpha}}}(\underline{ c_{{\boldsymbol{\alpha}}}  })^2}{  32w^2_{\mathbf{p}}\xi},  \quad\eta_{q} = \frac{\eta_{\gamma}(\underline{ c_{\gamma}  })^2}{  32w^2_{q}\xi} .\\
\end{aligned}
\end{equation}
$\underline{c_{{\boldsymbol{\lambda}}}} \ge 0 $, $\underline{c_{\mathbb\alpha}}\ge 0$, $\underline{c_{{\boldsymbol{\beta}}}}\ge 0$, $\underline{c_{\gamma}}\ge 0$ represent the lower bound of $c_{{\boldsymbol{\lambda}}}^{k}$, $c_{{\boldsymbol{\alpha}}}^{k}$, $c_{{\boldsymbol{\beta}}}^{k}$, $c_{\gamma}^{k}$, respectively.
For any prescribed $\epsilon$, 
the iteration complexity of Triadic-OCD to converge to the $\epsilon$-optimal point is shown in (\ref{ic}).
Notation $\xi$, $\nu$, $d_{7}$, $d_{8}$, $d_{9}$ and $k_d$ are all constants. 
\label{theorem_2}
\end{theorem}
Due to lack of space, we relegate the complete proof to the \cref{complete-proof}.

\begin{algorithm}[h] 
    \SetAlgoLined
    \KwIn{ Iteration variable $k = 0$. Initialize $\{\boldsymbol{\mu}^{0}_{l}\}$, $\{ \mathbf{r}_{l}^{0}\}$, $\{ \mathbf{p}_{l}^{0}\}$, $\{ q_{s}^{0}\}$, $\mathbf{x}^{0}$,  $\{{\boldsymbol{\lambda}}_{l}^{0}\}$, $\{{\boldsymbol{\alpha}}_{l}^{0}\}$, $\{{\boldsymbol{\beta}}_{l}^{0}\}$, $\{\gamma_{s}^{0}\}$.}
    \Repeat{termination}
    {
        \For {workers} {
        updates local variables $\boldsymbol{\mu}_{l}^{k+1}$,$\mathbf{r}_{l}^{k+1}$,$\mathbf{p}_{l}^{k+1}$ according to (\ref{mu_update}), (\ref{r_update}) and (\ref{p_update}); }
        $\boldsymbol{\mu}_{l}^{k+1}$, $\mathbf{r}_{l}^{k+1}$, $\mathbf{p}_{l}^{ k+1}$ in active workers are transmitted to master;
        
        \For {master} {
        updates variables $\{ q_{s}^{k+1}\}$, $\mathbf{x}^{k+1}$, $\{{\boldsymbol{\lambda}}_{l}^{k+1}\}$, $\{{\boldsymbol{\alpha}}_{l}^{k+1}\}$, $\{{\boldsymbol{\beta}}_{l}^{k+1}\}$, $\{\gamma_{s}^{k+1}\}$  according to (\ref{q_update}), (\ref{v_update}),  (\ref{lambda_update}), (\ref{alpha_update}), (\ref{beta_update}) and (\ref{gamma_update}); }
        master broadcasts $\{q_{s}^{k+1}\}$, $\mathbf{x}^{k+1}$,  $\{{\boldsymbol{\lambda}}_{l}^{k+1}\}$, $\{{\boldsymbol{\alpha}}_{l}^{k+1}\}$, $\{{\boldsymbol{\beta}}_{l}^{k+1}\}$, $\{ \gamma_{s}^{k+1} \}$ to active workers; 
    
        \If{ $(k+1) \bmod w = 0$ and $k \leq K_{1}$ }    {
        master computes $g_i(\mathbf{x}^{k+1})$ according to (\ref{function_g});
    
        master updates $\mathcal{P}^{k+1}$, $\{q^{k+1}\}$ and $ \{ {\gamma}^{k+1} \} $ according to (\ref{cuting_update}), (\ref{slack_cutting_update}), (\ref{dual_cutting_update}) and broadcasts them to all workers; }
        $k = k + 1$;
    }    
  \KwOut {$\{\boldsymbol{\mu}^{k}_{l}\}$, $\{ \mathbf{r}_{l}^{k}\}$, $\{ \mathbf{p}_{l}^{k}\}$, $\{ q_{s}^{k}\}$, $\mathbf{x}^{k}$, $\{{\boldsymbol{\lambda}}_{l}^{k}\}$, $\{{\boldsymbol{\alpha}}_{l}^{k}\}$, $\{{\boldsymbol{\beta}}_{l}^{k}\}$, $\{\gamma_{s}^{k}\}$}
  \caption{Triadic-OCD}
  \label{algorithm_1}
\end{algorithm}

\section{Experiment}    
As discussed before, the problem considered in this paper subsumes the detection of FDIA in smart grids.
Therefore, we conduct extensive experiments to show the effectiveness of Triadic-OCD for the FDIA detection task.
In this section, we first briefly introduce the background of the smart grid, followed by a presentation of the comprehensive experimental results.
\subsection{Dateset} 
A smart grid is an advanced electrical system that integrates information and communication technology to optimize energy supply, demand, and distribution. Given the potential for cyber-attacks to cause catastrophic consequences, the detection of false data injection attacks (FDIA) in smart grids is an important and extensively studied problem \cite{kosut2011malicious,cui2012coordinated}. 
In the context of a smart grid system comprising $(N+1)$ buses and $M$ meters, the dynamic DC power flow model of the system can be exactly expressed as equation (\ref{equ-1}), where $\boldsymbol{\theta}^{(t)} \in \mathbb{R}^{N}$ represents the phase angles (one reference angle) and  $\mathbf{y}^{(t)} \in \mathbb{R}^{M}$ represents the meter readings. $\mathbf{H} \in \mathbb{R}^{M \times N}$ is the measurement matrix that is determined by the topology of the system and the admittance of each transmission line. 
In wide-area monitoring, the power grid is divided into multiple sub-regions, with each having limited access to its own meter measurements and communicating with other sub-regions and the control center through the wireless medium.
The centralized setup may not be feasible due to power and bandwidth constraints, leading to a growing interest in distributed detection to minimize communication overhead.

In our experiments, we utilized the measurement data generated from the IEEE-14 bus power system, which can be regarded as a benchmark dataset in current research studies.
The IEEE 14-bus system  as well as the
system matrix can be divided into four sub-regions \cite{methodin14}.
In our experiment, we assume that the system matrix can be accurately determined in all areas except for the $4^{th}$ area.
Suppose attackers utilize a sequence of randomly generated fabricated vectors to compromise the meter reading $\mathbf{y}^{(t)}$ from the time instant $t_a$.
The generation process of the injected attacks is as follows,
\setcounter{equation}{34}
\begin{equation}
\mathbf{a}^{(t)} = \mathbf{P}_{\mathbf{H}}^{\perp}\mathbf{u}, {u}_i \sim \mathbf{U}(0.1,1),
\label{generate-attacks}
\end{equation}
where
$
\mathbf{P}_{\mathbf{H}}^{\perp} \triangleq \mathbf{I}-\mathbf{H}\left(\mathbf{H}^{T} \mathbf{H}\right)^{-1} \mathbf{H}^{T}
$.
Given the fluctuations in admittance caused by factors like environmental disturbances, the detection algorithms employed in the experiment utilize imprecise estimates of $\mathbf{H}$ to identify anomaly vectors injected into smart grid systems.
We provide the inaccurate estimates of $\mathbf{H}$ in Appendix \ref{Ap_D}, as well as the corresponding detailed settings of various uncertainty sets.

\begin{figure*}[ht]
  \centering
  \begin{minipage}[t]{0.48\textwidth}
    \centering
    \begin{minipage}[t]{\textwidth}
      \centering
      \input{Figures/ICML_1}
    \end{minipage}
    \caption{Performance comparison of different detection algorithms.}
    \label{pic-1}
  \end{minipage}\hfill
  \begin{minipage}[t]{0.48\textwidth}
    \centering
    \begin{minipage}[t]{\textwidth}
      \centering
      \input{Figures/ICML_2}
    \end{minipage}
    \caption{Performance of Triadic-OCD under various uncertainty Sets.}
    \label{pic-2}
  \end{minipage}
\end{figure*}
\begin{figure*}[ht]
  \centering
  \begin{minipage}[t]{0.48\textwidth}
    \centering
    \begin{minipage}[t]{\textwidth}
      \centering
      \input{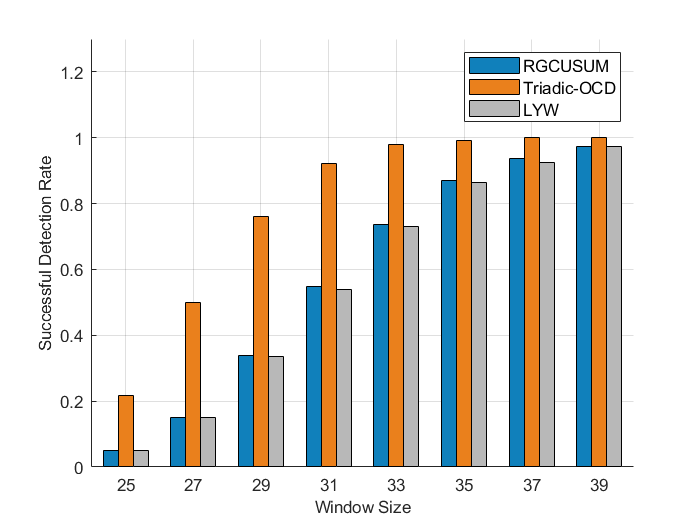}
    \end{minipage}
    \caption{Successful detection rate versus the corresponding upper bound on detection delay.}
    \label{pic-3}
  \end{minipage}\hfill
  \begin{minipage}[t]{0.48\textwidth}
    \centering
    \begin{minipage}[t]{\textwidth}
      \centering
      \input{Figures/asy_sy_3}
    \end{minipage}
    \caption{Convergence speed comparison between Triadic-OCD and its synchronous counterpart.}
    \label{asy_sy}
  \end{minipage}
\end{figure*}

\subsection{Numerical Results} 
As discussed in \cref{Related-work}, we emphasized that the conventional OCD methods are built upon various assumptions, such as the assumption of constant distributions before and after the change, as well as the requirement for the system state to adhere to specific conditions. As a result, these methods are unsuitable for the problem considered in this paper and we implemented the Adaptive CUSUM algorithm \cite{adaptive_cusum} to exemplify this fact. Additionally, we conducted a performance comparison between our algorithm and two others: the method proposed in \cite{methodin14} and the RGCUSUM algorithm \cite{zhangjiangfan}, which is claimed to be the state-of-the-art detector for the FDIA detection in smart grids. For simplicity, we name the 
wide-area cyber-attack detection method proposed in \cite{methodin14} as WCD.

The experimental results after 500 runs are shown in \cref{pic-1}. 
For a fair comparison, we adhere to the conventional metric for OCD task, that is, we compare the average detection delay of different algorithms under the same False Alarm Periods (FAP). 
FAP refers to the detector's stopping time when no change occurs, serving as a measure to assess the risk of false alarms.
From \cref{pic-1}, it is evident that our method consistently exhibit a smaller average detection delay for any given FAP. This observation underscores the superior performance achieved by Triadic-OCD.
The adaptive CUCUM exhibits the worst performance, 
primarily due to the inconsistency of its assumptions with the problem of interest. 
Adaptive CUCUM presupposes the Gaussian distribution of $\boldsymbol{\theta}^{(t)}$, which does not hold in our problem. 
In our experiment, both RGUCUSM and WCD demonstrate inferior performance compared to Triadic-OCD.
This is in line with our expectations since they can achieve commendable performance only when the system matrix $\mathbf H$ can be precisely determined. The imperfect knowledge about $\mathbf H$ leads to a rapid degradation of their detection performance. In contrast, Triadic-OCD effectively handles uncertainties within the system matrix.
We also conduct experiments to show the performance of Triadic-OCD when the system matrix $\mathbf{H}$ is assumed to belong to different uncertainty sets. The uncertainty sets we consider are commonly used in practical applications, including ellipsoid uncertainty set, D-norm uncertainty set, and polyhedron uncertainty set \cite{yang2014distributed}. The number of Monte Carlo runs is 500. 
As depicted in \cref{pic-2}, Triadic-OCD consistently demonstrates superior performance across diverse uncertainty sets, highlighting its generalization ability.

Furthermore, to demonstrate the robustness of Triadic-OCD against various attacking vectors, we generated $10^3$ instances of attacking vectors randomly to encompass a wide array of attack behaviors according to (\ref{generate-attacks}).
We calculate the success rate of different detectors under varying upper bounds on the detection delay. 
Specifically, the success rate refers to the proportion of attack vectors that can be successfully detected within the specified upper bound on the detection delay. 
We then plot the success rate corresponding to different upper bounds on detection delay in \cref{pic-3}. It can be seen from \cref{pic-3} that as the upper bound increases, the successful detection rates of all detectors increase as expected.
The proposed detector consistently outperforms RGCUSUM and WCD in successful detection rate, thus underscoring the robustness of Triadic-OCD against diverse attacking vectors.

Additionally, we also compare the convergence rate of Triadic-OCD under asynchronous and synchronous conditions to highlight the benefits of asynchronous variable updates in problem-solving. Following \cite{cohen2021asynchronous}, the delay of each worker is assumed to obey log-normal distribution LN(1, 0.5).
Recall that $S$ denotes the number of active workers required for updating the master node, and $\tau$ represents the maximum iteration interval for the communication between each worker and the master. 
For the asynchronous setting, we set $S = 10$ and $\tau = 10$. In the synchronous setup, the master can update its variables only after receiving updates from all workers. As shown in \cref{asy_sy}, Triadic-OCD exhibits significantly faster convergence in the asynchronous mode compared to its synchronous counterpart. This is attributed to asynchronous variable updates, which prevent the algorithm from being hindered by the workers with high delays during each iteration, thus demonstrating the efficiency of our approach.

\section{Conclusion}
Existing studies often conduct online change detection with perfect knowledge of system parameters—a presumption that proves unfeasible in practical scenarios. Moreover, these studies typically focus on either centralized or synchronously distributed settings, which can lead to privacy breaches, straggler issues, and high communication costs.
As a remedy, we develop an asynchronous framework for OCD with provable robustness, optimality, and convergence. 
To our best knowledge, this work represents the first step that tackles OCD with parameter uncertainties in an asynchronous setting. We also provide 
non-asymptotic theoretical analysis for the convergence property of triadic-OCD. Extensive experiments have been further conducted to elucidate the efficiency and effectiveness of the proposed algorithm.

\section*{Impact Statement}
This paper presents work whose goal is to advance the field of Machine Learning. There are many potential societal consequences of our work, none which we feel must be specifically highlighted here.

\newpage
\bibliography{example_paper}
\bibliographystyle{icml2024}

\newpage
\appendix
\onecolumn

\section{Experimental Settings}
\label{Ap_D}
The IEEE 14-bus system  as well as the
system matrix can be divided into four sub-regions .
The true system matrix for the $4^{th}$ area is as follows \cite{methodin14},
\begin{equation}
\mathbf{H}_4 = \begin{bmatrix}
-1 & 3 & -1 \\
0 & -1 & 0 \\
0 & 0 & -1 \\
0 & -1 & 1 \\
0 & -1 & 2 
\end{bmatrix}.
\end{equation}
In our experiment, we assume that the system matrix can be accurately determined in all areas except for the $4^{th}$ area.

The polyhedron uncertainty of $\mathbf{H}_4$ is set to be $ \mathbf{D}_{i} \mathbf{h}_{i}  \leq  \mathbf{c}_{i},~\forall i, $
where $\mathbf{h}_{i}$ represents the $i$-th column of $\mathbf{H}_4$. 
$\mathbf{D}_{1}$, $\mathbf{D}_{2}$ and $\mathbf{D}_{3}$ are set to be 
$$ \mathbf{D}_{1} = \mathbf{D}_{2} = \mathbf{D}_{3} = \begin{bmatrix}~  \mathbf{I} ~ \\ ~ \mathbf{-I} ~~  \end{bmatrix},$$ 
where  $\mathbf{I}$  represents the $5 \times 5$ identity matrix.  $\mathbf{c}_{1} $, $\mathbf{c}_{2}$ and $\mathbf{c}_{3}$ are set to be
\begin{equation}
\begin{aligned}
& \mathbf{c}_{1} = [ -0.5,0.5,0.5,0.5,0.5,1.5,0.5,0.5,0.5,0.5]^{T}, \\
& \mathbf{c}_{2} = [ 3.5,-0.5,0.5,-0.5,-0.5,-2.5,1.5,0.5,1.5,1.5]^{T}, \\
& \mathbf{c}_{3} = [ -0.5,0.5,-0.5,1.5,2.5,1.5,0.5,1.5,-0.5,-1.5]^{T}. 
\nonumber
\end{aligned}
\end{equation}
As for the ellipsoid uncertainty set, the uncertainty of each $\mathbf{h}_i$ is described as follows,
\begin{equation}
\nonumber
\mathbf{h}_i \in \{\overline{ \mathbf{h}}_i +\mathbf{u}\mid  \|\mathbf{u}\|_2\leq 0.36\}, \quad \forall i,
\end{equation}
where $\overline{ \mathbf{h}}_i $ is estimate of the $i$-th column of $\mathbf{H}_t$. In our experiment, $\overline{ \mathbf{h}}_i $ are set to be
\begin{equation}
\begin{aligned}
& \overline{ \mathbf{h}}_1  = [ -1.0,0.1,0.3,-0.2,0.0]^{T}, \\
& \overline{ \mathbf{h}}_2  = [ 3.0,-0.7,0.2,-1.3,-0.9]^{T}, \\
& \overline{ \mathbf{h}}_3  = [ -1.1,0.2,-0.6,0.7,2.0]^{T}. 
\nonumber
\end{aligned}
\end{equation}

For any vector belonging to the D-norm uncertain set, it is assumed that the vector has at most $\Gamma$ uncertain components, and each component falls within the error interval determined by $\hat{u}$ \cite{yang2014distributed}. For all the $\mathbf{h}_{i}$ in our experiment, the parameter $\Gamma$ and $\hat{u}$ are set to be $4$ and $0.5$, respectively.


\clearpage
\section{ Derivation of (\ref{v_t}) }
\label{Ad_A}
Since $\mathbf{n}^{(t)} \stackrel{i . i . d .}{\sim} \mathcal{N}\left(\mathbf{0}, \sigma_n^{2} \mathbf{I}_{M}\right)$ , according to (\ref{equ-1}) we have that,
\begin{equation}
\label{pre_pdf}
\begin{aligned}
f_p\left(\mathbf{y}^{(t)}\middle|\boldsymbol{\theta}^{(t)},\mathbf{H}\right) & = \frac{1}{\left(2\pi\sigma_n^2\right)^{\frac{M}{2}}}\exp\bigg[-\frac{1}{2\sigma_n^2}  \left(\mathbf{y}^{(t)}-\mathbf{H}\boldsymbol{\theta}^{(t)}\right)^T\left(\mathbf{y}^{(t)}-\mathbf{H}\boldsymbol{\theta}^{(t)}\right)\Bigg],\\
\end{aligned}
\end{equation}
\begin{equation}
\label{post_pdf}
\begin{aligned}
& f_q\left(\mathbf{y}^{(t)}\middle|\boldsymbol{\theta}^{(t)},\mathbf{a}^{(t)},\mathbf{H}\right) =\frac{1}{\left(2\pi\sigma_n^2\right)^{\frac{M}{2}}}\exp\bigg[-\frac{1}{2\sigma_n^2}   \left(\mathbf{y}^{(t)}-\mathbf{H}\boldsymbol{\theta}^{(t)}-\mathbf{a}^{(t)}\right)^T\left(\mathbf{y}^{(t)}-\mathbf{H}\boldsymbol{\theta}^{(t)}-\mathbf{a}^{(t)}\right)\Bigg].
\end{aligned}
\end{equation}

According to the definition of $\Lambda_{j}^{(J)}$ in (\ref{test statistic}),  we can obtain that,
\begin{align}
\Lambda_{j}^{(J)}  
& =   
\sup\limits_{\mathbf{H} \in \mathcal{U}} \sum_{t=j}^{J}  \bigg\{   \sup \limits_{\boldsymbol{\theta}^{(t)}, \mathbf{a}^{(t)}: -\rho_{U}\mathbf{1} \leq \mathbf{x}^{(t)} \leq \rho_{U}\mathbf{1}, 
\mathbf{H}\mathbf{x}^{(t)} =\mathbf{0}  } 
  \ln f_{q}\left(\mathbf{y}^{(t)} \mid \boldsymbol{\theta}^{(t)}, \mathbf{a}^{(t)}, \mathbf{H}\right)  
- \sup \limits_{\boldsymbol{\theta}^{(t)}}  \ln 
f_{p}\left(\mathbf{y}^{(t)} \mid \boldsymbol{\theta}^{(t)},\mathbf{H}\right)\bigg\} 
\label{small_lambda}
\end{align}
Let $\tilde{\mathbf{y}}^{(t)}$ represents the component of $\mathbf{y}^{(t)}$ orthogonal to the column space of $\mathbf{H}$, i.e., 
\begin{equation}
\tilde{\mathbf{y}}^{(t)}=\mathbf{P}_{\mathbf{H}}^{\perp} \mathbf{y}^{(t)}.
\label{comple_space_H}
\end{equation}
Recall that 
$
\mathbf{P}_{\mathbf{H}}^{\perp} \triangleq \mathbf{I}-\mathbf{H}\left(\mathbf{H}^{T} \mathbf{H}\right)^{-1} \mathbf{H}^{T}.
$
Plugging (\ref{pre_pdf}) and (\ref{post_pdf}) into (\ref{small_lambda}), we have that, 
\begin{align}
\Lambda_{j}^{(J)}  
& =   \sup\limits_{\mathbf{H}\in \mathcal{U}} 
\sum_{t=j}^{J}  \Big\{ \sup\limits_{\mathbf{a}^{(t)}: -\rho_{U}\mathbf{1} \leq \mathbf{x}^{(t)} \leq \rho_{U}\mathbf{1}, 
\mathbf{H}\mathbf{x}^{(t)} =\mathbf{0}  } \Bigg[ -\frac{1}{2\sigma_n^2} 
\times\left(\tilde{\mathbf{y}}^{(t)} - \mathbf{x}^{(t)}\right)^T\left(\tilde{\mathbf{y}}^{(t)}-\mathbf{x}^{(t)}\right)\Biggr]
+  \frac{1}{2\sigma_n^{2}}\|\tilde{\mathbf{y}}^{(t)}\|^{2}  \Big\}\\
& = \sup\limits_{\mathbf{H}\in \mathcal{U}} 
\sum_{t=j}^{J} \sup\limits_{-\rho_{U}\mathbf{1} \leq \mathbf{x}^{(t)} \leq \rho_{U}\mathbf{1},
\mathbf{H}\mathbf{x}^{(t)} =\mathbf{0}  } \Bigg[  \frac{1}{2\sigma_n^2} \left\{
2(\mathbf{x}^{(t)})^{T}\tilde{\mathbf{y}}^{(t)} - \| \mathbf{x}^{(t)}\|_{2}^{2} 
\right\} \Biggr]
\\
& = \sum_{t=j}^{J}  \frac{v_t}{2\sigma^2},
\end{align}
where 
\begin{equation}
v_t \triangleq \sup_{\mathbf{H} \in \mathcal{U}}  \sup\limits_{-\rho_{U}\mathbf{1} \leq \mathbf{x}^{(t)} \leq \rho_{U}\mathbf{1},
\mathbf{H}\mathbf{x}^{(t)} =\mathbf{0}  } \Bigg[  \frac{1}{2\sigma_n^2} \left\{
2(\mathbf{x}^{(t)})^{T}\tilde{\mathbf{y}}^{(t)} - \| \mathbf{x}^{(t)}\|_{2}^{2} 
\right\} \Biggr].
\label{2024-1}
\end{equation}
Since $\mathbf{H}\mathbf{x}^{(t)} =\mathbf{0}$ and $\tilde{\mathbf{y}}^{(t)}$ represents the component of $\mathbf{y}^{(t)}$ orthogonal to the column space of $\mathbf{H}$, (\ref{2024-1}) is equivalent to (\ref{v_t}).

\clearpage
\section{ Complete Proof of Theorem \ref{theorem_2}  }
\label{complete-proof}
In this section,  we provide the comprehensive proof of Theorem \ref{theorem_2}. We will start by introducing several definitions. Following that,  three lemmas are presented, which are crucial components in the proof of Theorem \ref{theorem_2}.
Finally, we provide the complete proof of Theorem \ref{theorem_2}. 
\begin{definition}
Based on the definition of $\nabla G^k$, we further define :
\begin{equation}
\begin{aligned}
&(\nabla G^k)_{\boldsymbol{\mu}_l} = \nabla_{\boldsymbol{\mu}_{l}}
{L}_{p} (\{\boldsymbol{\mu}^{k}_{l}\}, \{ \mathbf{r}_{l}^{k}\},\{ \mathbf{p}_{l}^{k}\},\{ q_{s}^{k}\},\mathbf{x}^{k}, \{{\boldsymbol{\lambda}}_{l}^{k}\},\{{\boldsymbol{\alpha}}_{l}^{k}\},\{{\boldsymbol{\beta}}_{l}^{k}\}, \{\gamma_{s}^{k}\} ), \\
&(\nabla G^k)_{\mathbf{r}_l} = \nabla_{\mathbf{r}_{l}}
{L}_{p} (\{\boldsymbol{\mu}^{k}_{l}\}, \{ \mathbf{r}_{l}^{k}\},\{ \mathbf{p}_{l}^{k}\},\{ q_{s}^{k}\},\mathbf{x}^{k}, \{{\boldsymbol{\lambda}}_{l}^{k}\},\{{\boldsymbol{\alpha}}_{l}^{k}\},\{{\boldsymbol{\beta}}_{l}^{k}\}, \{\gamma_{s}^{k}\} ), \\
&(\nabla G^k)_{\mathbf{p}_l} = \nabla_{\mathbf{p}_{l}}
{L}_{p} (\{\boldsymbol{\mu}^{k}_{l}\}, \{ \mathbf{r}_{l}^{k}\},\{ \mathbf{p}_{l}^{k}\},\{ q_{s}^{k}\}, \mathbf{x}^{k},\{{\boldsymbol{\lambda}}_{l}^{k}\},\{{\boldsymbol{\alpha}}_{l}^{k}\},\{{\boldsymbol{\beta}}_{l}^{k}\}, \{\gamma_{s}^{k}\} ), \\
&(\nabla G^k)_{q_s} = \nabla_{ q_s }
{L}_{p} (\{\boldsymbol{\mu}^{k}_{l}\}, \{ \mathbf{r}_{l}^{k}\},\{ \mathbf{p}_{l}^{k}\},\{ q_{s}^{k}\}, \mathbf{x}^{k},\{{\boldsymbol{\lambda}}_{l}^{k}\},\{{\boldsymbol{\alpha}}_{l}^{k}\},\{{\boldsymbol{\beta}}_{l}^{k}\}, \{\gamma_{s}^{k}\} ), \\
&(\nabla G^k)_{\mathbf{x}} = \nabla_{ \mathbf{x}}
{L}_{p} (\{\boldsymbol{\mu}^{k}_{l}\}, \{ \mathbf{r}_{l}^{k}\},\{ \mathbf{p}_{l}^{k}\},\{ q_{s}^{k}\}, \mathbf{x}^{k},\{{\boldsymbol{\lambda}}_{l}^{k}\},\{{\boldsymbol{\alpha}}_{l}^{k}\},\{{\boldsymbol{\beta}}_{l}^{k}\}, \{\gamma_{s}^{k}\} ), \\
&(\nabla G^k)_{{\boldsymbol{\lambda}}_{l}} = \nabla_{{\lambda}_{l}}
{L}_{p} (\{\boldsymbol{\mu}^{k}_{l}\}, \{ \mathbf{r}_{l}^{k}\},\{ \mathbf{p}_{l}^{k}\},\{ q_{s}^{k}\}, \mathbf{x}^{k},\{{\boldsymbol{\lambda}}_{l}^{k}\},\{{\boldsymbol{\alpha}}_{l}^{k}\},\{{\boldsymbol{\beta}}_{l}^{k}\}, \{\gamma_{s}^{k}\} ), \\
&(\nabla G^k)_{{\boldsymbol{\alpha}}_{l}} = \nabla_{{\alpha}_{l}}
{L}_{p} (\{\boldsymbol{\mu}^{k}_{l}\}, \{ \mathbf{r}_{l}^{k}\},\{ \mathbf{p}_{l}^{k}\},\{ q_{s}^{k}\}, \mathbf{x}^{k},\{{\boldsymbol{\lambda}}_{l}^{k}\},\{{\boldsymbol{\alpha}}_{l}^{k}\},\{{\boldsymbol{\beta}}_{l}^{k}\}, \{\gamma_{s}^{k}\} ), \\
&(\nabla G^k)_{{\boldsymbol{\beta}}_{l}} = \nabla_{{\beta}_{l}}
{L}_{p} (\{\boldsymbol{\mu}^{k}_{l}\}, \{ \mathbf{r}_{l}^{k}\},\{ \mathbf{p}_{l}^{k}\},\{ q_{s}^{k}\}, \mathbf{x}^{k},\{{\boldsymbol{\lambda}}_{l}^{k}\},\{{\boldsymbol{\alpha}}_{l}^{k}\},\{{\boldsymbol{\beta}}_{l}^{k}\}, \{\gamma_{s}^{k}\} ), \\
&(\nabla G^k)_{{\gamma}_{s}} = \nabla_{{\gamma}_{s}}
{L}_{p} (\{\boldsymbol{\mu}^{k}_{l}\}, \{ \mathbf{r}_{l}^{k}\},\{ \mathbf{p}_{l}^{k}\},\{ q_{s}^{k}\}, \mathbf{x}^{k},\{{\boldsymbol{\lambda}}_{l}^{k}\},\{{\boldsymbol{\alpha}}_{l}^{k}\},\{{\boldsymbol{\beta}}_{l}^{k}\}, \{\gamma_{s}^{k}\} ). \\
\end{aligned}
\end{equation}
\label{definition_2}
\end{definition}

\begin{definition}
\label{definition_3}
Similar to the definition of $\nabla {G}^{k}$, 
Denote $\nabla \widetilde{G}^{k}$ as the gradient of $\widetilde{L}_{p}$ in the $k$-th iteration, i.e., 
\begin{equation}
\nabla \widetilde{G}^{k}= \begin{bmatrix}
\{ \nabla_{\boldsymbol{\mu}_{l}}
\widetilde{L}_{p} (\{\boldsymbol{\mu}^{k}_{l}\}, \{ \mathbf{r}_{l}^{k}\},\{ \mathbf{p}_{l}^{k}\},\mathbf{x}^{k},\{ q_{s}^{k}\}, \{{\boldsymbol{\lambda}}_{l}^{k}\},\{{\boldsymbol{\alpha}}_{l}^{k}\},\{{\boldsymbol{\beta}}_{l}^{k}\}, \{\gamma_{s}^{k}\} )\} \\
\{ \nabla_{\mathbf{r}_{l}}
\widetilde{L}_{p} (\{\boldsymbol{\mu}^{k}_{l}\}, \{ \mathbf{r}_{l}^{k}\},\{ \mathbf{p}_{l}^{k}\},\mathbf{x}^{k},\{ q_{s}^{k}\}, \{{\boldsymbol{\lambda}}_{l}^{k}\},\{{\boldsymbol{\alpha}}_{l}^{k}\},\{{\boldsymbol{\beta}}_{l}^{k}\}, \{\gamma_{s}^{k}\} )\} \\
\{ \nabla_{\mathbf{p}_{l}}
\widetilde{L}_{p} (\{\boldsymbol{\mu}^{k}_{l}\}, \{ \mathbf{r}_{l}^{k}\},\{ \mathbf{p}_{l}^{k}\},\mathbf{x}^{k},\{ q_{s}^{k}\}, \{{\boldsymbol{\lambda}}_{l}^{k}\},\{{\boldsymbol{\alpha}}_{l}^{k}\},\{{\boldsymbol{\beta}}_{l}^{k}\}, \{\gamma_{s}^{k}\} )\} \\
\{ \nabla_{{q}_{s}}
\widetilde{L}_{p} (\{\boldsymbol{\mu}^{k}_{l}\}, \{ \mathbf{r}_{l}^{k}\},\{ \mathbf{p}_{l}^{k}\},\mathbf{x}^{k},\{ q_{s}^{k}\}, \{{\boldsymbol{\lambda}}_{l}^{k}\},\{{\boldsymbol{\alpha}}_{l}^{k}\},\{{\boldsymbol{\beta}}_{l}^{k}\}, \{\gamma_{s}^{k}\} )\} \\
 \nabla_{\mathbf{x}}
\widetilde{L}_{p} (\{\boldsymbol{\mu}^{k}_{l}\}, \{ \mathbf{r}_{l}^{k}\},\{ \mathbf{p}_{l}^{k}\},\mathbf{x}^{k},\{ q_{s}^{k}\}, \{{\boldsymbol{\lambda}}_{l}^{k}\},\{{\boldsymbol{\alpha}}_{l}^{k}\},\{{\boldsymbol{\beta}}_{l}^{k}\}, \{\gamma_{s}^{k}\} )\\
\{ \nabla_{{\boldsymbol{\lambda}}_{l}}
\widetilde{L}_{p} (\{\boldsymbol{\mu}^{k}_{l}\}, \{ \mathbf{r}_{l}^{k}\},\{ \mathbf{p}_{l}^{k}\},\mathbf{x}^{k},\{ q_{s}^{k}\}, \{{\boldsymbol{\lambda}}_{l}^{k}\},\{{\boldsymbol{\alpha}}_{l}^{k}\},\{{\boldsymbol{\beta}}_{l}^{k}\}, \{\gamma_{s}^{k}\} )\} \\
\{ \nabla_{{\boldsymbol{\alpha}}_{l}}
\widetilde{L}_{p} (\{\boldsymbol{\mu}^{k}_{l}\}, \{ \mathbf{r}_{l}^{k}\},\{ \mathbf{p}_{l}^{k}\},\mathbf{x}^{k},\{ q_{s}^{k}\}, \{{\boldsymbol{\lambda}}_{l}^{k}\},\{{\boldsymbol{\alpha}}_{l}^{k}\},\{{\boldsymbol{\beta}}_{l}^{k}\}, \{\gamma_{s}^{k}\} )\} \\
\{ \nabla_{{\boldsymbol{\beta}}_{l}}
\widetilde{L}_{p} (\{\boldsymbol{\mu}^{k}_{l}\}, \{ \mathbf{r}_{l}^{k}\},\{ \mathbf{p}_{l}^{k}\},\mathbf{x}^{k},\{ q_{s}^{k}\}, \{{\boldsymbol{\lambda}}_{l}^{k}\},\{{\boldsymbol{\alpha}}_{l}^{k}\},\{{\boldsymbol{\beta}}_{l}^{k}\}, \{\gamma_{s}^{k}\} )\} \\
\{ \nabla_{\gamma_{s}}
\widetilde{L}_{p} (\{\boldsymbol{\mu}^{k}_{l}\}, \{ \mathbf{r}_{l}^{k}\},\{ \mathbf{p}_{l}^{k}\},\mathbf{x}^{k},\{ q_{s}^{k}\}, \{{\boldsymbol{\lambda}}_{l}^{k}\},\{{\boldsymbol{\alpha}}_{l}^{k}\},\{{\boldsymbol{\beta}}_{l}^{k}\}, \{\gamma_{s}^{k}\} )\} \\
\end{bmatrix},
\end{equation}
with
\begin{equation}
\begin{aligned}
&(\nabla \widetilde{G}^k)_{\boldsymbol{\mu}_l} = \nabla_{\boldsymbol{\mu}_{l}}
\widetilde{L}_{p} (\{\boldsymbol{\mu}^{k}_{l}\}, \{ \mathbf{r}_{l}^{k}\},\{ \mathbf{p}_{l}^{k}\},\{ q_{s}^{k}\},\mathbf{x}^{k}, \{{\boldsymbol{\lambda}}_{l}^{k}\},\{{\boldsymbol{\alpha}}_{l}^{k}\},\{{\boldsymbol{\beta}}_{l}^{k}\}, \{\gamma_{s}^{k}\} ), \\
&(\nabla \widetilde{G}^k)_{\mathbf{r}_l} = \nabla_{\mathbf{r}_{l}}
\widetilde{L}_{p} (\{\boldsymbol{\mu}^{k}_{l}\}, \{ \mathbf{r}_{l}^{k}\},\{ \mathbf{p}_{l}^{k}\},\{ q_{s}^{k}\},\mathbf{x}^{k}, \{{\boldsymbol{\lambda}}_{l}^{k}\},\{{\boldsymbol{\alpha}}_{l}^{k}\},\{{\boldsymbol{\beta}}_{l}^{k}\}, \{\gamma_{s}^{k}\} ), \\
&(\nabla \widetilde{G}^k)_{\mathbf{p}_l} = \nabla_{\mathbf{p}_{l}}
\widetilde{L}_{p} (\{\boldsymbol{\mu}^{k}_{l}\}, \{ \mathbf{r}_{l}^{k}\},\{ \mathbf{p}_{l}^{k}\},\{ q_{s}^{k}\}, \mathbf{x}^{k},\{{\boldsymbol{\lambda}}_{l}^{k}\},\{{\boldsymbol{\alpha}}_{l}^{k}\},\{{\boldsymbol{\beta}}_{l}^{k}\}, \{\gamma_{s}^{k}\} ), \\
&(\nabla \widetilde{G}^k)_{q_s} = \nabla_{ q_s }
\widetilde{L}_{p} (\{\boldsymbol{\mu}^{k}_{l}\}, \{ \mathbf{r}_{l}^{k}\},\{ \mathbf{p}_{l}^{k}\},\{ q_{s}^{k}\}, \mathbf{x}^{k},\{{\boldsymbol{\lambda}}_{l}^{k}\},\{{\boldsymbol{\alpha}}_{l}^{k}\},\{{\boldsymbol{\beta}}_{l}^{k}\}, \{\gamma_{s}^{k}\} ), \\
&(\nabla \widetilde{G}^k)_{\mathbf{x}} = \nabla_{ \mathbf{x}}
\widetilde{L}_{p} (\{\boldsymbol{\mu}^{k}_{l}\}, \{ \mathbf{r}_{l}^{k}\},\{ \mathbf{p}_{l}^{k}\},\{ q_{s}^{k}\}, \mathbf{x}^{k},\{{\boldsymbol{\lambda}}_{l}^{k}\},\{{\boldsymbol{\alpha}}_{l}^{k}\},\{{\boldsymbol{\beta}}_{l}^{k}\}, \{\gamma_{s}^{k}\} ), \\
&(\nabla \widetilde{G}^k)_{{\boldsymbol{\lambda}}_{l}} = \nabla_{{\lambda}_{l}}
\widetilde{L}_{p} (\{\boldsymbol{\mu}^{k}_{l}\}, \{ \mathbf{r}_{l}^{k}\},\{ \mathbf{p}_{l}^{k}\},\{ q_{s}^{k}\}, \mathbf{x}^{k},\{{\boldsymbol{\lambda}}_{l}^{k}\},\{{\boldsymbol{\alpha}}_{l}^{k}\},\{{\boldsymbol{\beta}}_{l}^{k}\}, \{\gamma_{s}^{k}\} ), \\
&(\nabla \widetilde{G}^k)_{{\boldsymbol{\alpha}}_{l}} = \nabla_{{\alpha}_{l}}
\widetilde{L}_{p} (\{\boldsymbol{\mu}^{k}_{l}\}, \{ \mathbf{r}_{l}^{k}\},\{ \mathbf{p}_{l}^{k}\},\{ q_{s}^{k}\}, \mathbf{x}^{k},\{{\boldsymbol{\lambda}}_{l}^{k}\},\{{\boldsymbol{\alpha}}_{l}^{k}\},\{{\boldsymbol{\beta}}_{l}^{k}\}, \{\gamma_{s}^{k}\} ), \\
&(\nabla \widetilde{G}^k)_{{\boldsymbol{\beta}}_{l}} = \nabla_{{\beta}_{l}}
\widetilde{L}_{p} (\{\boldsymbol{\mu}^{k}_{l}\}, \{ \mathbf{r}_{l}^{k}\},\{ \mathbf{p}_{l}^{k}\},\{ q_{s}^{k}\}, \mathbf{x}^{k},\{{\boldsymbol{\lambda}}_{l}^{k}\},\{{\boldsymbol{\alpha}}_{l}^{k}\},\{{\boldsymbol{\beta}}_{l}^{k}\}, \{\gamma_{s}^{k}\} ), \\
&(\nabla \widetilde{G}^k)_{{\gamma}_{s}} = \nabla_{{\gamma}_{s}}
\widetilde{L}_{p} (\{\boldsymbol{\mu}^{k}_{l}\}, \{ \mathbf{r}_{l}^{k}\},\{ \mathbf{p}_{l}^{k}\},\{ q_{s}^{k}\}, \mathbf{x}^{k},\{{\boldsymbol{\lambda}}_{l}^{k}\},\{{\boldsymbol{\alpha}}_{l}^{k}\},\{{\boldsymbol{\beta}}_{l}^{k}\}, \{\gamma_{s}^{k}\} ).\\
\end{aligned}
\end{equation}
\label{defnition_3}
\end{definition}
\begin{definition}
In the $k^{th}$ iteration of our algorithm, we define the last iteration in which the $l^{th}$ worker was active as $\hat{k}_{l}$, and the next iteration in which the $l^{th}$ worker will be active as $\overline{k}_{l}$.
Furthermore, we represent the set of iteration indices in which the $l^{th}$ worker is active during the $K_{1} + K + \tau$ iteration as $\mathcal{V}_{l}(K)$. And the $j^{th}$ element in $\mathcal{V}_{l}(K)$ is represented as $\hat{v}_{l}(j)$.
\label{definition_4}
\end{definition}
Based on the above definitions, we next provide the proof of Lemma \ref{lemma_1}, Lemma \ref{lemma_2} and Lemma \ref{lemma_3}.
\newpage
\begin{lemma}
\label{lemma_1}
According to Eq. (\ref{ALM_2}), function $L_p$
has Lipschitz continuous Hessian and let $L_w$ denote the Lipschitz constant.
Based on definition of $\eta_{\boldsymbol{\mu}}$, $\eta_{\mathbf{x}}$, $ \eta_{\mathbf{r}}$, $\eta_{\mathbf{p}}$ and  $\eta_{q}$, we set  $\eta_{\boldsymbol{\mu}}^k$, $\eta_{\mathbf{x}}^k$ , $ \eta_{\mathbf{r}}^k$, $\eta_{\mathbf{p}}^k$ and  $\eta_{q}^k$ to be: 
\begin{equation}
\begin{aligned}
& \eta_{\boldsymbol{\mu}}^k = \frac{1}{ \frac{8\xi}{\eta_{{\boldsymbol{\lambda}}}( c_{{\boldsymbol{\lambda}}}^k)^2}+\frac{8\xi}{\eta_{{\boldsymbol{\alpha}}}(c_{{\boldsymbol{\alpha}}}^k )^2}+\frac{8\xi}{\eta_{{\boldsymbol{\beta}}}( c_{{\boldsymbol{\beta}}}^k)^2  }   }, \quad
\eta_{\mathbf{x}}^k = \frac{1}{ \frac{8\xi}{\eta_{{\boldsymbol{\beta}}}( c_{{\boldsymbol{\beta}}}^k )^2}  + 
\frac{8\xi}{\eta_{{\gamma}}( c_{{\gamma}}^k )^2}
 \sum_{s=1}^{|\mathcal{P}^{k}|}\|b_{s}\|^2 },\\
& \eta_{\mathbf{r}}^k = \frac{\eta_{{\boldsymbol{\lambda}}}(  c_{{\boldsymbol{\lambda}}}^k)^2}{  32w^2_{\mathbf{r}}\xi},\quad \eta_{\mathbf{p}}^k = \frac{\eta_{{\boldsymbol{\alpha}}}( c_{{\boldsymbol{\alpha}}}^k)^2}{  32w^2_{\mathbf{p}}\xi}, \quad\eta_{q}^k = \frac{\eta_{\gamma}(  c_{\gamma}^k)^2}{  32w^2_{q}\xi} .\\
\end{aligned}
\end{equation}
We can obtain that, 
\begin{align}
&{L}_{p} (\{\boldsymbol{\mu}^{k+1}_{l}\}, \{ \mathbf{r}_{l}^{k+1}\},\{ \mathbf{p}_{l}^{k+1}\},\{ q_{s}^{k+1}\},\mathbf{x}^{k+1}, \{{\boldsymbol{\lambda}}_{l}^{k}\},\{{\boldsymbol{\alpha}}_{l}^{k}\},\{{\boldsymbol{\beta}}_{l}^{k}\}, \{\gamma_{s}^{k}\} ) \nonumber \\
& -{L}_{p} (\{\boldsymbol{\mu}^{k}_{l}\}, \{ \mathbf{r}_{l}^{k}\},\{ \mathbf{p}_{l}^{k}\},\{ q_{s}^{k}\},\mathbf{x}^{k}, \{{\boldsymbol{\lambda}}_{l}^{k}\},\{{\boldsymbol{\alpha}}_{l}^{k}\},\{{\boldsymbol{\beta}}_{l}^{k}\}, \{\gamma_{s}^{k}\} ) 
 \\
& \leq
 ( \frac{ w_{\boldsymbol{\mu}} L_{w}M^{\frac{1}{2}}  }{3} + 1 - \frac{1}{\eta_{\boldsymbol{\mu}}^k}  )\sum_{l=1}^{L}\|\boldsymbol{\mu}_{l}^{k+1} -\boldsymbol{\mu}_{l}^{k}  \|^2  
+ (  - \frac{1}{\eta_{\mathbf{r}}^k} + \frac{ w_{\mathbf{r}} L_{w}M^{\frac{1}{2}}  }{3}  + w_{{\boldsymbol{\lambda}}} )\sum_{l=1}^{L} \|\mathbf{r}_{l}^{k+1} -\mathbf{r}^{k}_{l}  \|^2
\nonumber \\
& + (-\frac{1}{\eta_{\mathbf{p}}^{k}} +\frac{w_{\mathbf{p}}L_{w}M^{\frac{1}{2}}}{3} + w_{{\boldsymbol{\alpha}}} )\sum_{l=1}^L||\mathbf{p}_{l}^{k+1}-\mathbf{p}_{l}^{k}||^{2} + (-\frac{1}{\eta_{\mathbf{q}}^{k}} +\frac{w_{\mathbf{q}}L_{w}}{3} + w_{\gamma} )\sum_{s=1}^{|\mathcal{P}^{k}|}||{q}_{s}^{k+1}-q_{s}^{k}||^{2} \nonumber \\
& + ( \frac{L_{w}}{6} - \frac{1}{\eta^k_{\mathbf{x}}}) \|\mathbf{x}^{k+1} -\mathbf{x}^{k} \|^2 \nonumber .
\end{align}
\end{lemma}
\begin{proof}
According to the Lipschitz property of $L_{p}$,
we can obtain that,
\begin{align}
\label{lm_0}
& {L}_{p} (\{\boldsymbol{\mu}^{k+1}_{l}\}, \{ \mathbf{r}_{l}^{k}\},\{ \mathbf{p}_{l}^{k}\},\{ q_{s}^{k}\},\mathbf{x}^{k}, \{{\boldsymbol{\lambda}}_{l}^{k}\},\{{\boldsymbol{\alpha}}_{l}^{k}\},\{{\boldsymbol{\beta}}_{l}^{k}\}, \{\gamma_{s}^{k}\} ) 
\nonumber \\
& - {L}_{p} (\{\boldsymbol{\mu}^{k}_{l}\}, \{ \mathbf{r}_{l}^{k}\},\{ \mathbf{p}_{l}^{k}\},\{ q_{s}^{k}\},\mathbf{x}^{k}, \{{\boldsymbol{\lambda}}_{l}^{k}\},\{{\boldsymbol{\alpha}}_{l}^{k}\},\{{\boldsymbol{\beta}}_{l}^{k}\}, \{\gamma_{s}^{k}\} ) \nonumber \\
& \leq\sum\limits_{l=1}^{L}\left(\left\langle\nabla_{\boldsymbol{\mu}_{l}}{L}_{p} (\{\boldsymbol{\mu}^{k}_{l}\}, \{ \mathbf{r}_{l}^{k}\},\{ \mathbf{p}_{l}^{k}\},\{ q_{s}^{k}\},\mathbf{x}^{k}, \{{\boldsymbol{\lambda}}_{l}^{k}\},\{{\boldsymbol{\alpha}}_{l}^{k}\},\{{\boldsymbol{\beta}}_{l}^{k}\}, \{\gamma_{s}^{k}\} ),\boldsymbol{\mu}_{l}^{k+1}-\boldsymbol{\mu}_{l}^{k}\right\rangle+\frac{L_{w}}{6}||\boldsymbol{\mu}_{l}^{k+1}-\boldsymbol{\mu}_{l}^{k}||^{3}\right) \\
&+\frac{1}{2}\sum\limits_{l=1}^{L}\left\langle\nabla^2_{\boldsymbol{\mu}_{l}}{L}_{p} (\{\boldsymbol{\mu}^{k}_{l}\}, \{ \mathbf{r}_{l}^{k}\},\{ \mathbf{p}_{l}^{k}\},\{ q_{s}^{k}\},\mathbf{x}^{k}, \{{\boldsymbol{\lambda}}_{l}^{k}\},\{{\boldsymbol{\alpha}}_{l}^{k}\},\{{\boldsymbol{\beta}}_{l}^{k}\}, \{\gamma_{s}^{k}\} )(\boldsymbol{\mu}_{l}^{k+1}-\boldsymbol{\mu}_{l}^{k}),\boldsymbol{\mu}_{l}^{k+1}-\boldsymbol{\mu}_{l}^{k}\right\rangle.  \nonumber 
\end{align}
By employing (\ref{mu_update}),  since $\eta_{\boldsymbol{\mu}} \leq \eta_{\boldsymbol{\mu}}^{k}$ for $\forall k$, we have that,
\begin{align}
\label{lm_1}
& \left\langle\nabla_{\boldsymbol{\mu}_{l}}{L}_{p} (\{\boldsymbol{\mu}^{\hat{k}_{l}}_{l}\}, \{ \mathbf{r}_{l}^{\hat{k}_{l}}\},\{ \mathbf{p}_{l}^{\hat{k}_{l}}\},\{ q_{s}^{\hat{k}_{l}}\},\mathbf{x}^{\hat{k}_{l}}, \{{\boldsymbol{\lambda}}_{l}^{\hat{k}_{l}}\},\{{\boldsymbol{\alpha}}_{l}^{\hat{k}_{l}}\},\{{\boldsymbol{\beta}}_{l}^{\hat{k}_{l}}\}, \{\gamma_{s}^{\hat{k}_{l}}\} ),\boldsymbol{\mu}_{l}^{k+1}-\boldsymbol{\mu}_{l}^{k}\right\rangle 
 \leq -\frac{1}{\eta_{\boldsymbol{\mu}}^{k}}||\boldsymbol{\mu}_{l}^{k+1}-\boldsymbol{\mu}_{l}^{k}||^{2}.
\end{align}
Given that $\boldsymbol{\mu}^{k}_{l} = \boldsymbol{\mu}^{\hat{k}_{l}}_{l}$, ${\boldsymbol{\lambda}}^{k}_{l} = {\boldsymbol{\lambda}}^{\hat{k}_{l}}_{l}$, ${\boldsymbol{\alpha}}^{k}_{l} = {\boldsymbol{\alpha}}^{\hat{k}_{l}}_{l}$, and ${\boldsymbol{\beta}}^{k}_{l} = {\boldsymbol{\beta}}^{\hat{k}_{l}}_{l}$, 
according to the function $L_p$ provided in (\ref{ALM_2}),
we have
\begin{equation}
\left\langle\nabla_{\boldsymbol{\mu}_{l}}{L}_{p} (\{\boldsymbol{\mu}^{k}_{l}\}, \{ \mathbf{r}_{l}^{k}\},\{ \mathbf{p}_{l}^{k}\},\{ q_{s}^{k}\},\mathbf{x}^{k}, \{{\boldsymbol{\lambda}}_{l}^{k}\},\{{\boldsymbol{\alpha}}_{l}^{k}\},\{{\boldsymbol{\beta}}_{l}^{k}\}, \{\gamma_{s}^{k}\} ),\boldsymbol{\mu}_{l}^{k+1}-\boldsymbol{\mu}_{l}^{k}\right\rangle 
\leq-\frac{1}{\eta_{\boldsymbol{\mu}}^{k}}||\boldsymbol{\mu}_{l}^{k+1}-\boldsymbol{\mu}_{l}^{k}||^{2}.
\label{lm_3}
\end{equation}
It can also be seen from (\ref{ALM_2}) that
\begin{equation}
\resizebox{0.90\linewidth}{!}{$  
\frac{1}{2}\left\langle\nabla^2_{\boldsymbol{\mu}_{l}}{L}_{p} (\{\boldsymbol{\mu}^{k}_{l}\}, \{ \mathbf{r}_{l}^{k}\},\{ \mathbf{p}_{l}^{k}\},\{ q_{s}^{k}\},\mathbf{x}^{k}, \{{\boldsymbol{\lambda}}_{l}^{k}\},\{{\boldsymbol{\alpha}}_{l}^{k}\},\{{\boldsymbol{\beta}}_{l}^{k}\}, \{\gamma_{s}^{k}\} )(\boldsymbol{\mu}_{l}^{k+1}-\boldsymbol{\mu}_{l}^{k}),\boldsymbol{\mu}_{l}^{k+1}-\boldsymbol{\mu}_{l}^{k}\right\rangle  
= \|\boldsymbol{\mu}_{l}^{k+1}-\boldsymbol{\mu}_{l}^{k}\|^2.
$}
\label{o_1}
\end{equation}
Based on the upper bound of $\|\boldsymbol{\mu}_{l}\|_{\infty}$ provided in  
\cref{assumption_1},
by employing trigonometric inequality, we have
\begin{equation}
\label{mu_add}
\frac{L_{w}}{6}\sum_{l=1}^L||\boldsymbol{\mu}_{l}^{k+1}-\boldsymbol{\mu}_{l}^{k}||^{3}  \leq
\frac{w_{\boldsymbol{\mu}}L_{w}M^{\frac{1}{2}}}{3}\sum_{l=1}^L||\boldsymbol{\mu}_{l}^{k+1}-\boldsymbol{\mu}_{l}^{k}||^{2}.
\end{equation}
Combining (\ref{lm_0}), (\ref{lm_3}), (\ref{o_1}) with (\ref{mu_add}), we can obtain that, 
\begin{align}
\label{lm_4}
& {L}_{p} (\{\boldsymbol{\mu}^{k+1}_{l}\}, \{ \mathbf{r}_{l}^{k}\},\{ \mathbf{p}_{l}^{k}\},\{ q_{s}^{k}\},\mathbf{x}^{k}, \{{\boldsymbol{\lambda}}_{l}^{k}\},\{{\boldsymbol{\alpha}}_{l}^{k}\},\{{\boldsymbol{\beta}}_{l}^{k}\}, \{\gamma_{s}^{k}\} ) \nonumber \\
& -{L}_{p} (\{\boldsymbol{\mu}^{k}_{l}\}, \{ \mathbf{r}_{l}^{k}\},\{ \mathbf{p}_{l}^{k}\},\{ q_{s}^{k}\},\mathbf{x}^{k}, \{{\boldsymbol{\lambda}}_{l}^{k}\},\{{\boldsymbol{\alpha}}_{l}^{k}\},\{{\boldsymbol{\beta}}_{l}^{k}\}, \{\gamma_{s}^{k}\} ) \nonumber \\
& \leq( -\frac{1}{\eta_{\boldsymbol{\mu}}^{k}}+
\frac{ w_{\boldsymbol{\mu}} L_{w}M^{\frac{1}{2}}  }{3}
+ 1)\sum\limits_{l=1}^{L}\left(  ||\boldsymbol{\mu}_{l}^{k+1}-\boldsymbol{\mu}_{l}^{k}||^{2}  \right). 
\end{align}
According to the Lipschitz property of $L_p$, we can obtain that,
\begin{align}
\label{1_r_1}
& {L}_{p} (\{\boldsymbol{\mu}^{k+1}_{l}\}, \{ \mathbf{r}_{l}^{k+1}\},\{ \mathbf{p}_{l}^{k}\},\{ q_{s}^{k}\},\mathbf{x}^{k}, \{{\boldsymbol{\lambda}}_{l}^{k}\},\{{\boldsymbol{\alpha}}_{l}^{k}\},\{{\boldsymbol{\beta}}_{l}^{k}\}, \{\gamma_{s}^{k}\} )\}
\nonumber \\
& -{L}_{p} (\{\boldsymbol{\mu}^{k+1}_{l}\}, \{ \mathbf{r}_{l}^{k}\},\{ \mathbf{p}_{l}^{k}\},\{ q_{s}^{k}\},\mathbf{x}^{k}, \{{\boldsymbol{\lambda}}_{l}^{k}\},\{{\boldsymbol{\alpha}}_{l}^{k}\},\{{\boldsymbol{\beta}}_{l}^{k}\}, \{\gamma_{s}^{k}\} )\} \nonumber \\
& \leq\sum\limits_{l=1}^{L}\left\langle\nabla_{\mathbf{r}_{l}}{L}_{p} (\{\boldsymbol{\mu}^{k+1}_{l}\}, \{ \mathbf{r}_{l}^{k}\},\{ \mathbf{p}_{l}^{k}\},\{ q_{s}^{k}\},\mathbf{x}^{k}, \{{\boldsymbol{\lambda}}_{l}^{k}\},\{{\boldsymbol{\alpha}}_{l}^{k}\},\{{\boldsymbol{\beta}}_{l}^{k}\}, \{\gamma_{s}^{k}\} ),\mathbf{r}_{l}^{k+1}-\mathbf{r}_{l}^{k}\right\rangle   +\frac{L_{w}}{6}\sum_{l=1}^L||\mathbf{r}_{l}^{k+1}-\mathbf{r}_{l}^{k}||^{3}  \\
&+\frac{1}{2}\sum\limits_{l=1}^{L}\left\langle\nabla^2_{\mathbf{r}_{l}}{L}_{p} (\{\boldsymbol{\mu}^{k+1}_{l}\}, \{ \mathbf{r}_{l}^{k}\},\{ \mathbf{p}_{l}^{k}\},\{ q_{s}^{k}\},\mathbf{x}^{k}, \{{\boldsymbol{\lambda}}_{l}^{k}\},\{{\boldsymbol{\alpha}}_{l}^{k}\},\{{\boldsymbol{\beta}}_{l}^{k}\}, \{\gamma_{s}^{k}\} )(\mathbf{r}_{l}^{k+1}-\mathbf{r}_{l}^{k}),\mathbf{r}_{l}^{k+1}-\mathbf{r}_{l}^{k}\right\rangle, \nonumber
\end{align}
Similar to (\ref{lm_1})-(\ref{lm_3}), we have
\begin{equation}
\label{1_r_2}
\left\langle\nabla_{\mathbf{r}_{l}}{L}_{p} (\{\boldsymbol{\mu}^{k+1}_{l}\}, \{ \mathbf{r}_{l}^{k}\},\{ \mathbf{p}_{l}^{k}\},\{ q_{s}^{k}\},\mathbf{x}^{k}, \{{\boldsymbol{\lambda}}_{l}^{k}\},\{{\boldsymbol{\alpha}}_{l}^{k}\},\{{\boldsymbol{\beta}}_{l}^{k}\}, \{\gamma_{s}^{k}\} ),\mathbf{r}_{l}^{k+1}-\mathbf{r}_{l}^{k}\right\rangle 
\leq-\frac{1}{\eta_{\mathbf{r}}^{k}}||\mathbf{r}_{l}^{k+1}-\mathbf{r}_{l}^{k}||^{2}.
\end{equation}
Based on \cref{assumption_1}, we can obtain that,
\begin{equation}
\resizebox{0.9\linewidth}{!}{$ 
\frac{1}{2}\left(\left\langle\nabla^2_{\mathbf{r}_{l}}{L}_{p} (\{\boldsymbol{\mu}^{k+1}_{l}\}, \{ \mathbf{r}_{l}^{k}\},\{ \mathbf{p}_{l}^{k}\},\{ q_{s}^{k}\},\mathbf{x}^{k}, \{{\boldsymbol{\lambda}}_{l}^{k}\},\{{\boldsymbol{\alpha}}_{l}^{k}\},\{{\boldsymbol{\beta}}_{l}^{k}\}, \{\gamma_{s}^{k}\} )(\mathbf{r}_{l}^{k+1}-\mathbf{r}_{l}^{k}),\mathbf{r}_{l}^{k+1}-\mathbf{r}_{l}^{k}\right\rangle \right) 
\\
\leq w_{{\boldsymbol{\lambda}}}||\mathbf{r}_{l}^{k+1}-\mathbf{r}_{l}^{k}||^{2}.
$}
\end{equation}
By employing trigonometric inequality, we have that,
\begin{equation}
\label{1_r_3}
\frac{L_{w}}{6}\sum_{l=1}^L||\mathbf{r}_{l}^{k+1}-\mathbf{r}_{l}^{k}||^{3}  \leq
\frac{w_{\mathbf{r}}L_{w}M^{\frac{1}{2}}}{3}\sum_{l=1}^L||\mathbf{r}_{l}^{k+1}-\mathbf{r}_{l}^{k}||^{2}.
\end{equation}
Following (\ref{1_r_1})-(\ref{1_r_3}), we can obtain that, 
\begin{align}
\label{1_r_4}
& {L}_{p} (\{\boldsymbol{\mu}^{k+1}_{l}\}, \{ \mathbf{r}_{l}^{k+1}\},\{ \mathbf{p}_{l}^{k}\},\{ q_{s}^{k}\},\mathbf{x}^{k}, \{{\boldsymbol{\lambda}}_{l}^{k}\},\{{\boldsymbol{\alpha}}_{l}^{k}\},\{{\boldsymbol{\beta}}_{l}^{k}\}, \{\gamma_{s}^{k}\} )
\nonumber \\
& -{L}_{p} (\{\boldsymbol{\mu}^{k+1}_{l}\}, \{ \mathbf{r}_{l}^{k}\},\{ \mathbf{p}_{l}^{k}\},\{ q_{s}^{k}\},\mathbf{x}^{k}, \{{\boldsymbol{\lambda}}_{l}^{k}\},\{{\boldsymbol{\alpha}}_{l}^{k}\},\{{\boldsymbol{\beta}}_{l}^{k}\}, \{\gamma_{s}^{k}\} ) \nonumber \\
& \leq (-\frac{1}{\eta_{\mathbf{r}}^{k}} +\frac{w_{\mathbf{r}}L_{w}M^{\frac{1}{2}}}{3} + w_{{\boldsymbol{\lambda}}} )\sum_{l=1}^L||\mathbf{r}_{l}^{k+1}-\mathbf{r}_{l}^{k}||^{2}. 
\end{align}
Likewise, similar results can be obtained for the variable   $\mathbf{p}_{l}$ and $q_{s}$.
\begin{align}
\label{1_p_1}
& 
{L}_{p} (\{\boldsymbol{\mu}^{k+1}_{l}\}, \{ \mathbf{r}_{l}^{k+1}\},\{ \mathbf{p}_{l}^{k+1}\},\{ q_{s}^{k}\},\mathbf{x}^{k}, \{{\boldsymbol{\lambda}}_{l}^{k}\},\{{\boldsymbol{\alpha}}_{l}^{k}\},\{{\boldsymbol{\beta}}_{l}^{k}\}, \{\gamma_{s}^{k}\} )
\nonumber \\
& -{L}_{p} (\{\boldsymbol{\mu}^{k+1}_{l}\}, \{ \mathbf{r}_{l}^{k+1}\},\{ \mathbf{p}_{l}^{k}\},\{ q_{s}^{k}\},\mathbf{x}^{k}, \{{\boldsymbol{\lambda}}_{l}^{k}\},\{{\boldsymbol{\alpha}}_{l}^{k}\},\{{\boldsymbol{\beta}}_{l}^{k}\}, \{\gamma_{s}^{k}\} ) 
\nonumber \\
& \leq (-\frac{1}{\eta_{\mathbf{p}}^{k}} +\frac{w_{\mathbf{p}}L_{w}M^{\frac{1}{2}}}{3} + w_{{\boldsymbol{\alpha}}} )\sum_{l=1}^L||\mathbf{p}_{l}^{k+1}-\mathbf{p}_{l}^{k}||^{2}. 
\end{align}
\begin{align}
\label{1_q_1}
& 
{L}_{p} (\{\boldsymbol{\mu}^{k+1}_{l}\}, \{ \mathbf{r}_{l}^{k+1}\},\{ \mathbf{p}_{l}^{k+1}\},\{ q_{s}^{k+1}\},\mathbf{x}^{k}, \{{\boldsymbol{\lambda}}_{l}^{k}\},\{{\boldsymbol{\alpha}}_{l}^{k}\},\{{\boldsymbol{\beta}}_{l}^{k}\}, \{\gamma_{s}^{k}\} )
\nonumber \\
& -{L}_{p} (\{\boldsymbol{\mu}^{k+1}_{l}\}, \{ \mathbf{r}_{l}^{k+1}\},\{ \mathbf{p}_{l}^{k+1}\},\{ q_{s}^{k}\},\mathbf{x}^{k}, \{{\boldsymbol{\lambda}}_{l}^{k}\},\{{\boldsymbol{\alpha}}_{l}^{k}\},\{{\boldsymbol{\beta}}_{l}^{k}\}, \{\gamma_{s}^{k}\} )
\nonumber \\
& \leq (-\frac{1}{\eta_{\mathbf{q}}^{k}} +\frac{w_{\mathbf{q}}L_{w}}{3} + w_{\gamma} )\sum_{s=1}^{|\mathcal{P}^{k}|}||{q}_{s}^{k+1}-q_{s}^{k}||^{2}. 
\end{align}
For the variable $\mathbf{x}$, define a constant $\nu$ such that 
$\nu \geq \sum_{s=1}^{|\mathcal{P}^k|}\|\mathbf{b}_s\|^2$. Consequently, we have $\eta_{\mathbf{x}} \leq \eta_{\mathbf{x}}^{k}$ for all $k$. Similarly to (\ref{1_r_1})-(\ref{1_r_3}), we can obtain that,
\begin{align}
\label{1_v_1}
& 
{L}_{p} (\{\boldsymbol{\mu}^{k+1}_{l}\}, \{ \mathbf{r}_{l}^{k+1}\},\{ \mathbf{p}_{l}^{k+1}\},\{ q_{s}^{k+1}\},\mathbf{x}^{k+1}, \{{\boldsymbol{\lambda}}_{l}^{k}\},\{{\boldsymbol{\alpha}}_{l}^{k}\},\{{\boldsymbol{\beta}}_{l}^{k}\}, \{\gamma_{s}^{k}\} ) \nonumber \\
& -{L}_{p} (\{\boldsymbol{\mu}^{k+1}_{l}\}, \{ \mathbf{r}_{l}^{k+1}\},\{ \mathbf{p}_{l}^{k+1}\},\{ q_{s}^{k+1}\},\mathbf{x}^{k}, \{{\boldsymbol{\lambda}}_{l}^{k}\},\{{\boldsymbol{\alpha}}_{l}^{k}\},\{{\boldsymbol{\beta}}_{l}^{k}\}, \{\gamma_{s}^{k}\} ) 
 \nonumber
\\
& \leq   ( -\frac{1}{\eta_{\mathbf{x}}^{k}}+\frac{L_{w}}{6})||\mathbf{x}^{k+1}-\mathbf{x}^{k}||^{2} . 
\end{align}
Combining (\ref{lm_4}), (\ref{1_r_4}), (\ref{1_p_1}), (\ref{1_q_1}) and (\ref{1_v_1}), we conclude the proof of Lemma \ref{lemma_1}.
\end{proof}

\begin{lemma}
\label{lemma_2}
$\forall k \geq K_{1}$, we have
\begin{align}
& {L}_{p} (\{\boldsymbol{\mu}^{k+1}_{l}\}, \{ \mathbf{r}_{l}^{k+1}\},\{ \mathbf{p}_{l}^{k+1}\},\{ q_{s}^{k+1}\},\mathbf{x}^{k+1}, \{{\boldsymbol{\lambda}}_{l}^{k+1}\},\{{\boldsymbol{\alpha}}_{l}^{k+1}\},\{{\boldsymbol{\beta}}_{l}^{k+1}\}, \{\gamma_{s}^{k+1}\} ) \nonumber \\
& - {L}_{p} (\{\boldsymbol{\mu}^{k}_{l}\}, \{ \mathbf{r}_{l}^{k}\},\{ \mathbf{p}_{l}^{k}\},\{ q_{s}^{k}\},\mathbf{x}^{k}, \{{\boldsymbol{\lambda}}_{l}^{k}\},\{{\boldsymbol{\alpha}}_{l}^{k}\},\{{\boldsymbol{\beta}}_{l}^{k}\}, \{\gamma_{s}^{k}\} ) 
\nonumber \\
& \leq (\frac{1}{a_1} + \frac{1}{a_2} + \frac{1}{a_3} + \frac{ w_{\boldsymbol{\mu}} L_{w}M^{\frac{1}{2}}  }{3} + 1 -\frac{1}{\eta_{\boldsymbol{\mu}}^{k}})\sum_{l=1}^{L}\|\boldsymbol{\mu}_{l}^{k+1}-\boldsymbol{\mu}_{l}^{k}\|^2
\nonumber \\
&+  (\frac{4w_{\mathbf{r}}^2}{a_1}+\frac{w_{\mathbf{r}}L_{w}M^{\frac{1}{2}}}{3} + w_{{\boldsymbol{\lambda}}}-\frac{1}{\eta_{\mathbf{r}}^{k}} )\sum_{l=1}^{L}\|\mathbf{r}_{l}^{k+1}-\mathbf{r}_{l}^{k}\|^2 \nonumber \\
& + (\frac{4w_{\mathbf{p}}^2}{a_2}+\frac{w_{\mathbf{p}}L_{w}M^{\frac{1}{2}}}{3} + w_{{\boldsymbol{\alpha}}}-\frac{1}{\eta_{\mathbf{p}}^{k}} )\sum_{l=1}^{L}\|\mathbf{p}_{l}^{k+1}-\mathbf{p}_{l}^{k}\|^2 \nonumber \\
& + (\frac{4w_{q}^2}{a_4}+\frac{w_{q}L_{w}}{3} + w_{\gamma}-\frac{1}{\eta_{q}^{k}} )\sum_{s=1}^{|\mathcal{P}^k|}\|q_{s}^{k+1}-q_{s}^{k}\|^2  \nonumber \\
& + ( -\frac{1}{\eta_{\mathbf{x}}^{k}}+\frac{L_{w}}{6})||\mathbf{x}^{k+1}-\mathbf{x}^{k}||^{2} + \frac{1}{a_3}\sum_{l=1}^{L}\|\mathbf{B}_{l}\mathbf{x}^{k+1}-\mathbf{B}_{l}\mathbf{x}^{k}\|^2 + \frac{1}{a_4}\sum_{s=1}^{|\mathcal{P}^k|}\|\mathbf{b}_s\|^2\|\mathbf{x}^{k+1}-\mathbf{x}^{k}\|^2  \nonumber \\
& +  (\frac{a_1}{2} -  \frac{c_{{\boldsymbol{\lambda}}}^{k-1}-c_{{\boldsymbol{\lambda}}}^k}{2}+\frac{1}{2\eta_{{\boldsymbol{\lambda}}}})\sum_{l=1}^{L}\|{\boldsymbol{\lambda}}_{l}^{k+1}-{\boldsymbol{\lambda}}_{l}^{k}\|^2
+ \frac{c_{{\boldsymbol{\lambda}}}^{k-1}}{2}\sum_{l=1}^{L}( \|{\boldsymbol{\lambda}}_{l}^{k+1}\|^2 - \|{\boldsymbol{\lambda}}_{l}^{k}\|^2  )  \nonumber \\
& + (\frac{a_2}{2} -  \frac{c_{{\boldsymbol{\alpha}}}^{k-1}-c_{{\boldsymbol{\alpha}}}^k}{2}+\frac{1}{2\eta_{{\boldsymbol{\alpha}}}})\sum_{l=1}^{L}\|{\boldsymbol{\alpha}}_{l}^{k+1}-{\boldsymbol{\alpha}}_{l}^{k}\|^2
 + \frac{c_{{\boldsymbol{\alpha}}}^{k-1}}{2}\sum_{l=1}^{L}( \|{\boldsymbol{\alpha}}_{l}^{k+1}\|^2 - \|{\boldsymbol{\alpha}}_{l}^{k}\|^2  )   \nonumber \\
 & + (\frac{a_3}{2} -  \frac{c_{{\boldsymbol{\beta}}}^{k-1}-c_{{\boldsymbol{\beta}}}^k}{2}+\frac{1}{2\eta_{{\boldsymbol{\beta}}}})\sum_{l=1}^{L}\|{\boldsymbol{\beta}}_{l}^{k+1}-{\boldsymbol{\beta}}_{l}^{k}\|^2
 + \frac{c_{{\boldsymbol{\beta}}}^{k-1}}{2}\sum_{l=1}^{L}( \|{\boldsymbol{\beta}}_{l}^{k+1}\|^2 - \|{\boldsymbol{\beta}}_{l}^{k}\|^2  )   \nonumber \\
 & + (\frac{a_4}{2} -  \frac{c_{\gamma}^{k-1}-c_{\gamma}^k}{2}+\frac{1}{2\eta_{\gamma}})\sum_{s=1}^{|\mathcal{P}^k|}\|\gamma_{s}^{k+1}-\gamma_{s}^{k}\|^2
+ \frac{c_{\gamma}^{k-1}}{2}\sum_{s=1}^{|\mathcal{P}^k|}( \|\gamma_{s}^{k+1}\|^2 - \|\gamma_{s}^{k}\|^2  ) 
\nonumber \\
& + \frac{1}{2\eta_{{\boldsymbol{\lambda}}}}\sum_{l=1}^L\|{\boldsymbol{\lambda}}_{l}^{k}-{\boldsymbol{\lambda}}_{l}^{k-1}\|^2 
+ \frac{1}{2\eta_{{\boldsymbol{\alpha}}}}\sum_{l=1}^L\|{\boldsymbol{\alpha}}_{l}^{k}-{\boldsymbol{\alpha}}_{l}^{k-1}\|^2
\nonumber \\
& + \frac{1}{2\eta_{{\boldsymbol{\beta}}}}\sum_{l=1}^L\|{\boldsymbol{\beta}}_{l}^{k}-{\boldsymbol{\beta}}_{l}^{k-1}\|^2 + \frac{1}{2\eta_{\gamma}}\sum_{s=1}^{|\mathcal{P}^k|}\|\gamma_{s}^{k}-\gamma_{s}^{k-1}\|^2,
\end{align}
where $a_1$, $a_2$, $a_3$ and $a_4$ are positive constants.
\end{lemma}
\begin{proof}
According to (\ref{lambda_update}), in the $(k+1)^{th}$ iteration, for $\forall \boldsymbol{\lambda
}$ and $\forall l \in \mathcal{Q}^{k+1}$ it follows that, 
\begin{equation}
\left\langle {\boldsymbol{\lambda}}_l^{k+1}-{\boldsymbol{\lambda}}_l^k-\eta_{\boldsymbol{\lambda}} \nabla_{{\boldsymbol{\lambda}}_l}\widetilde{L}_{p} (\{\boldsymbol{\mu}^{k+1}_{l}\}, \{ \mathbf{r}_{l}^{k+1}\},\{ \mathbf{p}_{l}^{k+1}\},\{ q_{s}^{k+1}\},\mathbf{x}^{k+1}, \{{\boldsymbol{\lambda}}_{l}^{k}\},\{{\boldsymbol{\alpha}}_{l}^{k}\},\{{\boldsymbol{\beta}}_{l}^{k}\}, \{\gamma_{s}^{k}\} ),{\boldsymbol{\lambda}}-{\boldsymbol{\lambda}}_l^{k+1}\right\rangle  \geq    0.
\end{equation}
Let ${\boldsymbol{\lambda}} = {\boldsymbol{\lambda}}_{l}^{k}$
, we can obtain,
\vspace{-5pt}
\begin{equation}
\left\langle  \nabla_{{\boldsymbol{\lambda}}_l}\widetilde{L}_{p} (\{\boldsymbol{\mu}^{k+1}_{l}\}, \{ \mathbf{r}_{l}^{k+1}\},\{ \mathbf{p}_{l}^{k+1}\},\{ q_{s}^{k+1}\},\mathbf{x}^{k+1}, \{{\boldsymbol{\lambda}}_{l}^{k}\},\{{\boldsymbol{\alpha}}_{l}^{k}\},\{{\boldsymbol{\beta}}_{l}^{k}\}, \{\gamma_{s}^{k}\} ) -\frac{1}{
\eta_{\boldsymbol{\lambda}}}(
{\boldsymbol{\lambda}}_l^{k+1}-{\boldsymbol{\lambda}}_l^k),{\boldsymbol{\lambda}}_{l}^{k}-{\boldsymbol{\lambda}}_l^{k+1}\right\rangle \leq 0.
\label{2_lambda_0}
\end{equation}
Likewise, in the $k^{th}$ iteration, we have that, 
\vspace{-5pt}
\begin{equation}
\left\langle  \nabla_{{\boldsymbol{\lambda}}_l}\widetilde{L}_{p} (\{\boldsymbol{\mu}^{k}_{l}\}, \{ \mathbf{r}_{l}^{k}\},\{ \mathbf{p}_{l}^{k}\},\{ q_{s}^{k}\},\mathbf{x}^{k}, \{{\boldsymbol{\lambda}}_{l}^{k-1}\},\{{\boldsymbol{\alpha}}_{l}^{k-1}\},\{{\boldsymbol{\beta}}_{l}^{k-1}\}, \{\gamma_{s}^{k-1}\} ) -\frac{1}{
\eta_{\boldsymbol{\lambda}}}(
{\boldsymbol{\lambda}}_l^{k}-{\boldsymbol{\lambda}}_l^{k-1}),{\boldsymbol{\lambda}}_{l}^{k+1}-{\boldsymbol{\lambda}}_l^{k}\right\rangle \leq 0.
\label{2_lambda_1}
\end{equation}
Since ${\boldsymbol{\lambda}}_l^{k+1} - {\boldsymbol{\lambda}}_l^{k}  = \mathbf 0, l \notin  \mathcal{Q}^{k+1}$, inequality (\ref{2_lambda_1}) holds for $l$.
It can be seen from  (\ref{ALM_3}) that $\widetilde{L}_{p}$ is concave with respect to ${\boldsymbol{\lambda}}_{l}$. Therefore, we can obtain that,
\begin{align}
& \widetilde{L}_{p} (\{\boldsymbol{\mu}^{k+1}_{l}\}, \{ \mathbf{r}_{l}^{k+1}\},\{ \mathbf{p}_{l}^{k+1}\},\{ q_{s}^{k+1}\},\mathbf{x}^{k+1}, \{{\boldsymbol{\lambda}}_{l}^{k+1}\},\{{\boldsymbol{\alpha}}_{l}^{k}\},\{{\boldsymbol{\beta}}_{l}^{k}\}, \{\gamma_{s}^{k}\} ) \nonumber\\
& -\widetilde{L}_{p} (\{\boldsymbol{\mu}^{k+1}_{l}\}, \{ \mathbf{r}_{l}^{k+1}\},\{ \mathbf{p}_{l}^{k+1}\},\{ q_{s}^{k+1}\},\mathbf{x}^{k+1}, \{{\boldsymbol{\lambda}}_{l}^{k}\},\{{\boldsymbol{\alpha}}_{l}^{k}\},\{{\boldsymbol{\beta}}_{l}^{k}\}, \{\gamma_{s}^{k}\} ) \nonumber\\
& \leq\sum\limits_{l=1}^{L}\left\langle\nabla_{{\boldsymbol{\lambda}}_{l}}\widetilde{L}_{p} (\{\boldsymbol{\mu}^{k+1}_{l}\}, \{ \mathbf{r}_{l}^{k+1}\},\{ \mathbf{p}_{l}^{k+1}\},\{ q_{s}^{k+1}\},\mathbf{x}^{k+1}, \{{\boldsymbol{\lambda}}_{l}^{k}\},\{{\boldsymbol{\alpha}}_{l}^{k}\},\{{\boldsymbol{\beta}}_{l}^{k}\}, \{\gamma_{s}^{k}\} ),{\boldsymbol{\lambda}}_{l}^{k+1}-{\boldsymbol{\lambda}}_{l}^{k}\right\rangle \nonumber\\
&  \leq  \sum_{l=1}^{L}\big(\left\langle\nabla_{{\boldsymbol{\lambda}}_{l}}\widetilde{L}_{p} (\{\boldsymbol{\mu}^{k+1}_{l}\}, \{ \mathbf{r}_{l}^{k+1}\},\{ \mathbf{p}_{l}^{k+1}\},\{ q_{s}^{k+1}\},\mathbf{x}^{k+1}, \{{\boldsymbol{\lambda}}_{l}^{k}\},\{{\boldsymbol{\alpha}}_{l}^{k}\},\{{\boldsymbol{\beta}}_{l}^{k}\}, \{\gamma_{s}^{k}\} 
,{\boldsymbol{\lambda}}_{l}^{k+1}-{\boldsymbol{\lambda}}_{l}^{k} \right\rangle \nonumber \\
& +  \frac{1}{\eta_{{\boldsymbol{\lambda}}}}\big\langle{\boldsymbol{\lambda}}_l^k-{\boldsymbol{\lambda}}_l^{k-1},{\boldsymbol{\lambda}}_l^{k+1}-{\boldsymbol{\lambda}}_l^k\big\rangle \big) \nonumber \\
& -\sum_{l=1}^{L}\big(\left\langle\nabla_{{\boldsymbol{\lambda}}_{l}}\widetilde{L}_{p} (\{\boldsymbol{\mu}^{k}_{l}\}, \{ \mathbf{r}_{l}^{k}\},\{ \mathbf{p}_{l}^{k}\},\{ q_{s}^{k}\},\mathbf{x}^{k}, \{{\boldsymbol{\lambda}}_{l}^{k-1}\},\{{\boldsymbol{\alpha}}_{l}^{k-1}\},\{{\boldsymbol{\beta}}_{l}^{k-1}\}, \{\gamma_{s}^{k-1}\}
,{\boldsymbol{\lambda}}_{l}^{k+1}-{\boldsymbol{\lambda}}_{l}^{k} \right\rangle.
\label{2_lambda_2}
\end{align}
We have that, 
\begin{equation}
\begin{aligned}
& \left\langle\nabla_{{\boldsymbol{\lambda}}_{l}}\widetilde{L}_{p} (\{\boldsymbol{\mu}^{k+1}_{l}\}, \{ \mathbf{r}_{l}^{k+1}\},\{ \mathbf{p}_{l}^{k+1}\},\{ q_{s}^{k+1}\},\mathbf{x}^{k+1}, \{{\boldsymbol{\lambda}}_{l}^{k}\},\{{\boldsymbol{\alpha}}_{l}^{k}\},\{{\boldsymbol{\beta}}_{l}^{k}\}, \{\gamma_{s}^{k}\} 
,{\boldsymbol{\lambda}}_{l}^{k+1}-{\boldsymbol{\lambda}}_{l}^{k} \right\rangle \\
& -\left\langle \nabla_{{\boldsymbol{\lambda}}_{l}}\widetilde{L}_{p} (\{\boldsymbol{\mu}^{k}_{l}\}, \{ \mathbf{r}_{l}^{k}\},\{ \mathbf{p}_{l}^{k}\},\{ q_{s}^{k}\},\mathbf{x}^{k}, \{{\boldsymbol{\lambda}}_{l}^{k-1}\},\{{\boldsymbol{\alpha}}_{l}^{k-1}\},\{{\boldsymbol{\beta}}_{l}^{k-1}\}, \{\gamma_{s}^{k-1}\}
,{\boldsymbol{\lambda}}_{l}^{k+1}-{\boldsymbol{\lambda}}_{l}^{k} \right\rangle \\
& = \left\langle\nabla_{{\boldsymbol{\lambda}}_{l}}\widetilde{L}_{p} (\{\boldsymbol{\mu}^{k+1}_{l}\}, \{ \mathbf{r}_{l}^{k+1}\},\{ \mathbf{p}_{l}^{k+1}\},\{ q_{s}^{k+1}\},\mathbf{x}^{k+1}, \{{\boldsymbol{\lambda}}_{l}^{k}\},\{{\boldsymbol{\alpha}}_{l}^{k}\},\{{\boldsymbol{\beta}}_{l}^{k}\}, \{\gamma_{s}^{k}\} 
,{\boldsymbol{\lambda}}_{l}^{k+1}-{\boldsymbol{\lambda}}_{l}^{k} \right\rangle  \\ 
& - \left\langle\nabla_{{\boldsymbol{\lambda}}_{l}}\widetilde{L}_{p} (\{\boldsymbol{\mu}^{k}_{l}\}, \{ \mathbf{r}_{l}^{k}\},\{ \mathbf{p}_{l}^{k}\},\{ q_{s}^{k}\},\mathbf{x}^{k}, \{{\boldsymbol{\lambda}}_{l}^{k}\},\{{\boldsymbol{\alpha}}_{l}^{k}\},\{{\boldsymbol{\beta}}_{l}^{k}\}, \{\gamma_{s}^{k}\} 
,{\boldsymbol{\lambda}}_{l}^{k+1}-{\boldsymbol{\lambda}}_{l}^{k} \right\rangle\\
& + \left\langle\nabla_{{\boldsymbol{\lambda}}_{l}}\widetilde{L}_{p} (\{\boldsymbol{\mu}^{k}_{l}\}, \{ \mathbf{r}_{l}^{k}\},\{ \mathbf{p}_{l}^{k}\},\{ q_{s}^{k}\},\mathbf{x}^{k}, \{{\boldsymbol{\lambda}}_{l}^{k}\},\{{\boldsymbol{\alpha}}_{l}^{k}\},\{{\boldsymbol{\beta}}_{l}^{k}\}, \{\gamma_{s}^{k}\} 
,{\boldsymbol{\lambda}}_{l}^{k+1}-{\boldsymbol{\lambda}}_{l}^{k}-({\boldsymbol{\lambda}}_{l}^{k}-{\boldsymbol{\lambda}}_{l}^{k-1}) \right\rangle\\
& - \left\langle\nabla_{{\boldsymbol{\lambda}}_{l}}\widetilde{L}_{p} (\{\boldsymbol{\mu}^{k}_{l}\}, \{ \mathbf{r}_{l}^{k}\},\{ \mathbf{p}_{l}^{k}\},\{ q_{s}^{k}\},\mathbf{x}^{k}, \{{\boldsymbol{\lambda}}_{l}^{k-1}\},\{{\boldsymbol{\alpha}}_{l}^{k-1}\},\{{\boldsymbol{\beta}}_{l}^{k-1}\}, \{\gamma_{s}^{k-1}\} 
,{\boldsymbol{\lambda}}_{l}^{k+1}-{\boldsymbol{\lambda}}_{l}^{k}-({\boldsymbol{\lambda}}_{l}^{k}-{\boldsymbol{\lambda}}_{l}^{k-1}) \right\rangle  \\
& + \left\langle\nabla_{{\boldsymbol{\lambda}}_{l}}\widetilde{L}_{p} (\{\boldsymbol{\mu}^{k}_{l}\}, \{ \mathbf{r}_{l}^{k}\},\{ \mathbf{p}_{l}^{k}\},\{ q_{s}^{k}\},\mathbf{x}^{k}, \{{\boldsymbol{\lambda}}_{l}^{k}\},\{{\boldsymbol{\alpha}}_{l}^{k}\},\{{\boldsymbol{\beta}}_{l}^{k}\}, \{\gamma_{s}^{k}\} 
, {\boldsymbol{\lambda}}_{l}^{k}-{\boldsymbol{\lambda}}_{l}^{k-1} \right\rangle \\
& - \left\langle\nabla_{{\boldsymbol{\lambda}}_{l}}\widetilde{L}_{p} (\{\boldsymbol{\mu}^{k}_{l}\}, \{ \mathbf{r}_{l}^{k}\},\{ \mathbf{p}_{l}^{k}\},\{ q_{s}^{k}\},\mathbf{x}^{k}, \{{\boldsymbol{\lambda}}_{l}^{k-1}\},\{{\boldsymbol{\alpha}}_{l}^{k-1}\},\{{\boldsymbol{\beta}}_{l}^{k-1}\}, \{\gamma_{s}^{k-1}\} 
,{\boldsymbol{\lambda}}_{l}^{k}-{\boldsymbol{\lambda}}_{l}^{k-1} \right\rangle .
\label{n_1}
\end{aligned}
\end{equation}
According to the definition of $\widetilde{L}_{p}$ provided in (\ref{ALM_3}) and Cauchy-Schwarz inequality, we can obtain that,
\begin{align}
&  \left\langle\nabla_{{\boldsymbol{\lambda}}_{l}}\widetilde{L}_{p} (\{\boldsymbol{\mu}^{k+1}_{l}\}, \{ \mathbf{r}_{l}^{k+1}\},\{ \mathbf{p}_{l}^{k+1}\},\{ q_{s}^{k+1}\},\mathbf{x}^{k+1}, \{{\boldsymbol{\lambda}}_{l}^{k}\},\{{\boldsymbol{\alpha}}_{l}^{k}\},\{{\boldsymbol{\beta}}_{l}^{k}\}, \{\gamma_{s}^{k}\} 
,{\boldsymbol{\lambda}}_{l}^{k+1}-{\boldsymbol{\lambda}}_{l}^{k} \right\rangle \nonumber \\ 
& - \left\langle\nabla_{{\boldsymbol{\lambda}}_{l}}\widetilde{L}_{p} (\{\boldsymbol{\mu}^{k}_{l}\}, \{ \mathbf{r}_{l}^{k}\},\{ \mathbf{p}_{l}^{k}\},\{ q_{s}^{k}\},\mathbf{x}^{k}, \{{\boldsymbol{\lambda}}_{l}^{k}\},\{{\boldsymbol{\alpha}}_{l}^{k}\},\{{\boldsymbol{\beta}}_{l}^{k}\}, \{\gamma_{s}^{k}\} 
,{\boldsymbol{\lambda}}_{l}^{k+1}-{\boldsymbol{\lambda}}_{l}^{k} \right\rangle \nonumber \\
& =  \left\langle\nabla_{{\boldsymbol{\lambda}}_{l}}{L}_{p} (\{\boldsymbol{\mu}^{k+1}_{l}\}, \{ \mathbf{r}_{l}^{k+1}\},\{ \mathbf{p}_{l}^{k+1}\},\{ q_{s}^{k+1}\},\mathbf{x}^{k+1}, \{{\boldsymbol{\lambda}}_{l}^{k}\},\{{\boldsymbol{\alpha}}_{l}^{k}\},\{{\boldsymbol{\beta}}_{l}^{k}\}, \{\gamma_{s}^{k}\} 
,{\boldsymbol{\lambda}}_{l}^{k+1}-{\boldsymbol{\lambda}}_{l}^{k} \right\rangle \nonumber \\
& -\left\langle\nabla_{{\boldsymbol{\lambda}}_{l}}{L}_{p} (\{\boldsymbol{\mu}^{k}_{l}\}, \{ \mathbf{r}_{l}^{k}\},\{ \mathbf{p}_{l}^{k}\},\{ q_{s}^{k}\},\mathbf{x}^{k}, \{{\boldsymbol{\lambda}}_{l}^{k}\},\{{\boldsymbol{\alpha}}_{l}^{k}\},\{{\boldsymbol{\beta}}_{l}^{k}\}, \{\gamma_{s}^{k}\},{\boldsymbol{\lambda}}_{l}^{k+1}-{\boldsymbol{\lambda}}_{l}^{k} \right\rangle  
\nonumber \\ 
& + \frac{c_{{\boldsymbol{\lambda}}}^{k-1}-c_{{\boldsymbol{\lambda}}}^{k}}{2}(||\lambda_{l}^{k+1}||^{2}-||{\boldsymbol{\lambda}}_{l}^{k}||^{2})
-\frac{c_{{\boldsymbol{\lambda}}}^{k-1}-c_{{\boldsymbol{\lambda}}}^{k}}{2}||{\boldsymbol{\lambda}}_{l}^{k+1}-{\boldsymbol{\lambda}}_{l}^{k}||^{2} \nonumber \\ 
& \leq \frac{a_1}{2}\|{\boldsymbol{\lambda}}_{l}^{k+1}-{\boldsymbol{\lambda}}_{l}^{k}\| + 
\frac{1}{a_1}\| \boldsymbol{\mu}_{l}^{k+1}-\boldsymbol{\mu}_{l}^{k} \|^2 
+ \frac{1}{a_1}\|\mathbf{r}_{l}^{k+1}\circ\mathbf{r}_{l}^{k+1} -\mathbf{r}_{l}^{k}\circ \mathbf{r}_{l}^{k}\|^2  
+ \frac{c_{{\boldsymbol{\lambda}}}^{k-1}-c_{{\boldsymbol{\lambda}}}^{k}}{2}(||\lambda_{l}^{k+1}||^{2}-||{\boldsymbol{\lambda}}_{l}^{k}||^{2})  \nonumber 
 \\
& -\frac{c_{{\boldsymbol{\lambda}}}^{k-1}-c_{{\boldsymbol{\lambda}}}^{k}}{2}||{\boldsymbol{\lambda}}_{l}^{k+1}-{\boldsymbol{\lambda}}_{l}^{k}||^{2},  
\label{n_2}
\end{align}
where $a_1 > 0 $ is a constant.
From the definition of $\widetilde{L}_{p}$, with Cauchy-Schwarz inequality, we also have that,  
\begin{equation}
\begin{aligned}
& \left\langle\nabla_{{\boldsymbol{\lambda}}_{l}}\widetilde{L}_{p} (\{\boldsymbol{\mu}^{k}_{l}\}, \{ \mathbf{r}_{l}^{k}\},\{ \mathbf{p}_{l}^{k}\},\{ q_{s}^{k}\},\mathbf{x}^{k}, \{{\boldsymbol{\lambda}}_{l}^{k}\},\{{\boldsymbol{\alpha}}_{l}^{k}\},\{{\boldsymbol{\beta}}_{l}^{k}\}, \{\gamma_{s}^{k}\} 
,{\boldsymbol{\lambda}}_{l}^{k+1}-{\boldsymbol{\lambda}}_{l}^{k}-({\boldsymbol{\lambda}}_{l}^{k}-{\boldsymbol{\lambda}}_{l}^{k-1}) \right\rangle  \\
& - \left\langle\nabla_{{\boldsymbol{\lambda}}_{l}}\widetilde{L}_{p} (\{\boldsymbol{\mu}^{k}_{l}\}, \{ \mathbf{r}_{l}^{k}\},\{ \mathbf{p}_{l}^{k}\},\{ q_{s}^{k}\},\mathbf{x}^{k}, \{{\boldsymbol{\lambda}}_{l}^{k-1}\},\{{\boldsymbol{\alpha}}_{l}^{k-1}\},\{{\boldsymbol{\beta}}_{l}^{k-1}\}, \{\gamma_{s}^{k-1}\} 
,{\boldsymbol{\lambda}}_{l}^{k+1}-{\boldsymbol{\lambda}}_{l}^{k}-({\boldsymbol{\lambda}}_{l}^{k}-{\boldsymbol{\lambda}}_{l}^{k-1}) \right\rangle  \\
& \leq 
\frac{1}{2\eta_{{\boldsymbol{\lambda}}}} \| {\boldsymbol{\lambda}}_{l}^{k+1}-{\boldsymbol{\lambda}}_{l}^{k}-({\boldsymbol{\lambda}}_{l}^{k}-{\boldsymbol{\lambda}}_{l}^{k-1}) \|^2 + 
\frac{\eta_{{\boldsymbol{\lambda}}}}{2}\|   c_{{\boldsymbol{\lambda}}}^{k-1}( {\boldsymbol{\lambda}}_{l}^{k} - {\boldsymbol{\lambda}}^{k-1}_{l}   )       \|^2.
\label{n_3}
\end{aligned}
\end{equation}
Following \cite{xu}, since $\widetilde{L}_{p}$ is strong concave with respect to ${\boldsymbol{\lambda}}$,
we have
\begin{equation}
\begin{aligned}
& \left\langle\nabla_{{\boldsymbol{\lambda}}_{l}}\widetilde{L}_{p} (\{\boldsymbol{\mu}^{k}_{l}\}, \{ \mathbf{r}_{l}^{k}\},\{ \mathbf{p}_{l}^{k}\},\{ q_{s}^{k}\},\mathbf{x}^{k}, \{{\boldsymbol{\lambda}}_{l}^{k}\},\{{\boldsymbol{\alpha}}_{l}^{k}\},\{{\boldsymbol{\beta}}_{l}^{k}\}, \{\gamma_{s}^{k}\} 
, {\boldsymbol{\lambda}}_{l}^{k}-{\boldsymbol{\lambda}}_{l}^{k-1} \right\rangle \\
& - \left\langle\nabla_{{\boldsymbol{\lambda}}_{l}}\widetilde{L}_{p} (\{\boldsymbol{\mu}^{k}_{l}\}, \{ \mathbf{r}_{l}^{k}\},\{ \mathbf{p}_{l}^{k}\},\{ q_{s}^{k}\},\mathbf{x}^{k}, \{{\boldsymbol{\lambda}}_{l}^{k-1}\},\{{\boldsymbol{\alpha}}_{l}^{k-1}\},\{{\boldsymbol{\beta}}_{l}^{k-1}\}, \{\gamma_{s}^{k-1}\} 
,{\boldsymbol{\lambda}}_{l}^{k}-{\boldsymbol{\lambda}}_{l}^{k-1} \right\rangle  \\
& \leq 
 - \frac{1}{2c_{{\boldsymbol{\lambda}}}^{k-1} }  \| c_{{\boldsymbol{\lambda}}}^{k-1}( {\boldsymbol{\lambda}}_{l}^{k} - {\boldsymbol{\lambda}}^{k-1}_{l}   )\|^2    -  \frac{c_{{\boldsymbol{\lambda}}}^{k-1}}{2}\|  {\boldsymbol{\lambda}}_{l}^{k} - {\boldsymbol{\lambda}}^{k-1}_{l}   \|^2.
\label{n_4}
\end{aligned}
\end{equation}
 Since  $\frac{\eta_{{\boldsymbol{\lambda}}}}{2} \leq \frac{1}{2c_{{\boldsymbol{\lambda}}}^{k-1}}$, by employing (\ref{n_1})-(\ref{n_4}), we have that,
\begin{align}
& \left\langle\nabla_{{\boldsymbol{\lambda}}_{l}}\widetilde{L}_{p} (\{\boldsymbol{\mu}^{k+1}_{l}\}, \{ \mathbf{r}_{l}^{k+1}\},\{ \mathbf{p}_{l}^{k+1}\},\{ q_{s}^{k+1}\},\mathbf{x}^{k+1}, \{{\boldsymbol{\lambda}}_{l}^{k}\},\{{\boldsymbol{\alpha}}_{l}^{k}\},\{{\boldsymbol{\beta}}_{l}^{k}\}, \{\gamma_{s}^{k}\} 
,{\boldsymbol{\lambda}}_{l}^{k+1}-{\boldsymbol{\lambda}}_{l}^{k} \right\rangle \nonumber \\
& -\left\langle \nabla_{{\boldsymbol{\lambda}}_{l}}\widetilde{L}_{p} (\{\boldsymbol{\mu}^{k}_{l}\}, \{ \mathbf{r}_{l}^{k}\},\{ \mathbf{p}_{l}^{k}\},\{ q_{s}^{k}\},\mathbf{x}^{k}, \{{\boldsymbol{\lambda}}_{l}^{k-1}\},\{{\boldsymbol{\alpha}}_{l}^{k-1}\},\{{\boldsymbol{\beta}}_{l}^{k-1}\}, \{\gamma_{s}^{k-1}\}
,{\boldsymbol{\lambda}}_{l}^{k+1}-{\boldsymbol{\lambda}}_{l}^{k} \right\rangle \nonumber\\
& \leq \frac{a_1}{2}\|{\boldsymbol{\lambda}}_{l}^{k+1}-{\boldsymbol{\lambda}}_{l}^{k}\|^2 + 
\frac{1}{a_1}\| \boldsymbol{\mu}_{l}^{k+1}-\boldsymbol{\mu}_{l}^{k} \|^2 
+ \frac{1}{a_1}\|\mathbf{r}_{l}^{k+1}\circ\mathbf{r}_{l}^{k+1} -\mathbf{r}_{l}^{k}\circ \mathbf{r}_{l}^{k}\|^2  
+ \frac{c_{{\boldsymbol{\lambda}}}^{k-1}-c_{{\boldsymbol{\lambda}}}^{k}}{2}(||\boldsymbol\lambda_{l}^{k+1}||^{2}-||{\boldsymbol{\lambda}}_{l}^{k}||^{2}) 
\nonumber \\
& -\frac{c_{{\boldsymbol{\lambda}}}^{k-1}-c_{{\boldsymbol{\lambda}}}^{k}}{2}||{\boldsymbol{\lambda}}_{l}^{k+1}-{\boldsymbol{\lambda}}_{l}^{k}||^{2} + \frac{1}{2\eta_{{\boldsymbol{\lambda}}}} \| {\boldsymbol{\lambda}}_{l}^{k+1}-{\boldsymbol{\lambda}}_{l}^{k}-({\boldsymbol{\lambda}}_{l}^{k}-{\boldsymbol{\lambda}}_{l}^{k-1}) \|^2 -   \frac{c_{{\boldsymbol{\lambda}}}^{k-1}}{2}\|  {\boldsymbol{\lambda}}_{l}^{k} - {\boldsymbol{\lambda}}^{k-1}_{l}   \|^2.
\label{n_5}
\end{align}
It can be seen from \cref{assumption_1} that,
\begin{equation}
\frac{1}{a_1}\|\mathbf{r}_{l}^{k+1}\circ\mathbf{r}_{l}^{k+1} -\mathbf{r}_{l}^{k}\circ \mathbf{r}_{l}^{k}\|^2 \leq \frac{4w_r^2}{a_1}\| \mathbf{r}_{l}^{k+1}-\mathbf{r}_{l}^{k} \|^2.
\label{n_6}
\end{equation}
In addition, the following equality can be obtained,
\begin{equation}
\label{2_lambda_3}
\frac{1}{\eta_{\boldsymbol{\lambda}}}\left<{\boldsymbol{\lambda}}_l^k-{\boldsymbol{\lambda}}_l^{k-1},{\boldsymbol{\lambda}}_l^{k+1}-{\boldsymbol{\lambda}}_l^k\right> = 
\frac{1}{2\eta_{\boldsymbol{\lambda}}}||{\boldsymbol{\lambda}}_l^{k+1}-\boldsymbol\lambda_l^k||^2 + \frac{1}{2\eta_{\boldsymbol{\lambda}}}||{\boldsymbol{\lambda}}_l^k-{\boldsymbol{\lambda}}_l^{k-1}||^2 - \frac{1}{2\eta_{{\boldsymbol{\lambda}}}} \| {\boldsymbol{\lambda}}_{l}^{k+1}-{\boldsymbol{\lambda}}_{l}^{k}-({\boldsymbol{\lambda}}_{l}^{k}-{\boldsymbol{\lambda}}_{l}^{k-1}) \|^2.
\end{equation}
Combining (\ref{2_lambda_2}), (\ref{n_5}), (\ref{n_6}) with (\ref{2_lambda_3}), we can obtain,
\begin{equation}
\label{n_7}
\begin{aligned}
& {L}_{p} (\{\boldsymbol{\mu}^{k+1}_{l}\}, \{ \mathbf{r}_{l}^{k+1}\},\{ \mathbf{p}_{l}^{k+1}\},\{ q_{s}^{k+1}\},\mathbf{x}^{k+1}, \{{\boldsymbol{\lambda}}_{l}^{k+1}\},\{{\boldsymbol{\alpha}}_{l}^{k}\},\{{\boldsymbol{\beta}}_{l}^{k}\}, \{\gamma_{s}^{k}\} )\\
&-{L}_{p} (\{\boldsymbol{\mu}^{k+1}_{l}\}, \{ \mathbf{r}_{l}^{k+1}\},\{ \mathbf{p}_{l}^{k+1}\},\{ q_{s}^{k+1}\},\mathbf{x}^{k+1}, \{{\boldsymbol{\lambda}}_{l}^{k}\},\{{\boldsymbol{\alpha}}_{l}^{k}\},\{{\boldsymbol{\beta}}_{l}^{k}\}, \{\gamma_{s}^{k}\} ) \\
& \leq \frac{1}{a_1}\sum_{l=1}^{L}\|\boldsymbol{\mu}_{l}^{k+1}-\boldsymbol{\mu}_{l}^{k}\|^2 + 
\frac{4w_r^2}{a_1}\sum_{l=1}^{L}\|\mathbf{r}_{l}^{k+1}-\mathbf{r}_{l}^{k}\|^2 +
(\frac{a_1}{2} -  \frac{c_{{\boldsymbol{\lambda}}}^{k-1}-c_{{\boldsymbol{\lambda}}}^k}{2}+\frac{1}{2\eta_{{\boldsymbol{\lambda}}}})\sum_{l=1}^{L}\|{\boldsymbol{\lambda}}_{l}^{k+1}-{\boldsymbol{\lambda}}_{l}^{k}\|^2
 \\
& + \frac{c_{{\boldsymbol{\lambda}}}^{k-1}}{2}\sum_{l=1}^{L}( \|{\boldsymbol{\lambda}}_{l}^{k+1}\|^2 - \|{\boldsymbol{\lambda}}_{l}^{k}\|^2  ) + \frac{1}{2\eta_{{\boldsymbol{\lambda}}}}\sum_{l=1}^L\|{\boldsymbol{\lambda}}_{l}^{k}-{\boldsymbol{\lambda}}_{l}^{k-1}\|^2.
\end{aligned}
\end{equation}
Likewise, similar results can be obtained for other variables : 
\begin{equation}
\begin{aligned}
\label{2_alpha_1}
&  {L}_{p} (\{\boldsymbol{\mu}^{k+1}_{l}\}, \{ \mathbf{r}_{l}^{k+1}\},\{ \mathbf{p}_{l}^{k+1}\},\{ q_{s}^{k+1}\},\mathbf{x}^{k+1}, \{{\boldsymbol{\lambda}}_{l}^{k+1}\},\{{\boldsymbol{\alpha}}_{l}^{k+1}\},\{{\boldsymbol{\beta}}_{l}^{k}\}, \{\gamma_{s}^{k}\} ) \\
& -{L}_{p} (\{\boldsymbol{\mu}^{k+1}_{l}\}, \{ \mathbf{r}_{l}^{k+1}\},\{ \mathbf{p}_{l}^{k+1}\},\{ q_{s}^{k+1}\},\mathbf{x}^{k+1}, \{{\boldsymbol{\lambda}}_{l}^{k+1}\},\{{\boldsymbol{\alpha}}_{l}^{k}\},\{{\boldsymbol{\beta}}_{l}^{k}\}, \{\gamma_{s}^{k}\} )  \\
& \leq \frac{1}{a_2}\sum_{l=1}^{L}\|\boldsymbol{\mu}_{l}^{k+1}-\boldsymbol{\mu}_{l}^{k}\|^2 + 
\frac{4w_{\mathbf{p}}^2}{a_2}\sum_{l=1}^{L}\|\mathbf{p}_{l}^{k+1}-\mathbf{p}_{l}^{k}\|^2 +
(\frac{a_2}{2} -  \frac{c_{{\boldsymbol{\alpha}}}^{k-1}-c_{{\boldsymbol{\alpha}}}^k}{2}+\frac{1}{2\eta_{{\boldsymbol{\alpha}}}})\sum_{l=1}^{L}\|{\boldsymbol{\alpha}}_{l}^{k+1}-{\boldsymbol{\alpha}}_{l}^{k}\|^2
 \\
& + \frac{c_{{\boldsymbol{\alpha}}}^{k-1}}{2}\sum_{l=1}^{L}( \|{\boldsymbol{\alpha}}_{l}^{k+1}\|^2 - \|{\boldsymbol{\alpha}}_{l}^{k}\|^2  ) + \frac{1}{2\eta_{{\boldsymbol{\alpha}}}}\sum_{l=1}^L\|{\boldsymbol{\alpha}}_{l}^{k}-{\boldsymbol{\alpha}}_{l}^{k-1}\|^2,
\end{aligned}
\end{equation}
\begin{equation}
\label{2_beta_1}
\begin{aligned}
& {L}_{p} (\{\boldsymbol{\mu}^{k+1}_{l}\}, \{ \mathbf{r}_{l}^{k+1}\},\{ \mathbf{p}_{l}^{k+1}\},\{ q_{s}^{k+1}\},\mathbf{x}^{k+1}, \{{\boldsymbol{\lambda}}_{l}^{k+1}\},\{{\boldsymbol{\alpha}}_{l}^{k+1}\},\{{\boldsymbol{\beta}}_{l}^{k+1}\}, \{\gamma_{s}^{k}\} ) \\
& -{L}_{p} (\{\boldsymbol{\mu}^{k+1}_{l}\}, \{ \mathbf{r}_{l}^{k+1}\},\{ \mathbf{p}_{l}^{k+1}\},\{ q_{s}^{k+1}\},\mathbf{x}^{k+1}, \{{\boldsymbol{\lambda}}_{l}^{k+1}\},\{{\boldsymbol{\alpha}}_{l}^{k+1}\},\{{\boldsymbol{\beta}}_{l}^{k}\}, \{\gamma_{s}^{k}\} )
\\
& \leq \frac{1}{a_3}\sum_{l=1}^{L}\|\boldsymbol{\mu}_{l}^{k+1}-\boldsymbol{\mu}_{l}^{k}\|^2 + 
\frac{1}{a_3}\sum_{l=1}^{L}\|\mathbf{B}_{l}\mathbf{x}^{k+1}-\mathbf{B}_{l}\mathbf{x}^{k}\|^2 +
(\frac{a_3}{2} -  \frac{c_{{\boldsymbol{\beta}}}^{k-1}-c_{{\boldsymbol{\beta}}}^k}{2}+\frac{1}{2\eta_{{\boldsymbol{\beta}}}})\sum_{l=1}^{L}\|{\boldsymbol{\beta}}_{l}^{k+1}-{\boldsymbol{\beta}}_{l}^{k}\|^2
\\
& + \frac{c_{{\boldsymbol{\beta}}}^{k-1}}{2}\sum_{l=1}^{L}( \|{\boldsymbol{\beta}}_{l}^{k+1}\|^2 - \|{\boldsymbol{\beta}}_{l}^{k}\|^2  ) + \frac{1}{2\eta_{{\boldsymbol{\beta}}}}\sum_{l=1}^L\|{\boldsymbol{\beta}}_{l}^{k}-{\boldsymbol{\beta}}_{l}^{k-1}\|^2,
\end{aligned}
\end{equation}
\begin{equation}
\label{2_gamma_1}
\begin{aligned}
& {L}_{p} (\{\boldsymbol{\mu}^{k+1}_{l}\}, \{ \mathbf{r}_{l}^{k+1}\},\{ \mathbf{p}_{l}^{k+1}\},\{ q_{s}^{k+1}\},\mathbf{x}^{k+1}, \{{\boldsymbol{\lambda}}_{l}^{k+1}\},\{{\boldsymbol{\alpha}}_{l}^{k+1}\},\{{\boldsymbol{\beta}}_{l}^{k+1}\}, \{\gamma_{s}^{k+1}\} ) \\
& -{L}_{p} (\{\boldsymbol{\mu}^{k+1}_{l}\}, \{ \mathbf{r}_{l}^{k+1}\},\{ \mathbf{p}_{l}^{k+1}\},\{ q_{s}^{k+1}\},\mathbf{x}^{k+1}, \{{\boldsymbol{\lambda}}_{l}^{k+1}\},\{{\boldsymbol{\alpha}}_{l}^{k+1}\},\{{\boldsymbol{\beta}}_{l}^{k+1}\}, \{\gamma_{s}^{k}\} )
\\
& \leq \frac{1}{a_4}\sum_{s=1}^{|\mathcal{P}^k|}\|\mathbf{b}_s\|^2\|\mathbf{x}^{k+1}-\mathbf{x}^{k}\|^2 + 
\frac{4w_q^2}{a_4}\sum_{s=1}^{|\mathcal{P}^k|}\|{q}_s^{k+1}-{q}_{s}^{k}\|^2 +
(\frac{a_4}{2} -  \frac{c_{\gamma}^{k-1}-c_{\gamma}^k}{2}+\frac{1}{2\eta_{\gamma}})\sum_{s=1}^{|\mathcal{P}^k|}\|\gamma_{s}^{k+1}-\gamma_{s}^{k}\|^2
\\
& + \frac{c_{\gamma}^{k-1}}{2}\sum_{s=1}^{|\mathcal{P}^k|}( \|\gamma_{s}^{k+1}\|^2 - \|\gamma_{s}^{k}\|^2  ) + \frac{1}{2\eta_{\gamma}}\sum_{s=1}^{|\mathcal{P}^k|}\|\gamma_{s}^{k}-\gamma_{s}^{k-1}\|^2,
\end{aligned}
\end{equation}
where $a_2 > 0 $, $a_3 > 0 $ and $a_4 > 0 $ are constant. Combining Lemma \ref{lemma_1} with (\ref{n_7}), (\ref{2_alpha_1}), (\ref{2_beta_1}) and (\ref{2_gamma_1}),
we can obtain that,
\begin{align}
& {L}_{p} (\{\boldsymbol{\mu}^{k+1}_{l}\}, \{ \mathbf{r}_{l}^{k+1}\},\{ \mathbf{p}_{l}^{k+1}\},\{ q_{s}^{k+1}\},\mathbf{x}^{k+1}, \{{\boldsymbol{\lambda}}_{l}^{k+1}\},\{{\boldsymbol{\alpha}}_{l}^{k+1}\},\{{\boldsymbol{\beta}}_{l}^{k+1}\}, \{\gamma_{s}^{k+1}\} ) \nonumber \\
& - {L}_{p} (\{\boldsymbol{\mu}^{k}_{l}\}, \{ \mathbf{r}_{l}^{k}\},\{ \mathbf{p}_{l}^{k}\},\{ q_{s}^{k}\},\mathbf{x}^{k}, \{{\boldsymbol{\lambda}}_{l}^{k}\},\{{\boldsymbol{\alpha}}_{l}^{k}\},\{{\boldsymbol{\beta}}_{l}^{k}\}, \{\gamma_{s}^{k}\} ) 
\nonumber \\
& \leq (\frac{1}{a_1} + \frac{1}{a_2} + \frac{1}{a_3} +\frac{ w_{\boldsymbol{\mu}} L_{w}M^{\frac{1}{2}}  }{3} + 1 -\frac{1}{\eta_{\boldsymbol{\mu}}^{k}})\sum_{l=1}^{L}\|\boldsymbol{\mu}_{l}^{k+1}-\boldsymbol{\mu}_{l}^{k}\|^2
+  (\frac{4w_{\mathbf{r}}^2}{a_1}+\frac{w_{\mathbf{r}}L_{w}M^{\frac{1}{2}}}{3} + w_{{\boldsymbol{\lambda}}}-\frac{1}{\eta_{\mathbf{r}}^{k}} )\sum_{l=1}^{L}\|\mathbf{r}_{l}^{k+1}-\mathbf{r}_{l}^{k}\|^2 \nonumber \\
& + (\frac{4w_{\mathbf{p}}^2}{a_2}+\frac{w_{\mathbf{p}}L_{w}M^{\frac{1}{2}}}{3} + w_{{\boldsymbol{\alpha}}}-\frac{1}{\eta_{\mathbf{p}}^{k}} )\sum_{l=1}^{L}\|\mathbf{p}_{l}^{k+1}-\mathbf{p}_{l}^{k}\|^2 + (\frac{4w_{q}^2}{a_4}+\frac{w_{q}L_{w}}{3} + w_{\gamma}-\frac{1}{\eta_{q}^{k}} )\sum_{s=1}^{|\mathcal{P}^k|}\|q_{s}^{k+1}-q_{s}^{k}\|^2  \nonumber \\
& + ( -\frac{1}{\eta_{\mathbf{x}}^{k}}+\frac{L_{w}}{6})||\mathbf{x}^{k+1}-\mathbf{x}^{k}||^{2} + \frac{1}{a_3}\sum_{l=1}^{L}\|\mathbf{B}_{l}\mathbf{x}^{k+1}-\mathbf{B}_{l}\mathbf{x}^{k}\|^2 + \frac{1}{a_4}\sum_{s=1}^{|\mathcal{P}^k|}\|\mathbf{b}_s\|^2\|\mathbf{x}^{k+1}-\mathbf{x}^{k}\|^2  \nonumber \\
& +  (\frac{a_1}{2} -  \frac{c_{{\boldsymbol{\lambda}}}^{k-1}-c_{{\boldsymbol{\lambda}}}^k}{2}+\frac{1}{2\eta_{{\boldsymbol{\lambda}}}})\sum_{l=1}^{L}\|{\boldsymbol{\lambda}}_{l}^{k+1}-{\boldsymbol{\lambda}}_{l}^{k}\|^2
+ \frac{c_{{\boldsymbol{\lambda}}}^{k-1}}{2}\sum_{l=1}^{L}( \|{\boldsymbol{\lambda}}_{l}^{k+1}\|^2 - \|{\boldsymbol{\lambda}}_{l}^{k}\|^2  ) + \frac{1}{2\eta_{{\boldsymbol{\lambda}}}}\sum_{l=1}^L\|{\boldsymbol{\lambda}}_{l}^{k}-{\boldsymbol{\lambda}}_{l}^{k-1}\|^2  \nonumber \\
& + (\frac{a_2}{2} -  \frac{c_{{\boldsymbol{\alpha}}}^{k-1}-c_{{\boldsymbol{\alpha}}}^k}{2}+\frac{1}{2\eta_{{\boldsymbol{\alpha}}}})\sum_{l=1}^{L}\|{\boldsymbol{\alpha}}_{l}^{k+1}-{\boldsymbol{\alpha}}_{l}^{k}\|^2
 + \frac{c_{{\boldsymbol{\alpha}}}^{k-1}}{2}\sum_{l=1}^{L}( \|{\boldsymbol{\alpha}}_{l}^{k+1}\|^2 - \|{\boldsymbol{\alpha}}_{l}^{k}\|^2  ) + \frac{1}{2\eta_{{\boldsymbol{\alpha}}}}\sum_{l=1}^L\|{\boldsymbol{\alpha}}_{l}^{k}-{\boldsymbol{\alpha}}_{l}^{k-1}\|^2  \nonumber \\
 & + (\frac{a_3}{2} -  \frac{c_{{\boldsymbol{\beta}}}^{k-1}-c_{{\boldsymbol{\beta}}}^k}{2}+\frac{1}{2\eta_{{\boldsymbol{\beta}}}})\sum_{l=1}^{L}\|{\boldsymbol{\beta}}_{l}^{k+1}-{\boldsymbol{\beta}}_{l}^{k}\|^2
 + \frac{c_{{\boldsymbol{\beta}}}^{k-1}}{2}\sum_{l=1}^{L}( \|{\boldsymbol{\beta}}_{l}^{k+1}\|^2 - \|{\boldsymbol{\beta}}_{l}^{k}\|^2  ) + \frac{1}{2\eta_{{\boldsymbol{\beta}}}}\sum_{l=1}^L\|{\boldsymbol{\beta}}_{l}^{k}-{\boldsymbol{\beta}}_{l}^{k-1}\|^2  \nonumber \\
 & + (\frac{a_4}{2} -  \frac{c_{\gamma}^{k-1}-c_{\gamma}^k}{2}+\frac{1}{2\eta_{\gamma}})\sum_{s=1}^{|\mathcal{P}^k|}\|\gamma_{s}^{k+1}-\gamma_{s}^{k}\|^2
+ \frac{c_{\gamma}^{k-1}}{2}\sum_{s=1}^{|\mathcal{P}^k|}( \|\gamma_{s}^{k+1}\|^2 - \|\gamma_{s}^{k}\|^2  ) + \frac{1}{2\eta_{\gamma}}\sum_{s=1}^{|\mathcal{P}^k|}\|\gamma_{s}^{k}-\gamma_{s}^{k-1}\|^2,
\end{align}
which concludes the proof of Lemma \ref{lemma_2}.
\end{proof}

\begin{lemma}
\label{lemma_3}
Define $S_1^{k+1}$, $S_2^{k+1}$, $S_3^{k+1}$, $S_4^{k+1}$, $F^{k+1}$ and $a_5$ as :
\begin{equation}
\begin{aligned}
& S_1^{k+1}=\frac{4}{\eta_{\boldsymbol{\lambda}}^2c_{{\boldsymbol{\lambda}}}^{k+1}}\sum\limits_{l=1}^{L}||\lambda_l^{k+1}-\lambda_l^k||^2-\frac{4}{\eta_{{\boldsymbol{\lambda}}}}(\frac{c_{{\boldsymbol{\lambda}}}^{k-1}}{c_{{\boldsymbol{\lambda}}}^k}-1)\sum\limits_{l=1}^{L}||{\boldsymbol{\lambda}}_l^{k+1}||^2, \\
& S_2^{k+1}=\frac{4}{\eta_{{\boldsymbol{\alpha}}^2} c_{{\boldsymbol{\alpha}}}^{k+1}}\sum\limits_{l=1}^{L}||\alpha^{k+1}-\alpha^k||^2-\frac{4}{\eta_{{\boldsymbol{\alpha}}}}(\frac{c_{{\boldsymbol{\alpha}}}^{k-1}}{c_{{\boldsymbol{\alpha}}}^k}-1)\sum\limits_{l=1}^{L}||{\boldsymbol{\alpha}}_l^{k+1}||^2, \\
& S_3^{k+1}=\frac{4}{\eta_{\boldsymbol{\beta}}^2c_{{\boldsymbol{\beta}}}^{k+1}}\sum\limits_{l=1}^{L}||\beta^{k+1}-\beta^k||^2-\frac{4}{\eta_{{\boldsymbol{\beta}}}}(\frac{c_{{\boldsymbol{\beta}}}^{k-1}}{c_{{\boldsymbol{\beta}}}^k}-1)\sum\limits_{l=1}^{L}||{\boldsymbol{\beta}}^{k+1}_l||^2,\\
& S_4^{k+1}=\frac{4}{\eta_\gamma^2c_{\gamma}^{k+1}}\sum\limits_{s=1}^{|\mathcal{P}^k|}||\gamma_s^{k+1}-\gamma_s^k||^2-\frac{4}{\eta_{\gamma}}(\frac{c_{\gamma}^{k-1}}{c_{\gamma}^k}-1)\sum\limits_{l=1}^{L}||\gamma_s^{k+1}||^2.
\end{aligned}
\end{equation}
\begin{equation}
\begin{aligned}
F^{k+1} = & \quad {L}_{p} (\{\boldsymbol{\mu}^{k+1}_{l}\}, \{ \mathbf{r}_{l}^{k+1}\},\{ \mathbf{p}_{l}^{k+1}\},\{ q_{s}^{k+1}\},\mathbf{x}^{k+1}, \{{\boldsymbol{\lambda}}_{l}^{k+1}\},\{{\boldsymbol{\alpha}}_{l}^{k+1}\},\{{\boldsymbol{\beta}}_{l}^{k+1}\}, \{\gamma_{s}^{k+1}\} ) \\[3mm]
& \quad +   S_{1}^{k+1}+S_{2}^{k+1}+ S_{3}^{k+1}+ S_{4}^{k+1}  \\
& \quad -\frac{7}{2\eta_{\boldsymbol{\lambda}}}\sum\limits_{l=1}^{L}||{\boldsymbol{\lambda}}_l^{k+1}-{\boldsymbol{\lambda}}_l^k||^2-\frac{c_{{\boldsymbol{\lambda}}}^k}{2}\sum\limits_{l=1}^{L}||{\boldsymbol{\lambda}}_l^{k+1}||^2  \\
& \quad -\frac{7}{2\eta_{\boldsymbol{\alpha}}}\sum\limits_{l=1}^{L}||{\boldsymbol{\alpha}}_{l}^{k+1}-{\boldsymbol{\alpha}}_l^k||^2-\frac{c_{{\boldsymbol{\alpha}}}^k}{2}\sum\limits_{l=1}^{L}||{\boldsymbol{\alpha}}_l^{k+1}||^2  \\
&\quad -\frac{7}{2\eta_{\boldsymbol{\beta}}}\sum\limits_{l=1}^{L}||{\boldsymbol{\beta}}_{l}^{k+1}-{\boldsymbol{\beta}}_l^k||^2-\frac{c_{{\boldsymbol{\beta}}}^k}{2}\sum\limits_{l=1}^{L}||{\boldsymbol{\beta}}_l^{k+1}||^2\\
& \quad -\frac{7}{2\eta_\gamma}\sum\limits_{s=1}^{|\mathcal{P}^k|}||\gamma_{s}^{k+1}-\gamma_s^k||^2-\frac{c_{\gamma}^k}{2}\sum\limits_{s=1}^{|\mathcal{P}^k|}||\gamma_s^{k+1}||^2. 
\end{aligned}
\end{equation}
\vspace{-5pt}
\begin{equation}
\begin{aligned}
a_5 = &  \max\bigg\{ \eta_{{\boldsymbol{\lambda}}}+\eta_{{\boldsymbol{\alpha}}}+\eta_{{\boldsymbol{\beta}}}+ \frac{ w_{\boldsymbol{\mu}} L_{w}M^{\frac{1}{2}}  }{3}+ 1,
{4w_{\mathbf{r}}^2}\eta_{{\boldsymbol{\lambda}}}+\frac{w_{\mathbf{r}}L_{w}M^{\frac{1}{2}}}{3}  + w_{{\boldsymbol{\lambda}}}, \\ & 
{4w_{\mathbf{p}}^2}\eta_{{\boldsymbol{\alpha}}}+\frac{w_{\mathbf{p}}L_{w}M^{\frac{1}{2}}}{3} 
+ w_{{\boldsymbol{\alpha}}}, 
{4w_{q}^2}\eta_{\gamma}+\frac{w_{q}L_{w}}{3} + w_{\gamma},  
\eta_{{\boldsymbol{\beta}}} + \frac{L_{w}}{6} + \eta_{\gamma}\sum_{s=1}^{|\mathcal{P}^k|}\|\mathbf{b}_s\|^2
\bigg \}.
\end{aligned}
\end{equation}
And then $\forall k \geq K_1$, we have that, 
\begin{equation}
\begin{aligned}
F^{k+1} - F^{k} \leq & (a_5 - \frac{1}{\eta_{\boldsymbol{\mu}}^{k}} +   \frac{16}{\eta_{{\boldsymbol{\lambda}}}(c_{{\boldsymbol{\lambda}}}^k)^2}   + \frac{16}{\eta_{{\boldsymbol{\alpha}}}(c_{{\boldsymbol{\alpha}}}^k)^2}  +  \frac{16}{\eta_{{\boldsymbol{\beta}}}(c_{{\boldsymbol{\beta}}}^k)^2}  ) \sum_{l=1}^{L}\| \boldsymbol{\mu}_l^{k+1} - \boldsymbol{\mu}_l^k \|^2 \\
& +  (a_5-\frac{1}{\eta_{\mathbf{r}}^{k}} + \frac{64w^2_{\mathbf{r}}}{\eta_{{\boldsymbol{\lambda}}}(c_{{\boldsymbol{\lambda}}}^k)^2} )\sum_{l=1}^{L}\|\mathbf{r}_{l}^{k+1}-\mathbf{r}_{l}^{k}\|^2 \\
& + (a_5-\frac{1}{\eta_{\mathbf{p}}^{k}} + \frac{64w^2_{\mathbf{p}}}{\eta_{{\boldsymbol{\alpha}}}(c_{{\boldsymbol{\alpha}}}^k)^2} )\sum_{l=1}^{L}\|\mathbf{p}_{l}^{k+1}-\mathbf{p}_{l}^{k}\|^2 \\
& + (a_5-\frac{1}{\eta_{{q}}^{k}} + \frac{64w^2_{{q}}}{\eta_{{\gamma}}(c_{{\gamma}}^k)^2} )\sum_{s=1}^{|\mathcal{P}^k|}\|{q}_{s}^{k+1}-{q}_{s}^{k}\|^2 \\
& + (a_5 - \frac{1}{\eta_{\mathbf{x}}^{k}} + \frac{16}{\eta_{{\boldsymbol{\beta}}}(c_{{\boldsymbol{\beta}}}^k)^2}  + \frac{16}{\eta_{{\gamma}}(c_{{\gamma}}^k)^2} \sum_{s=1}^{|\mathcal{P}^k|}\|b_{s}\|^2
) \|\mathbf{x}^{k+1}-\mathbf{x}^{k}\|^2\\
& -\frac{1}{10\eta_{\boldsymbol{\lambda}}}\sum\limits_{l=1}^{L}||{\boldsymbol{\lambda}}_l^{k+1}-{\boldsymbol{\lambda}}_l^k||^2
+ \frac{c_{{\boldsymbol{\lambda}}}^{k-1}-c_{{\boldsymbol{\lambda}}}^k}{2}\sum\limits_{l=1}^{L}||{\boldsymbol{\lambda}}_l^{k+1}||^2
\\
& -\frac{1}{10\eta_{\boldsymbol{\alpha}}}\sum\limits_{l=1}^{L}||\boldsymbol\alpha_l^{k+1}-\boldsymbol\alpha_l^k||^2
+ \frac{c_{{\boldsymbol{\alpha}}}^{k-1}-c_{{\boldsymbol{\alpha}}}^k}{2}\sum\limits_{l=1}^{L}||{\boldsymbol{\alpha}}_l^{k+1}||^2
\\
& -\frac{1}{10\eta_{\boldsymbol{\beta}}}\sum\limits_{l=1}^{L}||{\boldsymbol{\beta}}_l^{k+1}-{\boldsymbol{\beta}}_l^k||^2
+ \frac{c_{{\boldsymbol{\beta}}}^{k-1}-c_{{\boldsymbol{\beta}}}^k}{2}\sum\limits_{l=1}^{L}||{\boldsymbol{\beta}}_l^{k+1}||^2
\\
& -\frac{1}{10\eta_\gamma}\sum\nolimits_{s=1}^{|\mathcal{P}^k|}||\gamma_s^{k+1}-\gamma_s^k||^2
+ \frac{c_{\gamma}^{k-1}-c_{\gamma}^k}{2}\sum\nolimits_{s=1}^{|\mathcal{P}^k|}||\gamma_s^{k+1}||^2
\\
& + \frac{4}{\eta_{{\boldsymbol{\lambda}}}}(\frac{c_{{\boldsymbol{\lambda}}}^{k-2}}{c_{{\boldsymbol{\lambda}}}^{k-1}}-\frac{c_{{\boldsymbol{\lambda}}}^{k-1}}{c_{{\boldsymbol{\lambda}}}^{k}})\sum_{l=1}^{L}||{\boldsymbol{\lambda}}_{l}^{k}||^{2}
+ \frac{4}{\eta_{{\boldsymbol{\alpha}}}}(\frac{c_{{\boldsymbol{\alpha}}}^{k-2}}{c_{{\boldsymbol{\alpha}}}^{k-1}}-\frac{c_{{\boldsymbol{\alpha}}}^{k-1}}{c_{{\boldsymbol{\alpha}}}^{k}})\sum_{l=1}^{L}||{\boldsymbol{\alpha}}_{l}^{k}||^{2}
\\
&+ \frac{4}{\eta_{{\boldsymbol{\beta}}}}(\frac{c_{{\boldsymbol{\beta}}}^{k-2}}{c_{{\boldsymbol{\beta}}}^{k-1}}-\frac{c_{{\boldsymbol{\beta}}}^{k-1}}{c_{{\boldsymbol{\beta}}}^{k}})\sum_{l=1}^{L}||{\boldsymbol{\beta}}_{l}^{k}||^{2}
+ \frac{4}{\eta_{\gamma}}(\frac{c_{\gamma}^{k-2}}{c_{\gamma}^{k-1}}-\frac{c_{\gamma}^{k-1}}{c_{\gamma}^{k}})\sum_{s=1}^{|\mathcal{P}^k|}||\gamma_{s}^{k}||^{2} .
\end{aligned}
\end{equation}
\end{lemma}
\begin{proof}
Let $a_1 = \frac{1}{\eta_{{\boldsymbol{\lambda}}}}, a_2 = \frac{1}{\eta_{{\boldsymbol{\alpha}}}}, a_3 = \frac{1}{\eta_{{\boldsymbol{\beta}}}}, a_4 = \frac{1}{\eta_{\gamma}},$ and substitute them into the Lemma \ref{lemma_2}, $\forall k \geq K_1$, we have,
\begin{equation}
\label{3_1}
\begin{aligned}
& {L}_{p} (\{\boldsymbol{\mu}^{k+1}_{l}\}, \{ \mathbf{r}_{l}^{k+1}\},\{ \mathbf{p}_{l}^{k+1}\},\{ q_{s}^{k+1}\},\mathbf{x}^{k+1}, \{{\boldsymbol{\lambda}}_{l}^{k+1}\},\{{\boldsymbol{\alpha}}_{l}^{k+1}\},\{{\boldsymbol{\beta}}_{l}^{k+1}\}, \{\gamma_{s}^{k+1}\} ) \\
& -{L}_{p} (\{\boldsymbol{\mu}^{k}_{l}\}, \{ \mathbf{r}_{l}^{k}\},\{ \mathbf{p}_{l}^{k}\},\{ q_{s}^{k}\},\mathbf{x}^{k}, \{{\boldsymbol{\lambda}}_{l}^{k}\},\{{\boldsymbol{\alpha}}_{l}^{k}\},\{{\boldsymbol{\beta}}_{l}^{k}\}, \{\gamma_{s}^{k}\} ) \\
& \leq (\eta_{{\boldsymbol{\lambda}}} + \eta_{{\boldsymbol{\alpha}}} +\eta_{{\boldsymbol{\beta}}} + \frac{ w_{\boldsymbol{\mu}} L_{w}M^{\frac{1}{2}}  }{3} + 1 -\frac{1}{\eta_{\boldsymbol{\mu}}^{k}})\sum_{l=1}^{L}\|\boldsymbol{\mu}_{l}^{k+1}-\boldsymbol{\mu}_{l}^{k}\|^2
+  (4w_{\mathbf{r}}^2\eta_{{\boldsymbol{\lambda}}}+\frac{w_{\mathbf{r}}L_{w}M^{\frac{1}{2}}}{3} + w_{{\boldsymbol{\lambda}}}-\frac{1}{\eta_{\mathbf{r}}^{k}} )\sum_{l=1}^{L}\|\mathbf{r}_{l}^{k+1}-\mathbf{r}_{l}^{k}\|^2  \\ 
& + (4w_{\mathbf{p}}^2\eta_{{\boldsymbol{\alpha}}}+\frac{w_{\mathbf{p}}L_{w}M^{\frac{1}{2}}}{3} + w_{{\boldsymbol{\alpha}}}-\frac{1}{\eta_{\mathbf{p}}^{k}} )\sum_{l=1}^{L}\|\mathbf{p}_{l}^{k+1}-\mathbf{p}_{l}^{k}\|^2 + (4w_{q}^2\eta_{\gamma}+\frac{w_{q}L_{w}}{3} + w_{\gamma}-\frac{1}{\eta_{q}^{k}} )\sum_{s=1}^{|\mathcal{P}^k|}\|q_{s}^{k+1}-q_{s}^{k}\|^2\\
& + ( -\frac{1}{\eta_{\mathbf{x}}^{k}}+\frac{L_{w}}{6})||\mathbf{x}^{k+1}-\mathbf{x}^{k}||^{2} + \eta_{{\boldsymbol{\beta}}}\sum_{l=1}^{L}\|\mathbf{B}_{l}\mathbf{x}^{k+1}-\mathbf{B}_{l}\mathbf{x}^{k}\|^2 + \eta_{\gamma}\sum_{s=1}^{|\mathcal{P}^k|}\|\mathbf{b}_s\|^2\|\mathbf{x}^{k+1}-\mathbf{x}^{k}\|^2 \\
& +  ( -  \frac{c_{{\boldsymbol{\lambda}}}^{k-1}-c_{{\boldsymbol{\lambda}}}^k}{2}+\frac{1}{\eta_{{\boldsymbol{\lambda}}}})\sum_{l=1}^{L}\|{\boldsymbol{\lambda}}_{l}^{k+1}-{\boldsymbol{\lambda}}_{l}^{k}\|^2
+ \frac{c_{{\boldsymbol{\lambda}}}^{k-1}}{2}\sum_{l=1}^{L}( \|{\boldsymbol{\lambda}}_{l}^{k+1}\|^2 - \|{\boldsymbol{\lambda}}_{l}^{k}\|^2  ) + \frac{1}{2\eta_{{\boldsymbol{\lambda}}}}\sum_{l=1}^L\|{\boldsymbol{\lambda}}_{l}^{k}-{\boldsymbol{\lambda}}_{l}^{k-1}\|^2 \\
& + ( -  \frac{c_{{\boldsymbol{\alpha}}}^{k-1}-c_{{\boldsymbol{\alpha}}}^k}{2}+\frac{1}{\eta_{{\boldsymbol{\alpha}}}})\sum_{l=1}^{L}\|{\boldsymbol{\alpha}}_{l}^{k+1}-{\boldsymbol{\alpha}}_{l}^{k}\|^2
 + \frac{c_{{\boldsymbol{\alpha}}}^{k-1}}{2}\sum_{l=1}^{L}( \|{\boldsymbol{\alpha}}_{l}^{k+1}\|^2 - \|{\boldsymbol{\alpha}}_{l}^{k}\|^2  ) + \frac{1}{2\eta_{{\boldsymbol{\alpha}}}}\sum_{l=1}^L\|{\boldsymbol{\alpha}}_{l}^{k}-{\boldsymbol{\alpha}}_{l}^{k-1}\|^2 \\
 & + (- \frac{c_{{\boldsymbol{\beta}}}^{k-1}-c_{{\boldsymbol{\beta}}}^k}{2}+\frac{1}{\eta_{{\boldsymbol{\beta}}}})\sum_{l=1}^{L}\|{\boldsymbol{\beta}}_{l}^{k+1}-{\boldsymbol{\beta}}_{l}^{k}\|^2
 + \frac{c_{{\boldsymbol{\beta}}}^{k-1}}{2}\sum_{l=1}^{L}(\|{\boldsymbol{\beta}}_{l}^{k+1}\|^2- \|{\boldsymbol{\beta}}_{l}^{k}\|^2  ) + \frac{1}{2\eta_{{\boldsymbol{\beta}}}}\sum_{l=1}^L\|{\boldsymbol{\beta}}_{l}^{k}-{\boldsymbol{\beta}}_{l}^{k-1}\|^2 \\
 & + (-  \frac{c_{\gamma}^{k-1}-c_{\gamma}^k}{2}+\frac{1}{\eta_{\gamma}})\sum_{s=1}^{|\mathcal{P}^k|}\|\gamma_{s}^{k+1}-\gamma_{s}^{k}\|^2
+ \frac{c_{\gamma}^{k-1}}{2}\sum_{s=1}^{|\mathcal{P}^k|}( \|\gamma_{s}^{k+1}\|^2 - \|\gamma_{s}^{k}\|^2  ) + \frac{1}{2\eta_{\gamma}}\sum_{s=1}^{|\mathcal{P}^k|}\|\gamma_{s}^{k}-\gamma_{s}^{k-1}\|^2.
\end{aligned}
\end{equation}
According to (\ref{2_lambda_0}) ,(\ref{2_lambda_1}) and (\ref{n_5}), in the ${(k+1)}^{th}$ iteration, we can obtain,
\begin{equation}
\begin{aligned}
&\frac{1}{\eta_{{\boldsymbol{\lambda}}}}\big\langle{\boldsymbol{\lambda}}_{l}^{k+1}-{\boldsymbol{\lambda}}_{l}^{k} - ({\boldsymbol{\lambda}}_{l}^{k}-{\boldsymbol{\lambda}}_{l}^{k-1}) ,
{\boldsymbol{\lambda}}_l^{k+1}-{\boldsymbol{\lambda}}_l^k\big\rangle
 \\
&\leq 
\left\langle\nabla_{{\boldsymbol{\lambda}}_{l}}\widetilde{L}_{p} (\{\boldsymbol{\mu}^{k+1}_{l}\}, \{ \mathbf{r}_{l}^{k+1}\},\{ \mathbf{p}_{l}^{k+1}\},\{ q_{s}^{k+1}\},\mathbf{x}^{k+1}, \{{\boldsymbol{\lambda}}_{l}^{k}\},\{{\boldsymbol{\alpha}}_{l}^{k}\},\{{\boldsymbol{\beta}}_{l}^{k}\}, \{\gamma_{s}^{k}\} 
,{\boldsymbol{\lambda}}_{l}^{k+1}-{\boldsymbol{\lambda}}_{l}^{k} \right\rangle 
\\
&-\left\langle\nabla_{{\boldsymbol{\lambda}}_{l}}\widetilde{L}_{p} (\{\boldsymbol{\mu}^{k}_{l}\}, \{ \mathbf{r}_{l}^{k}\},\{ \mathbf{p}_{l}^{k}\},\{ q_{s}^{k}\},\mathbf{x}^{k}, \{{\boldsymbol{\lambda}}_{l}^{k-1}\},\{{\boldsymbol{\alpha}}_{l}^{k-1}\},\{{\boldsymbol{\beta}}_{l}^{k-1}\}, \{\gamma_{s}^{k-1}\}
,{\boldsymbol{\lambda}}_{l}^{k+1}-{\boldsymbol{\lambda}}_{l}^{k} \right\rangle 
\\
&\leq \frac{1}{b_1^k}\| \boldsymbol{\mu}^{k+1}_{l} -\boldsymbol{\mu}^{k}_{l}  \|^2 + \frac{4w_{\mathbf{r}}^2}{b_1^k}\| \mathbf{r}^{k+1}_{l} -\mathbf{r}^{k}_{l}  \|^2 + 
\frac{b_1^k}{2}\| {\boldsymbol{\lambda}}_{l}^{k+1} -  {\boldsymbol{\lambda}}_{l}^{k}\|^2  + \frac{c_{{\boldsymbol{\lambda}}}^{k-1}-c_{{\boldsymbol{\lambda}}}^{k}}{2}(||{\boldsymbol{\lambda}}_{l}^{k+1}||^{2}-||{\boldsymbol{\lambda}}_{l}^{k}||^{2}) \\ 
& -\frac{c_{{\boldsymbol{\lambda}}}^{k-1}-c_{{\boldsymbol{\lambda}}}^{k}}{2}||{\boldsymbol{\lambda}}_{l}^{k+1}-{\boldsymbol{\lambda}}_{l}^{k}||^{2} + \frac{1}{2\eta_{{\boldsymbol{\lambda}}}} \| {\boldsymbol{\lambda}}_{l}^{k+1}-{\boldsymbol{\lambda}}_{l}^{k}-({\boldsymbol{\lambda}}_{l}^{k}-{\boldsymbol{\lambda}}_{l}^{k-1}) \|^2 -   \frac{c_{{\boldsymbol{\lambda}}}^{k-1}}{2}\|  {\boldsymbol{\lambda}}_{l}^{k} - {\boldsymbol{\lambda}}^{k-1}_{l}   \|^2.
\label{3_lambda_1}
\end{aligned}
\end{equation}
where $b_1^k > 0 $. 
And we have, 
\begin{equation}
\frac{1}{\eta_{{\boldsymbol{\lambda}}}}\big\langle{\boldsymbol{\lambda}}_{l}^{k+1}-{\boldsymbol{\lambda}}_{l}^{k} - ({\boldsymbol{\lambda}}_{l}^{k}-{\boldsymbol{\lambda}}_{l}^{k-1}) ,
{\boldsymbol{\lambda}}_l^{k+1}-{\boldsymbol{\lambda}}_l^k\big\rangle 
= \frac{1}{2\eta_{\boldsymbol{\lambda}}}||{\boldsymbol{\lambda}}_l^{k+1}-\lambda_l^k||^2 - \frac{1}{2\eta_{\boldsymbol{\lambda}}}||{\boldsymbol{\lambda}}_l^k-{\boldsymbol{\lambda}}_l^{k-1}||^2 +  \frac{1}{2\eta_{\boldsymbol{\lambda}}}
\| {\boldsymbol{\lambda}}_{l}^{k+1}-{\boldsymbol{\lambda}}_{l}^{k}-
({\boldsymbol{\lambda}}_{l}^{k}-{\boldsymbol{\lambda}}_{l}^{k-1} )\|^2.
\label{3_lambda_2}
\end{equation}
Combining (\ref{3_lambda_1}) with (\ref{3_lambda_2}),
we have 
\begin{equation}
\begin{aligned}
&\frac{1}{2\eta_{\boldsymbol{\lambda}}}||{\boldsymbol{\lambda}}_l^{k+1}-\boldsymbol\lambda_l^k||^2 - \frac{1}{2\eta_{\boldsymbol{\lambda}}}||{\boldsymbol{\lambda}}_l^k-{\boldsymbol{\lambda}}_l^{k-1}||^2
\leq \frac{1}{b_1^k}\| \boldsymbol{\mu}^{k+1}_{l} -\boldsymbol{\mu}^{k}_{l}  \|^2 + \frac{4w_{\mathbf{r}}^2}{b_1^k}\| \mathbf{r}^{k+1}_{l} -\mathbf{r}^{k}_{l}  \|^2 + 
\frac{b_1^k}{2}\| {\boldsymbol{\lambda}}_{l}^{k+1} -  {\boldsymbol{\lambda}}_{l}^{k}\|^2 \\
& + \frac{c_{{\boldsymbol{\lambda}}}^{k-1}-c_{{\boldsymbol{\lambda}}}^{k}}{2}(||{\boldsymbol{\lambda}}_{l}^{k+1}||^{2}-||{\boldsymbol{\lambda}}_{l}^{k}||^{2})-\frac{c_{{\boldsymbol{\lambda}}}^{k-1}-c_{{\boldsymbol{\lambda}}}^{k}}{2}||{\boldsymbol{\lambda}}_{l}^{k+1}-{\boldsymbol{\lambda}}_{l}^{k}||^{2}
 -   \frac{c_{{\boldsymbol{\lambda}}}^{k-1}}{2}\|  {\boldsymbol{\lambda}}_{l}^{k} - {\boldsymbol{\lambda}}^{k-1}_{l}   \|^2.
\end{aligned}
\end{equation}
Multiplying both sides of by $\frac{8}{\eta_{{\boldsymbol{\lambda}}}c_{{\boldsymbol{\lambda}}}^k}$, we have,
\begin{equation}
\begin{aligned}
&\frac{4}{\eta_{\boldsymbol{\lambda}}^2c_{{\boldsymbol{\lambda}}}^k}||{\boldsymbol{\lambda}}_l^{k+1}-{\boldsymbol{\lambda}}_l^k||^2-\frac{4}{\eta_{\boldsymbol{\lambda}}}(\frac{c_{{\boldsymbol{\lambda}}}^{k-1}-c_{{\boldsymbol{\lambda}}}^k}{c_{{\boldsymbol{\lambda}}}^k})||{\boldsymbol{\lambda}}_l^{k+1}||^2 \\
&\leq\frac{4}{\eta_{{\boldsymbol{\lambda}}}^{2}c_{{\boldsymbol{\lambda}}}^{k}}||{\boldsymbol{\lambda}}_{l}^{k}-{\boldsymbol{\lambda}}_{l}^{k-1}||^{2}-\frac{4}{\eta_{{\boldsymbol{\lambda}}}}(\frac{c_{{\boldsymbol{\lambda}}}^{k-1}-c_{{\boldsymbol{\lambda}}}^{k}}{c_{{\boldsymbol{\lambda}}}^{k}})||\boldsymbol\lambda_{l}^{k}||^{2}  \\ & 
+\frac{4b_{1}^{k}}{\eta_{{\boldsymbol{\lambda}}}c_{{\boldsymbol{\lambda}}}^{k}}||{\boldsymbol{\lambda}}_{l}^{k+1}-{\boldsymbol{\lambda}}_{l}^{k}||^{2}-\frac{4}{\eta_{{\boldsymbol{\lambda}}}}||{\boldsymbol{\lambda}}_{l}^{k}-{\boldsymbol{\lambda}}_{l}^{k-1}||^{2} \\
& + \frac{8}{b_1^k\eta_{{\boldsymbol{\lambda}}}c_{{\boldsymbol{\lambda}}}^k} \| \boldsymbol{\mu}_{l}^{k+1} - \boldsymbol\mu_{l}^k \|^2 
+ \frac{32w^2_{\mathbf{r}}}{b_1^k\eta_{{\boldsymbol{\lambda}}}c_{{\boldsymbol{\lambda}}}^k} \| \mathbf{r}_{l}^{k+1} - \mathbf{r}_{l}^k \|^2 .
\label{3_lambda_3}
\end{aligned}
\end{equation}
Setting $b_1^k = \frac{c_{{\boldsymbol{\lambda}}}^k}{2}$ in (\ref{3_lambda_3}) and combine it with the definition of $S_{1}^k$, we have,
\vspace{-5pt}
\begin{equation}
\label{3_2}
\begin{aligned}
S_1^{k+1}-S_1^k  &
\leq\sum\limits_{l=1}^{L}\frac{4}{\eta_{\boldsymbol{\lambda}}}(\frac{c_{{\boldsymbol{\lambda}}}^{k-2}}{c_{{\boldsymbol{\lambda}}}^{k-1}}-\frac{c_{{\boldsymbol{\lambda}}}^{k-1}}{c_{{\boldsymbol{\lambda}}}^{k}})||{\boldsymbol{\lambda}}_l^k||^2+\sum\limits_{l=1}^{L}(\frac{2}{\eta_{\boldsymbol{\lambda}}}+\frac{4}{\eta_{\boldsymbol{\lambda}}^2}(\frac{1}{c_{{\boldsymbol{\lambda}}}^{k+1}}-\frac{1}{c_{{\boldsymbol{\lambda}}}^{k}}))||{\boldsymbol{\lambda}}_l^{k+1}-{\boldsymbol{\lambda}}_l^k||^2 \\
&-\sum_{l=1}^{L}\frac{4}{\eta_{{\boldsymbol{\lambda}}}}||{\boldsymbol{\lambda}}_{l}^{k}-{\boldsymbol{\lambda}}_{l}^{k-1}||^{2} +  \frac{16}{\eta_{{\boldsymbol{\lambda}}}(c_{{\boldsymbol{\lambda}}}^k)^2} \sum_{l=1}^{L}\| \boldsymbol{\mu}_{l}^{k+1} - \boldsymbol\mu_{l}^k \|^2 
+ \frac{64w^2_{\mathbf{r}}}{\eta_{{\boldsymbol{\lambda}}}(c_{{\boldsymbol{\lambda}}}^k)^2} \sum_{l=1}^{L}\| \mathbf{r}_{l}^{k+1} - \mathbf{r}_{l}^k \|^2 .
\end{aligned}
\end{equation}
Similar results can be obtained for $S_2^k$, $S_3^k$, $S_4^k$,
\begin{equation}
\label{3_3}
\begin{aligned}
S_2^{k+1}-S_2^k &
\leq\sum\limits_{l=1}^{L}\frac{4}{\eta_{\boldsymbol{\alpha}}}(\frac{c_{{\boldsymbol{\alpha}}}^{k-2}}{c_{{\boldsymbol{\alpha}}}^{k-1}}-\frac{c_{{\boldsymbol{\alpha}}}^{k-1}}{c_{{\boldsymbol{\alpha}}}^{k}})||{\boldsymbol{\alpha}}^k_l||^2+\sum\limits_{l=1}^{L}(\frac{2}{\eta_{\boldsymbol{\alpha}}}+\frac{4}{\eta_{\boldsymbol{\alpha}}^2}(\frac{1}{c_{{\boldsymbol{\alpha}}}^{k+1}}-\frac{1}{c_{{\boldsymbol{\alpha}}}^{k}}))||{\boldsymbol{\alpha}}_l^{k+1}-{\boldsymbol{\alpha}}_l^k||^2 \\
&-\sum_{l=1}^{L}\frac{4}{\eta_{{\boldsymbol{\alpha}}}}||{\boldsymbol{\alpha}}_{l}^{k}-{\boldsymbol{\alpha}}_{l}^{k-1}||^{2} +  \frac{16}{\eta_{{\boldsymbol{\alpha}}}(c_{{\boldsymbol{\alpha}}}^k)^2} \sum_{l=1}^{L}\| \boldsymbol{\mu}_{l}^{k+1} - \boldsymbol\mu_{l}^k \|^2 
+ \frac{64w^2_{\mathbf{p}}}{\eta_{{\boldsymbol{\alpha}}}(c_{{\boldsymbol{\alpha}}}^k)^2} \sum_{l=1}^{L}\| \mathbf{p}_{l}^{k+1} - \mathbf{p}_{l}^k \|^2. 
\end{aligned}
\end{equation}
\begin{equation}
\label{3_4}
\begin{aligned}
S_3^{k+1}-S_3^k &
\leq\sum\limits_{l=1}^{L}\frac{4}{\eta_{\boldsymbol{\beta}}}(\frac{c_{{\boldsymbol{\beta}}}^{k-2}}{c_{{\boldsymbol{\beta}}}^{k-1}}-\frac{c_{{\boldsymbol{\beta}}}^{k-1}}{c_{{\boldsymbol{\beta}}}^{k}})||{\boldsymbol{\beta}}_l^k||^2+\sum\limits_{l=1}^{L}(\frac{2}{\eta_{\boldsymbol{\beta}}}+\frac{4}{\eta_{\boldsymbol{\beta}}^2}(\frac{1}{c_{{\boldsymbol{\beta}}}^{k+1}}-\frac{1}{c_{{\boldsymbol{\beta}}}^{k}}))||{\boldsymbol{\beta}}_{l}^{k+1}-{\boldsymbol{\beta}}_l^k||^2 \\
&-\sum_{l=1}^{L}\frac{4}{\eta_{{\boldsymbol{\beta}}}}||{\boldsymbol{\beta}}_{l}^{k}-{\boldsymbol{\beta}}_{l}^{k-1}||^{2} +  \frac{16}{\eta_{{\boldsymbol{\beta}}}(c_{{\boldsymbol{\beta}}}^k)^2} \sum_{l=1}^{L}\| \boldsymbol{\mu}_{l}^{k+1} - \boldsymbol\mu_{l}^k \|^2 
+ \frac{16}{\eta_{{\boldsymbol{\beta}}}(c_{{\boldsymbol{\beta}}}^k)^2} \sum_{l=1}^{L}\| \mathbf{B}_{l}\mathbf{x}^{k+1} - \mathbf{B}_{l}\mathbf{x}^k \|^2 .
\end{aligned}
\end{equation}
\begin{equation}
\label{3_5}
\begin{aligned}
S_4^{k+1}-S_4^k & 
\leq\sum\limits_{s=1}^{|\mathcal{P}^k|}\frac{4}{\eta_\gamma}(\frac{c_{\gamma}^{k-2}}{c_{\gamma}^{k-1}}-\frac{c_{\gamma}^{k-1}}{c_{\gamma}^{k}})||\gamma_s^k||^2+\sum\limits_{s=1}^{|\mathcal{P}^k|}(\frac{2}{\eta_\gamma}+\frac{4}{\eta_\gamma^2}(\frac{1}{c_{\gamma}^{k+1}}-\frac{1}{c_{\gamma}^{k}}))||\gamma_{s}^{k+1}-\gamma_s^k||^2 \\
&-\sum_{s=1}^{|\mathcal{P}^k|}\frac{4}{\eta_{\gamma}}||\gamma_{s}^{k}-\gamma_{s}^{k-1}||^{2} +  \frac{16}{\eta_{\gamma}(c_{\gamma}^k)^2} \sum_{s=1}^{|\mathcal{P}^k|} \|\mathbf{b}_s\|^2  \| \mathbf{x}^{k+1}_l - \mathbf{x}_{l}^k \|^2 
+ \frac{64w^2_{\mathbf{q}}}{\eta_{\gamma}(c_{\gamma}^k)^2} \sum_{s=1}^{|\mathcal{P}^k|}\| q_{s}^{k+1} - q_{s}^k \|^2 .
\end{aligned}    
\end{equation}
Based on the setting of $c_{{\boldsymbol{\lambda}}}^k$, $c_{{\boldsymbol{\alpha}}}^k$, $c_{{\boldsymbol{\beta}}}^k$ and $c_{\gamma}^k$, we can obtain that, $\frac{\eta_{{\boldsymbol{\lambda}}}}{10} \geq \frac{1}{c_{{\boldsymbol{\lambda}}}^{k+1}} - \frac{1}{c_{{\boldsymbol{\lambda}}}^k}$, $\frac{\eta_{{\boldsymbol{\alpha}}}}{10} \geq \frac{1}{c_{{\boldsymbol{\alpha}}}^{k+1}} - \frac{1}{c_{{\boldsymbol{\alpha}}}^k}$, $\frac{\eta_{{\boldsymbol{\beta}}}}{10} \geq \frac{1}{c_{{\boldsymbol{\beta}}}^{k+1}} - \frac{1}{c_{{\boldsymbol{\beta}}}^k}$, $\frac{\eta_{\gamma}}{10} \geq \frac{1}{c_{\gamma}^{k+1}} - \frac{1}{c_{\gamma}^k}$, $\forall k \geq K_1$.
In addition, according to the definition of $\mathbf{B}_l$, the following inequality can be obtained,
\vspace{-5pt}
\begin{equation}
\sum_{l=1}^{L}\| \mathbf{B}_{l}\mathbf{x}^{k+1} - \mathbf{B}_{l}\mathbf{x}^k \|^2 
 =  \| \mathbf{x}^{k+1} - \mathbf{x}^k \|^2 .
 \label{n_8}
\end{equation}
According to the definition of $a_5$,
combining (\ref{3_2})-(\ref{n_8}) with (\ref{3_1}), we can obtain that,
\begin{equation}
\begin{aligned}
F^{k+1} - F^{k} \leq & (a_5 - \frac{1}{\eta_{\boldsymbol{\mu}}^{k}} +   \frac{16}{\eta_{{\boldsymbol{\lambda}}}(c_{{\boldsymbol{\lambda}}}^k)^2}   + \frac{16}{\eta_{{\boldsymbol{\alpha}}}(c_{{\boldsymbol{\alpha}}}^k)^2}  +  \frac{16}{\eta_{{\boldsymbol{\beta}}}(c_{{\boldsymbol{\beta}}}^k)^2}  ) \sum_{l=1}^{L}\| \boldsymbol{\mu}_l^{k+1} - \boldsymbol{\mu}_l^k \|^2 \\
& +  (a_5-\frac{1}{\eta_{\mathbf{r}}^{k}} + \frac{64w^2_{\mathbf{r}}}{\eta_{{\boldsymbol{\lambda}}}(c_{{\boldsymbol{\lambda}}}^k)^2} )\sum_{l=1}^{L}\|\mathbf{r}_{l}^{k+1}-\mathbf{r}_{l}^{k}\|^2 \\
& + (a_5-\frac{1}{\eta_{\mathbf{p}}^{k}} + \frac{64w^2_{\mathbf{p}}}{\eta_{{\boldsymbol{\alpha}}}(c_{{\boldsymbol{\alpha}}}^k)^2} )\sum_{l=1}^{L}\|\mathbf{p}_{l}^{k+1}-\mathbf{p}_{l}^{k}\|^2 \\
& + (a_5-\frac{1}{\eta_{{q}}^{k}} + \frac{64w^2_{{q}}}{\eta_{{\gamma}}(c_{{\gamma}}^k)^2} )\sum_{s=1}^{|\mathcal{P}^k|}\|{q}_{s}^{k+1}-{q}_{s}^{k}\|^2 \\
& + (a_5 - \frac{1}{\eta_{\mathbf{x}}^{k}} + \frac{16}{\eta_{{\boldsymbol{\beta}}}(c_{{\boldsymbol{\beta}}}^k)^2}  + \frac{16}{\eta_{{\gamma}}(c_{{\gamma}}^k)^2} \sum_{s=1}^{|\mathcal{P}^k|}\|b_{s}\|^2
) \|\mathbf{x}^{k+1}-\mathbf{x}^{k}\|^2\\
& -\frac{1}{10\eta_{\boldsymbol{\lambda}}}\sum\limits_{l=1}^{L}||{\boldsymbol{\lambda}}_l^{k+1}-{\boldsymbol{\lambda}}_l^k||^2
+ \frac{c_{{\boldsymbol{\lambda}}}^{k-1}-c_{{\boldsymbol{\lambda}}}^k}{2}\sum\limits_{l=1}^{L}||{\boldsymbol{\lambda}}_l^{k+1}||^2
\\
& -\frac{1}{10\eta_{\boldsymbol{\alpha}}}\sum\limits_{l=1}^{L}||\boldsymbol\alpha_l^{k+1}-\boldsymbol\alpha_l^k||^2
+ \frac{c_{{\boldsymbol{\alpha}}}^{k-1}-c_{{\boldsymbol{\alpha}}}^k}{2}\sum\limits_{l=1}^{L}||{\boldsymbol{\alpha}}_l^{k+1}||^2
\\
& -\frac{1}{10\eta_{\boldsymbol{\beta}}}\sum\limits_{l=1}^{L}||{\boldsymbol{\beta}}_l^{k+1}-{\boldsymbol{\beta}}_l^k||^2
+ \frac{c_{{\boldsymbol{\beta}}}^{k-1}-c_{{\boldsymbol{\beta}}}^k}{2}\sum\limits_{l=1}^{L}||{\boldsymbol{\beta}}_l^{k+1}||^2
\\
& -\frac{1}{10\eta_\gamma}\sum\limits_{s=1}^{|\mathcal{P}^k|}||\gamma_s^{k+1}-\gamma_s^k||^2
+ \frac{c_{\gamma}^{k-1}-c_{\gamma}^k}{2}\sum\limits_{s=1}^{|\mathcal{P}^k|}||\gamma_s^{k+1}||^2
\\
& + \frac{4}{\eta_{{\boldsymbol{\lambda}}}}(\frac{c_{{\boldsymbol{\lambda}}}^{k-2}}{c_{{\boldsymbol{\lambda}}}^{k-1}}-\frac{c_{{\boldsymbol{\lambda}}}^{k-1}}{c_{{\boldsymbol{\lambda}}}^{k}})\sum_{l=1}^{L}||{\boldsymbol{\lambda}}_{l}^{k}||^{2}
+ \frac{4}{\eta_{{\boldsymbol{\alpha}}}}(\frac{c_{{\boldsymbol{\alpha}}}^{k-2}}{c_{{\boldsymbol{\alpha}}}^{k-1}}-\frac{c_{{\boldsymbol{\alpha}}}^{k-1}}{c_{{\boldsymbol{\alpha}}}^{k}})\sum_{l=1}^{L}||{\boldsymbol{\alpha}}_{l}^{k}||^{2}
\\
&+ \frac{4}{\eta_{{\boldsymbol{\beta}}}}(\frac{c_{{\boldsymbol{\beta}}}^{k-2}}{c_{{\boldsymbol{\beta}}}^{k-1}}-\frac{c_{{\boldsymbol{\beta}}}^{k-1}}{c_{{\boldsymbol{\beta}}}^{k}})\sum_{l=1}^{L}||{\boldsymbol{\beta}}_{l}^{k}||^{2}
+ \frac{4}{\eta_{\gamma}}(\frac{c_{\gamma}^{k-2}}{c_{\gamma}^{k-1}}-\frac{c_{\gamma}^{k-1}}{c_{\gamma}^{k}})\sum_{s=1}^{|\mathcal{P}^k|}||\gamma_{s}^{k}||^{2} .
\end{aligned}
\end{equation}
which concludes the proof of Lemma \ref{lemma_3}.
\end{proof}
\clearpage
Finally, based on \cref{lemma_1}, \cref{lemma_2} and \cref{lemma_3}, we provide the proof of Theorem \ref{theorem_2}.
\begin{proof}
First, we set
\begin{equation}
\begin{aligned}
a_6^k = \min\bigg\{  \frac{16}{\eta_{{\boldsymbol{\lambda}}}(c_{{\boldsymbol{\lambda}}}^k)^2}   + \frac{16}{\eta_{{\boldsymbol{\alpha}}}(c_{{\boldsymbol{\alpha}}}^k)^2}  +  \frac{16}{\eta_{{\boldsymbol{\beta}}}(c_{{\boldsymbol{\beta}}}^k)^2},
\frac{64w^2_{\mathbf{r}}}{\eta_{{\boldsymbol{\lambda}}}(c_{{\boldsymbol{\lambda}}}^k)^2},
\frac{64w^2_{\mathbf{p}}}{\eta_{{\boldsymbol{\alpha}}}(c_{{\boldsymbol{\alpha}}}^k)^2}, \\
\frac{64w^2_{{q}}}{\eta_{{\gamma}}(c_{{\gamma}}^k)^2},
\frac{16}{\eta_{{\boldsymbol{\beta}}}(c_{{\boldsymbol{\beta}}}^k)^2}  + \frac{16}{\eta_{{\gamma}}(c_{{\gamma}}^k)^2} \sum_{s=1}^{|\mathcal{P}^k|}\|b_{s}\|^2
\bigg\} \frac{\xi-2}{2} - a_5.
\end{aligned}
\end{equation}
where constant $\xi > 2 $ and satisfies 
\begin{equation}
\begin{aligned}
\min\bigg\{  \frac{16}{\eta_{{\boldsymbol{\lambda}}}(c_{{\boldsymbol{\lambda}}}^0)^2}   + \frac{16}{\eta_{{\boldsymbol{\alpha}}}(c_{{\boldsymbol{\alpha}}}^0)^2}  +  \frac{16}{\eta_{{\boldsymbol{\beta}}}(c_{{\boldsymbol{\beta}}}^0)^2},
\frac{64w^2_{\mathbf{r}}}{\eta_{{\boldsymbol{\lambda}}}(c_{{\boldsymbol{\lambda}}}^0)^2},
\frac{64w^2_{\mathbf{p}}}{\eta_{{\boldsymbol{\alpha}}}(c_{{\boldsymbol{\alpha}}}^0)^2}, \\
\frac{64w^2_{{q}}}{\eta_{{\gamma}}(c_{{\gamma}}^0)^2},
\frac{16}{\eta_{{\boldsymbol{\beta}}}(c_{{\boldsymbol{\beta}}}^0)^2}  + \frac{16}{\eta_{{\gamma}}(c_{{\gamma}}^0)^2} \sum_{s=1}^{|\mathcal{P}^k|}\|b_{s}\|^2
\bigg\} \frac{\xi-2}{2} > a_5.
\end{aligned}
\end{equation}
Thus, we have $a_6^k > 0, \forall k$.
According to the setting of $\eta^k_{{\boldsymbol{\lambda}}}$, $\eta^k_{{\boldsymbol{\alpha}}}$, $\eta^k_{{\boldsymbol{\beta}}}$, $\eta^k_{\gamma}$ and $c^k_{{\boldsymbol{\lambda}}}$, $c^k_{{\boldsymbol{\alpha}}}$, $c^k_{{\boldsymbol{\beta}}}$, $c^k_{\gamma}$, we have,
\begin{equation}
\begin{aligned}
a_5 - \frac{1}{\eta_{\boldsymbol{\mu}}^{k}} +   \frac{16}{\eta_{{\boldsymbol{\lambda}}}(c_{{\boldsymbol{\lambda}}}^k)^2}   + \frac{16}{\eta_{{\boldsymbol{\alpha}}}(c_{{\boldsymbol{\alpha}}}^k)^2}  +  \frac{16}{\eta_{{\boldsymbol{\beta}}}(c_{{\boldsymbol{\beta}}}^k)^2} & \leq -a_6^k, \\
a_5-\frac{1}{\eta_{\mathbf{r}}^{k}} + \frac{64w^2_{\mathbf{r}}}{\eta_{{\boldsymbol{\lambda}}}(c_{{\boldsymbol{\lambda}}}^k)^2} & \leq -a_6^k, \\
a_5-\frac{1}{\eta_{\mathbf{p}}^{k}} + \frac{64w^2_{\mathbf{p}}}{\eta_{{\boldsymbol{\alpha}}}(c_{{\boldsymbol{\alpha}}}^k)^2}  &\leq -a_6^k, \\
(a_5-\frac{1}{\eta_{{q}}^{k}} + \frac{64w^2_{{q}}}{\eta_{{\gamma}}(c_{{\gamma}}^k)^2} )&\leq -a_6^k, \\
a_5 - \frac{1}{\eta_{\mathbf{x}}^{k}} + \frac{16}{\eta_{{\boldsymbol{\beta}}}(c_{{\boldsymbol{\beta}}}^k)^2}  + \frac{16}{\eta_{{\gamma}}(c_{{\gamma}}^k)^2} \sum_{s=1}^{|\mathcal{P}^k|}\|b_{s}\|^2 &\leq -a_6^k .
\end{aligned}
\end{equation}
Combining it with Lemma \ref{lemma_3}, $\forall k \geq K_1$,  we can obtain that,
\begin{equation}
\label{add_1}
\begin{aligned}
& a_6^k\left(\sum_{l=1}^{L}\| \boldsymbol{\mu}_l^{k+1} - \boldsymbol{\mu}_l^k \|^2  + \sum_{l=1}^{L}\|\mathbf{r}_{l}^{k+1}-\mathbf{r}_{l}^{k}\|^2 + 
\sum_{l=1}^{L}\|\mathbf{p}_{l}^{k+1}-\mathbf{p}_{l}^{k}\|^2 + 
\sum_{s=1}^{|\mathcal{P}^k|}\|{q}_{s}^{k+1}-{q}_{s}^{k}\|^2 +
\|\mathbf{x}^{k+1}-\mathbf{x}^{k}\|^2  \right)\\
& \leq   F^{k} - F^{k+1}
-\frac{1}{10\eta_{\boldsymbol{\lambda}}}\sum\limits_{l=1}^{L}||{\boldsymbol{\lambda}}_l^{k+1}-{\boldsymbol{\lambda}}_l^k||^2
+ \frac{c_{{\boldsymbol{\lambda}}}^{k-1}-c_{{\boldsymbol{\lambda}}}^k}{2}\sum\limits_{l=1}^{L}||{\boldsymbol{\lambda}}_l^{k+1}||^2
+ \frac{4}{\eta_{{\boldsymbol{\lambda}}}}(\frac{c_{{\boldsymbol{\lambda}}}^{k-2}}{c_{{\boldsymbol{\lambda}}}^{k-1}}-\frac{c_{{\boldsymbol{\lambda}}}^{k-1}}{c_{{\boldsymbol{\lambda}}}^{k}})\sum_{l=1}^{L}||{\boldsymbol{\lambda}}_{l}^{k}||^{2}
\\
& -\frac{1}{10\eta_{\boldsymbol{\alpha}}}\sum\limits_{l=1}^{L}||\boldsymbol\alpha_l^{k+1}-\boldsymbol\alpha_l^k||^2
+ \frac{c_{{\boldsymbol{\alpha}}}^{k-1}-c_{{\boldsymbol{\alpha}}}^k}{2}\sum\limits_{l=1}^{L}||{\boldsymbol{\alpha}}_l^{k+1}||^2
+ \frac{4}{\eta_{{\boldsymbol{\alpha}}}}(\frac{c_{{\boldsymbol{\alpha}}}^{k-2}}{c_{{\boldsymbol{\alpha}}}^{k-1}}-\frac{c_{{\boldsymbol{\alpha}}}^{k-1}}{c_{{\boldsymbol{\alpha}}}^{k}})\sum_{l=1}^{L}||{\boldsymbol{\alpha}}_{l}^{k}||^{2}
\\
& -\frac{1}{10\eta_{\boldsymbol{\beta}}}\sum\limits_{l=1}^{L}||{\boldsymbol{\beta}}_l^{k+1}-{\boldsymbol{\beta}}_l^k||^2
+ \frac{c_{{\boldsymbol{\beta}}}^{k-1}-c_{{\boldsymbol{\beta}}}^k}{2}\sum\limits_{l=1}^{L}||{\boldsymbol{\beta}}_l^{k+1}||^2
+ \frac{4}{\eta_{{\boldsymbol{\beta}}}}(\frac{c_{{\boldsymbol{\beta}}}^{k-2}}{c_{{\boldsymbol{\beta}}}^{k-1}}-\frac{c_{{\boldsymbol{\beta}}}^{k-1}}{c_{{\boldsymbol{\beta}}}^{k}})\sum_{l=1}^{L}||{\boldsymbol{\beta}}_{l}^{k}||^{2}
\\
& -\frac{1}{10\eta_\gamma}\sum\limits_{s=1}^{|\mathcal{P}^k|}||\gamma_s^{k+1}-\gamma_s^k||^2
+ \frac{c_{\gamma}^{k-1}-c_{\gamma}^k}{2}\sum\limits_{s=1}^{|\mathcal{P}^k|}||\gamma_s^{k+1}||^2
+ \frac{4}{\eta_{\gamma}}(\frac{c_{\gamma}^{k-2}}{c_{\gamma}^{k-1}}-\frac{c_{\gamma}^{k-1}}{c_{\gamma}^{k}})\sum_{s=1}^{|\mathcal{P}^k|}||\gamma_{s}^{k}||^{2}.
\end{aligned}
\end{equation}
Given the definition of $\nabla\widetilde{G}^{k}$, we have that, 
\begin{equation}
\begin{aligned}
(\nabla\widetilde{G}^{k})_{{\boldsymbol{\mu}}_{l}} = & \nabla_{{\boldsymbol{\mu}}_{l}}
\widetilde{L}_{p} (\{\boldsymbol{\mu}^{\hat{k}_l}_{l}\}, \{ \mathbf{r}_{l}^{\hat{k}_l}\},\{ \mathbf{p}_{l}^{\hat{k}_l}\},\{ q_{s}^{\hat{k}_l}\},\mathbf{x}^{\hat{k}_l}, \{{\boldsymbol{\lambda}}_{l}^{\hat{k}}\},\{{\boldsymbol{\alpha}}_{l}^{\hat{k}}\},\{{\boldsymbol{\beta}}_{l}^{\hat{k}}\}, \{\gamma_{s}^{\hat{k}}\} ) \\
+ &  \nabla_{{\boldsymbol{\mu}}_{l}}
\widetilde{L}_{p} (\{\boldsymbol{\mu}^{k}_{l}\}, \{ \mathbf{r}_{l}^{k}\},\{ \mathbf{p}_{l}^{k}\},\{ q_{s}^{k}\},\mathbf{x}^{k}, \{{\boldsymbol{\lambda}}_{l}^{k}\},\{{\boldsymbol{\alpha}}_{l}^{k}\},\{{\boldsymbol{\beta}}_{l}^{k}\}, \{\gamma_{s}^{k}\} )\\
- &  \nabla_{{\boldsymbol{\mu}}_{l}}
\widetilde{L}_{p} (\{\boldsymbol{\mu}^{\hat{k}_l}_{l}\}, \{ \mathbf{r}_{l}^{\hat{k}_l}\},\{ \mathbf{p}_{l}^{\hat{k}_l}\},\{ q_{s}^{\hat{k}_l}\},\mathbf{x}^{\hat{k}_l}, \{{\boldsymbol{\lambda}}_{l}^{\hat{k}}\},\{{\boldsymbol{\alpha}}_{l}^{\hat{k}}\},\{{\boldsymbol{\beta}}_{l}^{\hat{k}}\}, \{\gamma_{s}^{\hat{k}}\} ).\\
\end{aligned}
\end{equation}
Combining it with (\ref{lm_1}), we have that,
\begin{equation}
\|(\nabla \widetilde{G}^k)_{\boldsymbol{\mu}_l} \|^2\leq \frac{1}{\eta_{\boldsymbol{\mu}}^2}\| \boldsymbol{\mu}_l^{\bar{k}_l} -   \boldsymbol{\mu}_l^k \|^2.
\label{th_1}
\end{equation}
Similar results can be derived for other variables, 
\begin{equation}
\begin{aligned}
& \|(\nabla \widetilde{G}^k)_{\mathbf{r}_l}\|^2 \leq \frac{1}{\eta_{\mathbf{r}}^2}\| \mathbf{r}_l^{\bar{k}_l} -   \mathbf{r}_l^k \|^2 .\\
& \|(\nabla \widetilde{G}^k)_{\mathbf{p}_l}\|^2 \leq \frac{1}{\eta_{\mathbf{p}}^2}\| \mathbf{p}_l^{\bar{k}_l} -   \mathbf{p}_l^k \|^2. \\
& \|(\nabla \widetilde{G}^k)_{q_s}\|^2 \leq \frac{1}{\eta_{q}^2}\| q_s^{k+1} - q_s^k \|^2. \\
& \|(\nabla \widetilde{G}^k)_{\mathbf{x}}\|^2 \leq \frac{1}{\eta_{\mathbf{x}}^2}\| \mathbf{x}^{k+1} -   \mathbf{x}^k \|^2. \\
\end{aligned}
\end{equation}
According to the definition \ref{definition_3}, we have,  
\begin{equation}
\begin{aligned}
(\nabla\widetilde{G}^{k})_{{\boldsymbol{\lambda}}_{l}} = & \nabla_{{\boldsymbol{\lambda}}_{l}}
\widetilde{L}_{p} (\{\boldsymbol{\mu}^{\bar{k}_l}_{l}\}, \{ \mathbf{r}_{l}^{\bar{k}_l}\},\{ \mathbf{p}_{l}^{\bar{k}_l}\},\{ q_{s}^{\bar{k}_l}\},\mathbf{x}^{\bar{k}_l}, \{{\boldsymbol{\lambda}}_{l}^{k}\},\{{\boldsymbol{\alpha}}_{l}^{k}\},\{{\boldsymbol{\beta}}_{l}^{k}\}, \{\gamma_{s}^{k}\} ) \\
+ &  \nabla_{{\boldsymbol{\lambda}}_{l}}
\widetilde{L}_{p} (\{\boldsymbol{\mu}^{k}_{l}\}, \{ \mathbf{r}_{l}^{k}\},\{ \mathbf{p}_{l}^{k}\},\{ q_{s}^{k}\},\mathbf{x}^{k}, \{{\boldsymbol{\lambda}}_{l}^{k}\},\{{\boldsymbol{\alpha}}_{l}^{k}\},\{{\boldsymbol{\beta}}_{l}^{k}\}, \{\gamma_{s}^{k}\} )\\
- &  \nabla_{{\boldsymbol{\lambda}}_{l}}
\widetilde{L}_{p} (\{\boldsymbol{\mu}^{\bar{k}_l}_{l}\}, \{ \mathbf{r}_{l}^{\bar{k}_l}\},\{ \mathbf{p}_{l}^{\bar{k}_l}\},\{ q_{s}^{\bar{k}_l}\},\mathbf{x}^{\bar{k}_l}, \{{\boldsymbol{\lambda}}_{l}^{k}\},\{{\boldsymbol{\alpha}}_{l}^{k}\},\{{\boldsymbol{\beta}}_{l}^{k}\}, \{\gamma_{s}^{k}\} ).
\end{aligned}
\end{equation}
Combining trigonometric inequality, (\ref{lambda_update}) with Assumption \ref{assumption_1}, we can obtain 
\begin{equation}
\begin{aligned}
\|(\nabla\widetilde{G}^{k})_{{\boldsymbol{\lambda}}_{l}}\|^{2} \leq &  3\| \nabla_{{\boldsymbol{\lambda}}_{l}}
\widetilde{L}_{p} (\{\boldsymbol{\mu}^{\bar{k}_l}_{l}\}, \{ \mathbf{r}_{l}^{\bar{k}_l}\},\{ \mathbf{p}_{l}^{\bar{k}_l}\},\{ q_{s}^{\bar{k}_l}\},\mathbf{x}^{\bar{k}_l}, \{{\boldsymbol{\lambda}}_{l}^{k}\},\{{\boldsymbol{\alpha}}_{l}^{k}\},\{{\boldsymbol{\beta}}_{l}^{k}\}, \{\gamma_{s}^{k}\} ) \|^2   \\
  + & 3((c_{{\boldsymbol{\lambda}}}^{\hat{k}_l-1})^{2}-(c_{{\boldsymbol{\lambda}}}^{\bar{k}_l-1})^{2})||{\boldsymbol{\lambda}}_{l}^{k}||^{2} +  3 \|  \boldsymbol{\mu}_{l}^{\bar{k}} - \boldsymbol{\mu}_{l}^{k} + \mathbf{r}_{l}^{k} \circ  \mathbf{r}_{l}^{k} - \mathbf{r}_{l}^{\bar{k}} \circ  \mathbf{r}_{l}^{\bar{k}}
\|^2   \\
 \leq & \frac{3}{\eta_{{\boldsymbol{\lambda}}}^2} ||{\boldsymbol{\lambda}}_{l}^{\bar{k}_l}-{\boldsymbol{\lambda}}_{l}^{k}||^{2}+3((c_{{\boldsymbol{\lambda}}}^{\hat{k}_l-1})^{2}-(c_{{\boldsymbol{\lambda}}}^{\bar{k}_l-1})^{2})||{\boldsymbol{\lambda}}_{l}^{k}||^{2} 
\\  + & 6\| \boldsymbol{\mu}_l^{\bar{k}_l} -   \boldsymbol{\mu}_l^k \|^2 + 24w_{\mathbf{r}}^2\| \mathbf{r}_l^{\bar{k}_l} -   \mathbf{r}_l^k \|^2.  
\end{aligned}
\end{equation}
Similar results can be derived for other variables as well,
\begin{equation}
\begin{aligned}
 \|(\nabla\widetilde{G}^{k})_{{\boldsymbol{\alpha}}_{l}}\|^{2} & \leq\frac{3}{\eta_{{\boldsymbol{\alpha}}}^{2}}||{\boldsymbol{\alpha}}_{l}^{\bar{k}_l}-{\boldsymbol{\alpha}}_{l}^{k}||^{2}+3((c_{{\boldsymbol{\alpha}}}^{\hat{k}_l-1})^{2}-(c_{{\boldsymbol{\alpha}}}^{\bar{k}_l-1})^{2})||{\boldsymbol{\alpha}}_{l}^{k}||^{2} 
 \\
& + 6\| \boldsymbol{\mu}_l^{\bar{k}_l} -   \boldsymbol{\mu}_l^k \|^2 + 24w_{\mathbf{p}}^2\| \mathbf{p}_l^{\bar{k}_l} -   \mathbf{p}_l^k \|^2.  \\
 \|(\nabla\widetilde{G}^{k})_{{\boldsymbol{\beta}}_{l}}\|^{2} & \leq\frac{3}{\eta_{{\boldsymbol{\beta}}}^{2}}||{\boldsymbol{\beta}}_{l}^{\bar{k}_l}-{\boldsymbol{\beta}}_{l}^{k}||^{2}+3((c_{{\boldsymbol{\beta}}}^{\hat{k}_l-1})^{2}-(c_{{\boldsymbol{\beta}}}^{\bar{k}_l-1})^{2})||{\boldsymbol{\beta}}_{l}^{k}||^{2} 
 \\
& + 6\| \boldsymbol{\mu}_l^{\bar{k}_l} -   \boldsymbol{\mu}_l^k \|^2 + 6\| \mathbf{x}^{\bar{k}_l} -   \mathbf{x}^k \|^2 . \\
\|(\nabla\widetilde{G}^{k})_{\gamma_{s}}\|^{2} & \leq\frac{3}{\eta_{\gamma}^{2}}||\gamma_{s}^{k+1}-\gamma_{s}^{k}||^{2}+3((c_{\gamma}^{k-1})^{2}-(c_{\gamma}^{k})^{2})||\gamma_{s}^{k}||^{2} 
 \\
&+ 6 \|\mathbf{b}_s\|^2 \| \mathbf{x}^{k+1} - \mathbf{x}^{k} \|^2 + 24w_q^2\| q^{k+1}_s -   q_s^k \|^2.  \\
\end{aligned}
\label{th_2}
\end{equation}
According to Assumption \ref{assumption_1}, we have,
\begin{equation}
\| \mathbf{x}^{\bar{k}_l} -   \mathbf{x}^k \|^2 \leq \tau k_{1}\vartheta \leq \tau k_{1}\| \mathbf{x}^{k+1} - \mathbf{x}^{k} \|^{2}.
\label{v_tau}
\end{equation}
Combining it with (\ref{th_1})-(\ref{th_2}), we can obtain that,
\begin{equation}
\begin{aligned}
||\nabla \widetilde{G}^{k}||^{2}
&=\sum_{l=1}^{L}(||(\nabla \widetilde{G}^{k})_{\boldsymbol{\mu}_{l}}||^{2}+||(\nabla \widetilde{G}^{k})_{\mathbf{r}_{l}}||^{2}+||\nabla G^{k})_{\mathbf{p}_{l}}||^{2} + ||\nabla \widetilde{G}^{k})_{{\boldsymbol{\lambda}}_{l}}||^{2} +||(\nabla G^{k})_{{\boldsymbol{\alpha}}_{l}}||^{2} + ||(\nabla \widetilde{G}^{k})_{{\boldsymbol{\beta}}_{l}}||^{2} ) \\
&  + \sum_{s=1}^{|\mathcal{P}^{k}|}||(\nabla \widetilde{G}^{k})_{{q}_{s}}||^{2}
+||(\nabla \widetilde{G}^{k})_{\mathbf{x}}||^{2}+\sum_{s=1}^{|\mathcal{P}^{k}|}||(\nabla \widetilde{G}^{k})_{\gamma_{s}}||^{2}
\\
& \leq (\frac{1}{\eta_{\boldsymbol{\mu}}^2} + 18)  \sum_{l=1}^{L}\| \boldsymbol{\mu}_l^{\bar{k}_l} -   \boldsymbol{\mu}_l^k \|^2 +  (\frac{1}{\eta_{\mathbf{r}}^2}+ 24w_{\mathbf{r}}^2)\sum_{l=1}^{L}\| \mathbf{r}_l^{\bar{k}_l} -   \mathbf{r}_l^k \|^2 
\\& +  (\frac{1}{\eta_{\mathbf{p}}^2}+24w_{\mathbf{p}}^2 )\sum_{l=1}^{L}\| \mathbf{p}_l^{\bar{k}_l} -   \mathbf{p}_l^k \|^2\\
& + (\frac{1}{\eta_{q}^2}+ 24w_{\mathbf{q}}^2)\sum_{s=1}^{|\mathcal{P}^k|}\| q_s^{k+1} - q_s^k \|^2 
\\& + (\frac{1}{\eta_{\mathbf{x}}^2} + \sum_{s=1}^{|\mathcal{P}^k|}6 \|\mathbf{b}_s\|^2 + 6L\tau k_1)\| \mathbf{x}^{k+1} -   \mathbf{x}^k \|^2  \\
& + \frac{3}{\eta_{{\boldsymbol{\lambda}}}^2}\sum_{l=1}^{L} ||{\boldsymbol{\lambda}}_{l}^{\bar{k}_l}-{\boldsymbol{\lambda}}_{l}^{k}||^{2}+3\sum_{l=1}^{L}((c_{{\boldsymbol{\lambda}}}^{\hat{k}_l-1})^{2}-(c_{{\boldsymbol{\lambda}}}^{\bar{k}_l-1})^{2})||{\boldsymbol{\lambda}}_{l}^{k}||^{2} \\
& + \frac{3}{\eta_{{\boldsymbol{\alpha}}}^{2}}\sum_{l=1}^L||{\boldsymbol{\alpha}}_{l}^{\bar{k}_l}-{\boldsymbol{\alpha}}_{l}^{k}||^{2}+3\sum_{l=1}^L((c_{{\boldsymbol{\alpha}}}^{\hat{k}_l-1})^{2}-(c_{{\boldsymbol{\alpha}}}^{\bar{k}_l-1})^{2})||{\boldsymbol{\alpha}}_{l}^{k}||^{2} \\
& +\frac{3}{\eta_{{\boldsymbol{\beta}}}^{2}}\sum_{l=1}^L||{\boldsymbol{\beta}}_{l}^{\bar{k}_l}-{\boldsymbol{\beta}}_{l}^{k}||^{2}+3\sum_{l=1}^L((c_{{\boldsymbol{\beta}}}^{\hat{k}_l-1})^{2}-(c_{{\boldsymbol{\beta}}}^{\bar{k}_l-1})^{2})||{\boldsymbol{\beta}}_{l}^{k}||^{2} \\
& + \frac{3}{\eta_{\gamma}^{2}}\sum_{s=1}^{|\mathcal{P}^k|}||\gamma_{s}^{k+1}-\gamma_{s}^{k}||^{2}+3\sum_{s=1}^{|\mathcal{P}^k|}((c_{\gamma}^{k-1})^{2}-(c_{\gamma}^{k})^{2})||\gamma_{s}^{k}||^{2} .
\label{th_3}
\end{aligned}
\end{equation}
Let constant $\underline{a_6}$ denote the lower bound of $a_6^k$ ($\underline{a_6}>0$), and we set constants $d_1$, $d_2$, $d_3$, $d_4$, $d_5$ that,
\begin{equation}
\label{th_4}
\begin{aligned}
 d_{1} = \frac{k_{\tau}\tau + 18k_{\tau}\tau   \eta_{\boldsymbol{\mu}}^2}
{{\eta_{\boldsymbol\mu}}^{2}(\underline{a_{6}})^{2}} &\geq  \frac{k_{\tau}\tau + 18k_{\tau}\tau   \eta_{\boldsymbol{\mu}}^2}
{{\eta_{\boldsymbol\mu}}^{2}(a_{6}^k)^{2}},
\\
 d_2 = \frac{k_{\tau}\tau + 24w_{\mathbf{r}}^2 k_{\tau}\tau {\eta_{\mathbf r}}^{2}  }
{{\eta_{\mathbf r}}^{2}(\underline{a_{6}})^{2}}  &\geq  \frac{k_{\tau}\tau + 24w_{\mathbf{r}}^2 k_{\tau}\tau {\eta_{\mathbf r}}^{2}  }
{{\eta_{\mathbf r}}^{2}(a^k_{6})^{2}} ,
\\
 d_3 = \frac{k_{\tau}\tau + 24w_{\mathbf{p}}^2 k_{\tau}\tau {\eta_{\mathbf p}}^{2}  }
{{\eta_{\mathbf p}}^{2}(\underline{a_{6}})^{2}}  &\geq  \frac{k_{\tau}\tau + 24w_{\mathbf{p}}^2 k_{\tau}\tau {\eta_{\mathbf p}}^{2}  }
{{\eta_{\mathbf p}}^{2}({a^k_{6}})^{2}} ,
\\
 d_4 = \frac{1 + 24w_{{q}}^2  {\eta_{ q}}^{2}  }
{{\eta_{ q}}^{2}(\underline{a_{6}})^{2}} & \geq  \frac{1 + 24w_{{q}}^2 {\eta_{ q}}^{2}  }
{{\eta_{ q}}^{2}({a_{6}^k})^{2}},
\\
d_5 = \frac{1 + (\sum_{s=1}^{|\mathcal{P}^k|}6 \|\mathbf{b}_s\|^2 + 6L\tau k_1)  {\eta_{\mathbf x}}^{2}  }
{{\eta_{\mathbf x}}^{2}(\underline{a_{6}})^{2}} &\geq \frac{1 + (\sum_{s=1}^{|\mathcal{P}^k|}6 \|\mathbf{b}_s\|^2 + 6L\tau k_1)  {\eta_{\mathbf x}}^{2}  }
{{\eta_{\mathbf x}}^{2}({a_{6}^k})^{2}}.
\end{aligned}
\end{equation}
where $k_{\tau}$ is a positive constant.
By employing (\ref{th_3}) and (\ref{th_4}), we can obtain, 
\vspace{-10pt}
\begin{equation}
\label{th_5}
\begin{aligned}
||\nabla \widetilde{G}^{k}||^{2} 
 & \leq  (a_6^k)^2\left(d_1\sum_{l=1}^{L}\| \boldsymbol{\mu}_l^{k+1} - \boldsymbol{\mu}_l^k \|^2  + d_2\sum_{l=1}^{L}\|\mathbf{r}_{l}^{k+1}-\mathbf{r}_{l}^{k}\|^2 + 
d_3\sum_{l=1}^{L}\|\mathbf{p}_{l}^{k+1}-\mathbf{p}_{l}^{k}\|^2 
 \right)\\
& + (a_6^k)^2\left(  d_4\sum_{s=1}^{|\mathcal{P}^k|}\|{q}_{s}^{k+1}-{q}_{s}^{k}\|^2 +
d_5\|\mathbf{x}^{k+1}-\mathbf{x}^{k}\|^2    \right) \\
& + (\frac{1}{\eta_{\boldsymbol{\mu}}^2} + 18)  \sum_{l=1}^{L}\| \boldsymbol{\mu}_l^{\bar{k}_l} -   \boldsymbol{\mu}_l^k \|^2  - (\frac{1}{\eta_{\boldsymbol{\mu}}^2} + 18) k_{\tau}\tau \sum_{l=1}^{L}\| \boldsymbol{\mu}_l^{k+1} -   \boldsymbol{\mu}_l^k \|^2 \\
& +  (\frac{1}{\eta_{\mathbf{r}}^2} + 24w_{\mathbf{r}}^2)\sum_{l=1}^{L}\| \mathbf{r}_l^{\bar{k}_l} -   \mathbf{r}_l^k \|^2 -(\frac{1}{\eta_{\mathbf{r}}^2} + 24w_{\mathbf{r}}^2)k_{\tau}\tau\sum_{l=1}^{L}\| \mathbf{r}_l^{k+1} -   \mathbf{r}_l^k \|^2 \\
& +  (\frac{1}{\eta_{\mathbf{p}}^2}+24w_{\mathbf{p}}^2 )\sum_{l=1}^{L}\| \mathbf{p}_l^{\bar{k}_l} -   \mathbf{p}_l^k \|^2 - (\frac{1}{\eta_{\mathbf{p}}^2}+24w_{\mathbf{p}}^2 )k_{\tau}\tau\sum_{l=1}^{L}\| \mathbf{p}_l^{k+1} -   \mathbf{p}_l^k \|^2 \\
& + \frac{3}{\eta_{{\boldsymbol{\lambda}}}^2}\sum_{l=1}^{L} ||{\boldsymbol{\lambda}}_{l}^{\bar{k}_l}-{\boldsymbol{\lambda}}_{l}^{k}||^{2}+3\sum_{l=1}^{L}((c_{{\boldsymbol{\lambda}}}^{\hat{k}_l-1})^{2}-(c_{{\boldsymbol{\lambda}}}^{\bar{k}_l-1})^{2})||{\boldsymbol{\lambda}}_{l}^{k}||^{2} \\
& + \frac{3}{\eta_{{\boldsymbol{\alpha}}}^{2}}\sum_{l=1}^L||{\boldsymbol{\alpha}}_{l}^{\bar{k}_l}-{\boldsymbol{\alpha}}_{l}^{k}||^{2}+3\sum_{l=1}^L((c_{{\boldsymbol{\alpha}}}^{\hat{k}_l-1})^{2}-(c_{{\boldsymbol{\alpha}}}^{\bar{k}_l-1})^{2})||{\boldsymbol{\alpha}}_{l}^{k}||^{2} \\
& +\frac{3}{\eta_{{\boldsymbol{\beta}}}^{2}}\sum_{l=1}^L||{\boldsymbol{\beta}}_{l}^{\bar{k}_l}-{\boldsymbol{\beta}}_{l}^{k}||^{2}+3\sum_{l=1}^L((c_{{\boldsymbol{\beta}}}^{\hat{k}_l-1})^{2}-(c_{{\boldsymbol{\beta}}}^{\bar{k}_l-1})^{2})||{\boldsymbol{\beta}}_{l}^{k}||^{2} \\
& + \frac{3}{\eta_{\gamma}^{2}}\sum\nolimits_{s=1}^{|\mathcal{P}^k|}||\gamma_{s}^{k+1}-\gamma_{s}^{k}||^{2}+3\sum\nolimits_{s=1}^{|\mathcal{P}^k|}((c_{\gamma}^{k-1})^{2}-(c_{\gamma}^{k})^{2})||\gamma_{s}^{k}||^{2}. 
\end{aligned}
\end{equation}
On the basis of (\ref{th_5}), we can further obtain that,
\begin{equation}
\label{add_2}
\begin{aligned}
d_6^k||\nabla \widetilde{G}^{k}||^{2}
 \leq & a_6^k\left(\sum_{l=1}^{L}\| \boldsymbol{\mu}_l^{k+1} - \boldsymbol{\mu}_l^k \|^2  + \sum_{l=1}^{L}\|\mathbf{r}_{l}^{k+1}-\mathbf{r}_{l}^{k}\|^2 + 
\sum_{l=1}^{L}\|\mathbf{p}_{l}^{k+1}-\mathbf{p}_{l}^{k}\|^2 + 
\sum_{s=1}^{|\mathcal{P}^k|}\|{q}_{s}^{k+1}-{q}_{s}^{k}\|^2 +
\|\mathbf{x}^{k+1}-\mathbf{x}^{k}\|^2  \right)\\
& + d_6^k(\frac{1}{\eta_{\boldsymbol{\mu}}^2} + 18)  \sum_{l=1}^{L}\| \boldsymbol{\mu}_l^{\bar{k}_l} -   \boldsymbol{\mu}_l^k \|^2  - d_6^k(\frac{1}{\eta_{\boldsymbol{\mu}}^2} + 18) k_{\tau}\tau \sum_{l=1}^{L}\| \boldsymbol{\mu}_l^{k+1} -   \boldsymbol{\mu}_l^k \|^2 \\
& +  d_6^k(\frac{1}{\eta_{\mathbf{r}}^2} + 24w_{\mathbf{r}}^2)\sum_{l=1}^{L}\| \mathbf{r}_l^{\bar{k}_l} -   \mathbf{r}_l^k \|^2 -d_6^k(\frac{1}{\eta_{\mathbf{r}}^2} + 24w_{\mathbf{r}}^2)k_{\tau}\tau\sum_{l=1}^{L}\| \mathbf{r}_l^{k+1} -   \mathbf{r}_l^k \|^2 \\
& +  d_6^k(\frac{1}{\eta_{\mathbf{p}}^2}+24w_{\mathbf{p}}^2 )\sum_{l=1}^{L}\| \mathbf{p}_l^{\bar{k}_l} -   \mathbf{p}_l^k \|^2 -d_6^k(\frac{1}{\eta_{\mathbf{p}}^2}+24w_{\mathbf{p}}^2 )k_{\tau}\tau\sum_{l=1}^{L}\| \mathbf{p}_l^{k+1} -   \mathbf{p}_l^k \|^2 \\
& + \frac{1}{10\tau\eta_{{\boldsymbol{\lambda}}}}\sum_{l=1}^{L} ||{\boldsymbol{\lambda}}_{l}^{\bar{k}_l}-{\boldsymbol{\lambda}}_{l}^{k}||^{2}+3d_6^k\sum_{l=1}^{L}((c_{{\boldsymbol{\lambda}}}^{\hat{k}_l-1})^{2}-(c_{{\boldsymbol{\lambda}}}^{\bar{k}_l-1})^{2})||{\boldsymbol{\lambda}}_{l}^{k}||^{2} \\
& + \frac{1}{10\tau\eta_{{\boldsymbol{\alpha}}}}\sum_{l=1}^L||{\boldsymbol{\alpha}}_{l}^{\bar{k}_l}-{\boldsymbol{\alpha}}_{l}^{k}||^{2}+3d_6^k\sum_{l=1}^L((c_{{\boldsymbol{\alpha}}}^{\hat{k}_l-1})^{2}-(c_{{\boldsymbol{\alpha}}}^{\bar{k}_l-1})^{2})||{\boldsymbol{\alpha}}_{l}^{k}||^{2} \\
& +\frac{1}{10\tau\eta_{{\boldsymbol{\beta}}}}\sum_{l=1}^L||{\boldsymbol{\beta}}_{l}^{\bar{k}_l}-{\boldsymbol{\beta}}_{l}^{k}||^{2}+3d_6^k\sum_{l=1}^L((c_{{\boldsymbol{\beta}}}^{\hat{k}_l-1})^{2}-(c_{{\boldsymbol{\beta}}}^{\bar{k}_l-1})^{2})||{\boldsymbol{\beta}}_{l}^{k}||^{2} \\
& + \frac{1}{10\eta_{\gamma}}\sum_{s=1}^{|\mathcal{P}^k|}||\gamma_{s}^{k+1}-\gamma_{s}^{k}||^{2}+3d_6^k\sum_{s=1}^{|\mathcal{P}^k|}((c_{\gamma}^{k-1})^{2}-(c_{\gamma}^{k})^{2})||\gamma_{s}^{k}||^{2},
\end{aligned}
\end{equation}
where $d_6^k$ represents a nonnegative sequence, i.e, 
\begin{equation}
d_6^k = \left( 
\operatorname*{max}\left\{d_{1}a_{6}^{k},d_{2}a_{6}^{k},d_{3}a_{6}^{k},d_{4}a_{6}^{k},d_{5}a_{6}^{k},
\frac{30\tau}{\eta_{{\boldsymbol{\lambda}}}},\frac{30\tau}{\eta_{{\boldsymbol{\alpha}}}},\frac{30\tau}{\eta_{{\boldsymbol{\beta}}}},\frac{30}{\eta_{{\gamma}}}\right\}
\right)^{-1}
.
\end{equation}
And the upper and lower bound of $d_6$ is denoted as $\overline{d_6}$ and $\underline{d_6}$, respectively. 
According to Assumption \ref{assumption_1} and combining (\ref{add_1}) with (\ref{add_2}), we have,
\vspace{-10pt}
\begin{equation}
\label{add_3}
\begin{aligned}
d_6^k||\nabla \widetilde{G}^{k}||^{2}
& \leq  F^{k} - F^{k+1} 
\\ & 
+ \overline{d_6}(\frac{1}{\eta_{\boldsymbol{\mu}}^2} + 18)  \sum_{l=1}^{L}\| \boldsymbol{\mu}_l^{\bar{k}_l} -   \boldsymbol{\mu}_l^k \|^2  - \underline{d_6}(\frac{1}{\eta_{\boldsymbol{\mu}}^2} + 18) k_{\tau}\tau \sum_{l=1}^{L}\| \boldsymbol{\mu}_l^{k+1} -   \boldsymbol{\mu}_l^k \|^2 \\
& +  \overline{d_6}(\frac{1}{\eta_{\mathbf{r}}^2} + 24w_{\mathbf{r}}^2)\sum_{l=1}^{L}\| \mathbf{r}_l^{\bar{k}_l} -   \mathbf{r}_l^k \|^2 -\underline{d_6}(\frac{1}{\eta_{\mathbf{r}}^2} + 24w_{\mathbf{r}}^2)k_{\tau}\tau\sum_{l=1}^{L}\| \mathbf{r}_l^{k+1} -   \mathbf{r}_l^k \|^2 \\
& +  \overline{d_6}(\frac{1}{\eta_{\mathbf{p}}^2}+24w_{\mathbf{p}}^2 )\sum_{l=1}^{L}\| \mathbf{p}_l^{\bar{k}_l} -   \mathbf{p}_l^k \|^2 -\underline{d_6}(\frac{1}{\eta_{\mathbf{p}}^2}+24w_{\mathbf{p}}^2 )k_{\tau}\tau\sum_{l=1}^{L}\| \mathbf{p}_l^{k+1} -   \mathbf{p}_l^k \|^2 \\
& + \frac{1}{10\tau\eta_{{\boldsymbol{\lambda}}}}\sum_{l=1}^{L} ||{\boldsymbol{\lambda}}_{l}^{\bar{k}_l}-{\boldsymbol{\lambda}}_{l}^{k}||^{2} -\frac{1}{10\eta_{{\boldsymbol{\lambda}}}}\sum_{l=1}^{L} ||{\boldsymbol{\lambda}}_{l}^{k+1}-{\boldsymbol{\lambda}}_{l}^{k}||^{2}
 \\
& + \frac{1}{10\tau\eta_{{\boldsymbol{\alpha}}}}\sum_{l=1}^L||{\boldsymbol{\alpha}}_{l}^{\bar{k}_l}-{\boldsymbol{\alpha}}_{l}^{k}||^{2}
-\frac{1}{10\eta_{{\boldsymbol{\alpha}}}}\sum_{l=1}^L||{\boldsymbol{\alpha}}_{l}^{\bar{k}_l}-{\boldsymbol{\alpha}}_{l}^{k}||^{2}
 \\
& +\frac{1}{10\tau\eta_{{\boldsymbol{\beta}}}}\sum_{l=1}^L||{\boldsymbol{\beta}}_{l}^{\bar{k}_l}-{\boldsymbol{\beta}}_{l}^{k}||^{2}
- \frac{1}{10\eta_{{\boldsymbol{\beta}}}}\sum_{l=1}^L||{\boldsymbol{\beta}}_{l}^{\bar{k}_l}-{\boldsymbol{\beta}}_{l}^{k}||^{2}
 \\
&+3\overline{d_6}\sum_{l=1}^{L}((c_{{\boldsymbol{\lambda}}}^{\hat{k}_l-1})^{2}-(c_{{\boldsymbol{\lambda}}}^{\bar{k}_l-1})^{2})||{\boldsymbol{\lambda}}_{l}^{k}||^{2}
\\
&+3\overline{d_6}\sum_{l=1}^L((c_{{\boldsymbol{\alpha}}}^{\hat{k}_l-1})^{2}-(c_{{\boldsymbol{\alpha}}}^{\bar{k}_l-1})^{2})||{\boldsymbol{\alpha}}_{l}^{k}||^{2} \\& +3\overline{d_6}\sum_{l=1}^L((c_{{\boldsymbol{\beta}}}^{\hat{k}_l-1})^{2}-(c_{{\boldsymbol{\beta}}}^{\bar{k}_l-1})^{2})||{\boldsymbol{\beta}}_{l}^{k}||^{2}  \\
& +3\overline{d_6}\sum_{s=1}^{|\mathcal{P}^k|}((c_{\gamma}^{k-1})^{2}-(c_{\gamma}^{k})^{2})||\gamma_{s}^{k}||^{2} \\
& + \frac{c_{{\boldsymbol{\lambda}}}^{k-1}-c_{{\boldsymbol{\lambda}}}^k}{2}Mw_{{\boldsymbol{\lambda}}}^2
+ \frac{c_{{\boldsymbol{\alpha}}}^{k-1}-c_{{\boldsymbol{\alpha}}}^k}{2}Mw_{{\boldsymbol{\alpha}}}^2
\\ 
& + \frac{c_{{\boldsymbol{\beta}}}^{k-1}-c_{{\boldsymbol{\beta}}}^k}{2}Mw_{{\boldsymbol{\beta}}}^2
+ \frac{c_{\gamma}^{k-1}-c_{\gamma}^k}{2}Pw_{\gamma}^2
\\ 
&  + \frac{4}{\eta_{{\boldsymbol{\lambda}}}}(\frac{c_{{\boldsymbol{\lambda}}}^{k-2}}{c_{{\boldsymbol{\lambda}}}^{k-1}}-\frac{c_{{\boldsymbol{\lambda}}}^{k-1}}{c_{{\boldsymbol{\lambda}}}^{k}})\sum_{l=1}^{L}||{\boldsymbol{\lambda}}_{l}^{k}||^{2} + \frac{4}{\eta_{{\boldsymbol{\alpha}}}}(\frac{c_{{\boldsymbol{\alpha}}}^{k-2}}{c_{{\boldsymbol{\alpha}}}^{k-1}}-\frac{c_{{\boldsymbol{\alpha}}}^{k-1}}{c_{{\boldsymbol{\alpha}}}^{k}})\sum_{l=1}^{L}||{\boldsymbol{\alpha}}_{l}^{k}||^{2}
 \\
 & + \frac{4}{\eta_{{\boldsymbol{\beta}}}}(\frac{c_{{\boldsymbol{\beta}}}^{k-2}}{c_{{\boldsymbol{\beta}}}^{k-1}}-\frac{c_{{\boldsymbol{\beta}}}^{k-1}}{c_{{\boldsymbol{\beta}}}^{k}})\sum_{l=1}^{L}||{\boldsymbol{\beta}}_{l}^{k}||^{2}  + \frac{4}{\eta_{\gamma}}(\frac{c_{\gamma}^{k-2}}{c_{\gamma}^{k-1}}-\frac{c_{\gamma}^{k-1}}{c_{\gamma}^{k}})\sum_{s=1}^{|\mathcal{P}^k|}||\gamma_{s}^{k}||^{2} 
.
\end{aligned}
\end{equation}
Denoting $\widetilde{K}(\epsilon)$ as $\widetilde{K}(\epsilon)=\min\{k \mid||\nabla\widetilde{G}^{K_1+k}||^2\le\frac{\epsilon}{4},k\ge2\}$.
Summing up (\ref{add_3}) from $K_1 + 2$
to $K_1 + \widetilde{K}(\epsilon)$, we can obtain that,
\begin{equation}
\label{add_4}
\begin{aligned}
\sum_{k=K_{1}+2}^{K_{1}+\widetilde{K}(\epsilon)}d_{6}^{k}||\nabla\widetilde{G}^{k}||^{2} 
& \leq  F^{K_1+2} - F^{K_1 + \widetilde{K}(\epsilon) +1} 
 + \frac{c_{{\boldsymbol{\lambda}}}^1}{2}Mw_{{\boldsymbol{\lambda}}}^2
+ \frac{c_{{\boldsymbol{\alpha}}}^1}{2}Mw_{{\boldsymbol{\alpha}}}^2
+ \frac{c_{{\boldsymbol{\beta}}}^1}{2}Mw_{{\boldsymbol{\beta}}}^2
+ \frac{c_{\gamma}^1}{2}Pw_{\gamma}^2
\\
& + \frac{4}{\eta_{\boldsymbol{\lambda}}}\frac{c_{{\boldsymbol{\lambda}}}^0}{c_{{\boldsymbol{\lambda}}}^1} M w_{{\boldsymbol{\lambda}}}^2 
+ \frac{4}{\eta_{\boldsymbol{\alpha}}}\frac{c_{{\boldsymbol{\alpha}}}^0}{c_{{\boldsymbol{\alpha}}}^1} M w_{{\boldsymbol{\alpha}}}^2
+ \frac{4}{\eta_{\boldsymbol{\beta}}}\frac{c_{{\boldsymbol{\beta}}}^0}{c_{{\boldsymbol{\beta}}}^1} M w_{{\boldsymbol{\beta}}}^2
+ \frac{4}{\eta_\gamma}\frac{c_{\mathbf{\gamma}}^0}{c_{\mathbf{\gamma}}^1} P w_{\mathbf{\gamma}}^2
\\
& + \overline{d_6}(\frac{1}{\eta_{\boldsymbol{\mu}}^2} + 18)  \sum_{k=K_{1}+2}^{K_{1}+\widetilde{K}(\epsilon)}\sum_{l=1}^{L}\| \boldsymbol{\mu}_l^{\bar{k}_l} -   \boldsymbol{\mu}_l^k \|^2  - \underline{d_6}(\frac{1}{\eta_{\boldsymbol{\mu}}^2} + 18) k_{\tau}\tau \sum_{k=K_{1}+2}^{K_{1}+\widetilde{K}(\epsilon)}\sum_{l=1}^{L}\| \boldsymbol{\mu}_l^{k+1} -   \boldsymbol{\mu}_l^k \|^2 \\
& +  \overline{d_6}(\frac{1}{\eta_{\mathbf{r}}^2} + 24w_{\mathbf{r}}^2)\sum_{k=K_{1}+2}^{K_{1}+\widetilde{K}(\epsilon)}\sum_{l=1}^{L}\| \mathbf{r}_l^{\bar{k}_l} -   \mathbf{r}_l^k \|^2 -\underline{d_6}(\frac{1}{\eta_{\mathbf{r}}^2} + 24w_{\mathbf{r}}^2)k_{\tau}\tau\sum_{k=K_{1}+2}^{K_{1}+\widetilde{K}(\epsilon)}\sum_{l=1}^{L}\| \mathbf{r}_l^{k+1} -   \mathbf{r}_l^k \|^2 \\
& +  \overline{d_6}(\frac{1}{\eta_{\mathbf{p}}^2}+24w_{\mathbf{p}}^2 )\sum_{k=K_{1}+2}^{K_{1}+\widetilde{K}(\epsilon)}\sum_{l=1}^{L}\| \mathbf{p}_l^{\bar{k}_l} -   \mathbf{p}_l^k \|^2 -\underline{d_6}(\frac{1}{\eta_{\mathbf{p}}^2}+24w_{\mathbf{p}}^2 )k_{\tau}\tau\sum_{k=K_{1}+2}^{K_{1}+\widetilde{K}(\epsilon)}\sum_{l=1}^{L}\| \mathbf{p}_l^{k+1} -   \mathbf{p}_l^k \|^2 \\
& + \frac{1}{10\tau\eta_{{\boldsymbol{\lambda}}}}\sum_{k=K_{1}+2}^{K_{1}+\widetilde{K}(\epsilon)}\sum_{l=1}^{L} ||{\boldsymbol{\lambda}}_{l}^{\bar{k}_l}-{\boldsymbol{\lambda}}_{l}^{k}||^{2} -\frac{1}{10\eta_{{\boldsymbol{\lambda}}}}\sum_{k=K_{1}+2}^{K_{1}+\widetilde{K}(\epsilon)}\sum_{l=1}^{L} ||{\boldsymbol{\lambda}}_{l}^{k+1}-{\boldsymbol{\lambda}}_{l}^{k}||^{2}
 \\
& + \frac{1}{10\tau\eta_{{\boldsymbol{\alpha}}}}\sum_{k=K_{1}+2}^{K_{1}+\widetilde{K}(\epsilon)}\sum_{l=1}^L||{\boldsymbol{\alpha}}_{l}^{\bar{k}_l}-{\boldsymbol{\alpha}}_{l}^{k}||^{2}
-\frac{1}{10\eta_{{\boldsymbol{\alpha}}}}\sum_{k=K_{1}+2}^{K_{1}+\widetilde{K}(\epsilon)}\sum_{l=1}^L||{\boldsymbol{\alpha}}_{l}^{k+1}-{\boldsymbol{\alpha}}_{l}^{k}||^{2}
 \\
& +\frac{1}{10\tau\eta_{{\boldsymbol{\beta}}}}\sum_{k=K_{1}+2}^{K_{1}+\widetilde{K}(\epsilon)}\sum_{l=1}^L||{\boldsymbol{\beta}}_{l}^{\bar{k}_l}-{\boldsymbol{\beta}}_{l}^{k}||^{2}
- \frac{1}{10\eta_{{\boldsymbol{\beta}}}}\sum_{k=K_{1}+2}^{K_{1}+\widetilde{K}(\epsilon)}\sum_{l=1}^L||{\boldsymbol{\beta}}_{l}^{k+1}-{\boldsymbol{\beta}}_{l}^{k}||^{2}
 \\
&+3\overline{d_6} \sum_{k=K_{1}+2}^{K_{1}+\widetilde{K}(\epsilon)}\sum_{l=1}^{L}((c_{{\boldsymbol{\lambda}}}^{\hat{k}_l-1})^{2}-(c_{{\boldsymbol{\lambda}}}^{\bar{k}_l-1})^{2})||{\boldsymbol{\lambda}}_{l}^{k}||^{2}+3\overline{d_6}\sum_{k=K_{1}+2}^{K_{1}+\widetilde{K}(\epsilon)}\sum_{l=1}^L((c_{{\boldsymbol{\alpha}}}^{\hat{k}_l-1})^{2}-(c_{{\boldsymbol{\alpha}}}^{\bar{k}_l-1})^{2})||{\boldsymbol{\alpha}}_{l}^{k}||^{2} \\& +3\overline{d_6}\sum_{k=K_{1}+2}^{K_{1}+\widetilde{K}(\epsilon)}\sum_{l=1}^L((c_{{\boldsymbol{\beta}}}^{\hat{k}_l-1})^{2}-(c_{{\boldsymbol{\beta}}}^{\bar{k}_l-1})^{2})||{\boldsymbol{\beta}}_{l}^{k}||^{2}  +  3\overline{d_6}{(c_{\gamma}^{1})}^2 Pw_{\gamma}^2.\\
\end{aligned}
\end{equation}
For each worker $l$, we have that $\overline{k_l} - \hat{k}_l \le \tau$, thus,
\begin{equation}
\label{add_5}
\begin{aligned}
& 3\overline{d_6} \sum_{k=K_{1}+2}^{K_{1}+\widetilde{K}(\epsilon)}\sum_{l=1}^{L}((c_{{\boldsymbol{\lambda}}}^{\hat{k}_l-1})^{2}-(c_{{\boldsymbol{\lambda}}}^{\bar{k}_l-1})^{2})||{\boldsymbol{\lambda}}_{l}^{k}||^{2} \leq 3\tau \overline{d_6}({c_{{\boldsymbol{\lambda}}}^1})^2Mw_{{\boldsymbol{\lambda}}}^2,
\\
& 3\overline{d_6} \sum_{k=K_{1}+2}^{K_{1}+\widetilde{K}(\epsilon)}\sum_{l=1}^{L}((c_{{\boldsymbol{\alpha}}}^{\hat{k}_l-1})^{2}-(c_{{\boldsymbol{\alpha}}}^{\bar{k}_l-1})^{2})||{\boldsymbol{\alpha}}_{l}^{k}||^{2} \leq 3\tau \overline{d_6}({c_{{\boldsymbol{\alpha}}}^1})^2Mw_{{\boldsymbol{\alpha}}}^2,
\\
& 3\overline{d_6}\sum_{k=K_{1}+2}^{K_{1}+\widetilde{K}(\epsilon)}\sum_{l=1}^{L}((c_{{\boldsymbol{\beta}}}^{\hat{k}_l-1})^{2}-(c_{{\boldsymbol{\beta}}}^{\bar{k}_l-1})^{2})||{\boldsymbol{\beta}}_{l}^{k}||^{2} \leq 3\tau \overline{d_6}({c_{{\boldsymbol{\beta}}}^1})^2Mw_{{\boldsymbol{\beta}}}^2.
\end{aligned}
\end{equation}
In our asynchronous algorithm, inactive workers do not update their variables in each master iteration, Thus, for any $k$ which satisfies $\hat{v}_l(j-1)\leq k<\hat{v}_l(j)$, we have 
$\boldsymbol{\mu}_{l}^{k}=\boldsymbol{\mu}_{l}^{\hat{v}_{l}(j)-1}$.  And for $k \notin \mathcal{V}_l(K)$, we have $\| \boldsymbol{\mu}_{l}^k - \boldsymbol{\mu}_{l}^{k-1}\|^2 =0$.
Since $\hat{v}_l(j)-\hat{v}_l(j-1)\leq\tau$, we can obtain
\begin{equation}
\label{add_6}
\sum_{k= K_{1}+2}^{K_{1}+\widetilde{K}(\epsilon)}\sum_{l=1}^{L}\| \boldsymbol{\mu}_l^{\bar{k}_l} -   \boldsymbol{\mu}_l^k \|^2 \leq
\tau \sum_{k=K_{1}+2}^{K_{1}+\widetilde{K}(\epsilon)}\sum_{l=1}^{L}||\boldsymbol{\mu}_{l}^{k+1}-\boldsymbol{\mu}_{l}^{k}||^{2}+4\tau(\tau-1)M w_{\boldsymbol{\mu}}^2.
\end{equation}
Similarly, we can obtain
\begin{equation}
\label{add_7}
\begin{aligned}
& \sum_{k= K_{1}+2}^{K_{1}+\widetilde{K}(\epsilon)}\sum_{l=1}^{L}\| \mathbf{r}_l^{\bar{k}_l} -   \mathbf{r}_l^k \|^2 \leq
\tau \sum_{k=K_{1}+2}^{K_{1}+\widetilde{K}(\epsilon)}\sum_{l=1}^{L}||\mathbf{r}_{l}^{k+1}-\mathbf{r}_{l}^{k}||^{2}+4\tau(\tau-1)M w_{\mathbf{r}}^2,
\\
& \sum_{k= K_{1}+2}^{K_{1}+\widetilde{K}(\epsilon)}\sum_{l=1}^{L}\| \mathbf{p}_l^{\bar{k}_l} -   \mathbf{p}_l^k \|^2 \leq
\tau \sum_{k=K_{1}+2}^{K_{1}+\widetilde{K}(\epsilon)}\sum_{l=1}^{L}||\mathbf{p}_{l}^{k+1}-\mathbf{p}_{l}^{k}||^{2}+4\tau(\tau-1)M w_{\mathbf{p}}^2,
\\
& \sum_{k= K_{1}+2}^{K_{1}+\widetilde{K}(\epsilon)}\sum_{l=1}^{L}\| {\boldsymbol{\lambda}}_l^{\bar{k}_l} -   {\boldsymbol{\lambda}}_l^k \|^2 \leq
\tau \sum_{k=K_{1}+2}^{K_{1}+\widetilde{K}(\epsilon)}\sum_{l=1}^{L}||{\boldsymbol{\lambda}}_{l}^{k+1}-{\boldsymbol{\lambda}}_{l}^{k}||^{2}+4\tau(\tau-1)M w_{{\boldsymbol{\lambda}}}^2,
\\
& \sum_{k= K_{1}+2}^{K_{1}+\widetilde{K}(\epsilon)}\sum_{l=1}^{L}\| {\boldsymbol{\alpha}}_l^{\bar{k}_l} -   {\boldsymbol{\alpha}}_l^k \|^2 \leq
\tau \sum_{k=K_{1}+2}^{K_{1}+\widetilde{K}(\epsilon)}\sum_{l=1}^{L}||{\boldsymbol{\alpha}}_{l}^{k+1}-{\boldsymbol{\alpha}}_{l}^{k}||^{2}+4\tau(\tau-1)M w_{{\boldsymbol{\alpha}}}^2,
\\
& \sum_{k= K_{1}+2}^{K_{1}+\widetilde{K}(\epsilon)}\sum_{l=1}^{L}\| {\boldsymbol{\beta}}_l^{\bar{k}_l} -   {\boldsymbol{\beta}}_l^k \|^2 \leq
\tau \sum_{k=K_{1}+2}^{K_{1}+\widetilde{K}(\epsilon)}\sum_{l=1}^{L}||{\boldsymbol{\beta}}_{l}^{k+1}-{\boldsymbol{\beta}}_{l}^{k}||^{2}+4\tau(\tau-1)M w_{{\boldsymbol{\beta}}}^2.
\end{aligned}
\end{equation}
We set the value of $k_{\tau}$ to satisfy that,
\begin{equation}
\label{k_tau}
k_{\tau}\geq\operatorname*{max}\left\{\frac{\overline{{d_{6}}}(\frac{1}{\eta_{\boldsymbol{\mu}}^2} + 18)}{\underline{{d_{6}}}(\frac{1}{\overline{\eta_{\boldsymbol{\mu}}^2}} + 18)},
\frac{\overline{{d_{6}}}( \frac{1}{\eta_{\mathbf{r}}^2} + 24w_{\mathbf{r}}^2)}{\underline{{d_{6}}}(\frac{1}{\overline{\eta_{\mathbf{r}}^2}} + 24w_{\mathbf{r}}^2)},
\frac{\overline{{d_{6}}}( \frac{1}{\eta_{\mathbf{p}}^2} + 24w_{\mathbf{p}}^2)}{\underline{{d_{6}}}(\frac{1}{\overline{\eta_{\mathbf{p}}^2}} + 24w_{\mathbf{p}}^2)}\right\},
\end{equation}
where $\overline{\eta_{\boldsymbol{\mu}}}$, $\overline{\eta_{\mathbf{r}}}$, and $\overline{\eta_{\mathbf{p}}}$  are the upper bounds of $\eta_{\boldsymbol{\mu}}^k$, $\eta_{\mathbf{r}}^k$, and $\eta_{\mathbf{p}}^k$, respectively. 
By employing (\ref{add_4}), (\ref{add_5}), (\ref{add_6})   (\ref{add_7}) and (\ref{k_tau}), we can obtain
\begin{equation}
\begin{aligned}
\sum_{k=K_{1}+2}^{K_{1}+\widetilde{K}(\epsilon)}d_{6}^{k}||\nabla\widetilde{G}^{k}||^{2} 
& \leq  F^{K_1+2} - F^{K_1 + \widetilde{K}(\epsilon) +1} \\[-2mm]
&  + \frac{c_{{\boldsymbol{\lambda}}}^1}{2}Mw_{{\boldsymbol{\lambda}}}^2
+ \frac{c_{{\boldsymbol{\alpha}}}^1}{2}Mw_{{\boldsymbol{\alpha}}}^2
+ \frac{c_{{\boldsymbol{\beta}}}^1}{2}Mw_{{\boldsymbol{\beta}}}^2
+ \frac{c_{\gamma}^1}{2}Pw_{\gamma}^2
\\
& + \frac{4}{\eta_{\boldsymbol{\lambda}}}\frac{c_{{\boldsymbol{\lambda}}}^0}{c_{{\boldsymbol{\lambda}}}^1} M w_{{\boldsymbol{\lambda}}}^2 
+ \frac{4}{\eta_{\boldsymbol{\alpha}}}\frac{c_{{\boldsymbol{\alpha}}}^0}{c_{{\boldsymbol{\alpha}}}^1} M w_{{\boldsymbol{\alpha}}}^2
+ \frac{4}{\eta_{\boldsymbol{\beta}}}\frac{c_{{\boldsymbol{\beta}}}^0}{c_{{\boldsymbol{\beta}}}^1} M w_{{\boldsymbol{\beta}}}^2
+ \frac{4}{\eta_\gamma}\frac{c_{\mathbf{\gamma}}^0}{c_{\mathbf{\gamma}}^1} P w_{\mathbf{\gamma}}^2
 \\
& + 3\tau \overline{d_6}({c_{{\boldsymbol{\lambda}}}^1})^2Mw_{{\boldsymbol{\lambda}}}^2 +  
3\tau \overline{d_6}({c_{{\boldsymbol{\alpha}}}^1})^2Mw_{{\boldsymbol{\alpha}}}^2 +
3\tau \overline{d_6}({c_{{\boldsymbol{\beta}}}^1})^2Mw_{{\boldsymbol{\beta}}}^2 +  3\overline{d_6}({c_{\gamma}^{1}})^2 Pw_{\gamma}^2\\
& + \bigg(\frac{2Mw_{{\boldsymbol{\lambda}}}^2}{5\eta_{{\boldsymbol{\lambda}}}} + \frac{2Mw_{{\boldsymbol{\alpha}}}^2}{5\eta_{{\boldsymbol{\alpha}}}} + \frac{2Mw_{{\boldsymbol{\beta}}}^2}{5\eta_{{\boldsymbol{\beta}}}}
+ 4\overline{d_6}(\frac{1}{\eta_{\boldsymbol{\mu}}^2} + 18)Mw_{\boldsymbol{\mu}}^2\tau 
\\
& + 4\overline{d_6}(\frac{1}{\eta_{\mathbf{r}}^2} + 24w^2_{\mathbf{r}})Mw_{\mathbf{r}}^2\tau
+ 4\overline{d_6}(\frac{1}{\eta_{\mathbf{p}}^2} + 24w^2_{\mathbf{p}})Mw_{\mathbf{p}}^2\tau
\bigg)  (\tau -1)
\\& = (d_7 +  k_d\tau)(\tau -1).
\end{aligned}
\end{equation}
where $d_7$ and $k_d$ are constants.
Set constant $d_8$ as 
\begin{equation}
d_8 = \left( 
\operatorname*{max}\left\{d_{1},d_{2},d_{3},d_{4},d_{5},
\frac{30\tau}{\eta_{{\boldsymbol{\lambda}}}\underline{a_6}},\frac{30\tau}{\eta_{{\boldsymbol{\alpha}}}\underline{a_6}},\frac{30\tau}{\eta_{{\boldsymbol{\beta}}}\underline{a_6}},\frac{30}{\eta_{{\gamma}}\underline{a_6}}\right\}
\right)
 \geq \frac{1}{d_6^k a_6^k} .
\end{equation}
Thus, we can obtain that 
\begin{equation}
\sum_{k=K_1+2}^{K_1+\tilde{K}(\epsilon)}\frac{1}{d_8a_6^k}\|\nabla\widetilde{G}^{K_1+\tilde{K}(\epsilon)}\|^2\le\sum_{k=K_1+2}^{K_1+\tilde{K}(\epsilon)}\frac{1}{d_8a_6^k}\|\nabla\widetilde{G}^k\|^2\le\sum_{k=K_1+2}^{K_1+\tilde{K}(\epsilon)}d_6^k\|\nabla\widetilde{G}^k\|^2\leq (d_7+ k_d\tau)(\tau-1).
\end{equation}
We can further obtain
\begin{equation}
\label{add_9}
||\nabla\widetilde{G}^{K_1+\widetilde{K}(\epsilon)}||^2\le\frac{(d_7 + k_d\tau)(\tau-1)d_8}{\sum\limits_{k=K_1+2}^{K_1+\tilde{K}} \frac{1}{a_6^k}}.
\end{equation}
According to the setting of $c_{{\boldsymbol{\lambda}}}^k$, $c_{{\boldsymbol{\alpha}}}^k$, $c_{{\boldsymbol{\beta}}}^k$ and $c_{\gamma}^k$, we have,
\begin{equation}
\frac{1}{a_6^k}\geq \frac{1} {8({\xi-2})(k+1)^{\frac{1}{2}}  \min\{ \eta_{{\boldsymbol{\lambda}}}+ \eta_{{\boldsymbol{\alpha}}} + \eta_{{\boldsymbol{\beta}}}, 4w^2_{\mathbf{r}}\eta_{{\boldsymbol{\lambda}}},  4w^2_{\mathbf{p}}\eta_{{\boldsymbol{\alpha}}}, 4w^2_{\mathbf{q}}\eta_{\gamma},\eta_{{\boldsymbol\beta}} + \eta_{{\gamma}}\sum_{s=1}^{|\mathcal{P}^k|}\|b_{s}\|^2 \}}.
\end{equation}
Summing up $a_6^k$ from $k = K_1 +2 $  to $k = K_1+\widetilde{K} $, we have
\begin{equation}
\sum_{k=K_1+2}^{K_1+\tilde{K}(\epsilon)}\frac{1}{a_6^k} \geq \frac{(K_1+\widetilde{K}(\epsilon))^{\frac{1}{2}} -(K_1+2)^{\frac{1}{2}}} {
8({\xi-2})\min\{ \eta_{{\boldsymbol{\lambda}}}+ \eta_{{\boldsymbol{\alpha}}} + \eta_{{\boldsymbol{\beta}}}, 4w^2_{\mathbf{r}}\eta_{{\boldsymbol{\lambda}}},  4w^2_{\mathbf{p}}\eta_{{\boldsymbol{\alpha}}}, 4w^2_{\mathbf{q}}\eta_{\gamma},\eta_{{\boldsymbol\beta}} + \eta_{{\gamma}}\sum_{s=1}^{|\mathcal{P}^k|}\|b_{s}\|^2 \} }.
\label{add_8}
\end{equation}

Recall that $\widetilde{K}(\epsilon)=\min\{k \mid||\nabla\widetilde{G}^{K_1+k}||^2\le\frac{\epsilon}{4},k\ge2\}$. Therefore, by employing (\ref{add_8}) and (\ref{add_9}), when 
\begin{equation}
\label{add_10}
K_1+\widetilde{K}(\epsilon)\geq(\frac{  d_9(d_7+ k_d\tau)(\tau-1)d_8}{\epsilon}+(K_1+2)^{\frac{1}{2}})^2,
\end{equation}
the value of $||\nabla\widetilde{G}^{K_1+\widetilde{K}(\epsilon)}||^2$ can be guaranteed to be smaller than $\frac{\epsilon}{4}$, where 
\begin{equation}
d_9 = 32({\xi-2}) \min\{ \eta_{{\boldsymbol{\lambda}}}+ \eta_{{\boldsymbol{\alpha}}} + \eta_{{\boldsymbol{\beta}}}, 4w^2_{\mathbf{r}}\eta_{{\boldsymbol{\lambda}}},  4w^2_{\mathbf{p}}\eta_{{\boldsymbol{\alpha}}}, 4w^2_{\mathbf{q}}\eta_{\gamma},\eta_{{\gamma}} + \eta_{{\gamma}}\sum_{s=1}^{L}\|b_{s}\|^2 \}.
\end{equation}
Combining the definition of $\nabla{G}^{k}$ and $\nabla\widetilde{G}^{k}$ with trigonometric inequality, we then get:
\begin{equation}
||\nabla G^k||-||\nabla\widetilde{G}^k||\leq||\nabla G^k-\nabla\widetilde{G}^k||\leq\sqrt{\sum_{l=1}^{L}||c_{{\boldsymbol{\lambda}}}^{k-1}{\boldsymbol{\lambda}}_l^k||^2+\sum_{l=1}^L||c_{{\boldsymbol{\alpha}}}^{k-1}{\boldsymbol{\alpha}}_l^k||^2 + +\sum_{l=1}^L||c_{{\boldsymbol{\beta}}}^{k-1}{\boldsymbol{\beta}}_l^k||^2+ +\sum_{s=1}^{|\mathcal{P}^k|}||c_{\gamma}^{k-1}\gamma_s^k||^2}.
\end{equation}
If $k \geq 16(\frac{Mw_{{\boldsymbol{\lambda}}}^2}{\eta_\lambda^2} + \frac{Mw_{{\boldsymbol{\alpha}}}^2}{\eta_{\boldsymbol{\alpha}}^2} +
\frac{Mw_{{\boldsymbol{\beta}}}^2}{\eta_{\boldsymbol{\beta}}^2} + 
\frac{Pw_{\gamma}^2}{\eta_\gamma^2})^2 \frac{1}{\epsilon^2}$, we have 
\begin{equation}
\sqrt{\sum_{l=1}^{L}||c_{{\boldsymbol{\lambda}}}^{k-1}{\boldsymbol{\lambda}}_l^k||^2+\sum_{l=1}^L||c_{{\boldsymbol{\alpha}}}^{k-1}{\boldsymbol{\alpha}}_l^k||^2 + +\sum_{l=1}^L||c_{{\boldsymbol{\beta}}}^{k-1}{\boldsymbol{\beta}}_l^k||^2+ +\sum_{s=1}^{|\mathcal{P}^k|}|||c_{\gamma}^{k-1}\gamma_s^k||^2} \leq \frac{\sqrt{\epsilon}}{2}.
\end{equation}
Combining it
with (\ref{add_10}), we can conclude that 
\begin{equation}
K(\epsilon) \sim \mathcal{O}\left(\max ~\left\{~(16(\frac{Mw_{{\boldsymbol{\lambda}}}^2}{\eta_\lambda^2} + \frac{Mw_{{\boldsymbol{\alpha}}}^2}{\eta_{\boldsymbol{\alpha}}^2} +
\frac{Mw_{{\boldsymbol{\beta}}}^2}{\eta_{\boldsymbol{\beta}}^2} + 
\frac{Pw_{\gamma}^2}{\eta_\gamma^2})^2 \frac{1}{\epsilon^2},
(\frac{  d_9(d_7 + k_d\tau)(\tau-1)d_8}{\epsilon}+(K_1+2)^{\frac{1}{2}})^2
\right\}~ 
\right).
\end{equation}
which concludes our proof.
\end{proof}

\section{Proof of Theorem \ref{theorem_1}}
\label{AD_C}
In Algorithm \ref{algorithm_1}, cutting plane set are updated every $w$ iteration, 
i.e.,
\begin{equation}
\mathcal{P}^0\supseteq\mathcal{P}^w\supseteq\cdots\supseteq\mathcal{P}^{nw}.
\end{equation}
The feasible region of problem (\ref{distributed_3}) in the $w^{th}$ iteration is represented as $\mathcal{R}^w$, while the feasible region of problem (\ref{distributed_2}) is represented as  $\mathcal{R}^{*}$. As such, we have
\begin{equation}
\mathcal{R}^0\supseteq\mathcal{R}^w\supseteq\cdots\supseteq\mathcal{R}^{nw} \supseteq\mathcal{R}^{*}.
\label{c_1}
\end{equation}
Let $F^{w*}$ represent the optimal objective value of the problem (\ref{distributed_3}) in the $w^{th}$ iteration and let $F^{*} \leq 0 $ represent the optimal objective value of the problem in (\ref{distributed_2}). Based on (\ref{c_1}), we can obtain
\begin{equation}
F^{0*} \leq F^{w*} \leq \cdots \leq F^{nw*} \leq F^*.
\end{equation}
Thus, we have 
\begin{equation}
\frac{F^*}{F^{0*}} \leq \frac{F^*}{F^{w*}} \leq \cdots \leq \frac{F^*}{F^{nw*}} \leq \Omega,
\label{c_2}
\end{equation}
where $\Omega \leq 1$.
It can be seen from (\ref{c_2}) that as the number of cutting planes increases, the sequence is monotonically non-decreasing.
When $nw \to \infty$, $\frac{F^*}{F^{nw*}}$ monotonically converges to $\Omega$.

\end{document}

%% file: Figures/ICML_1.tex
%
%
\definecolor{mycolor1}{rgb}{0.00000,0.44700,0.74100}%
\definecolor{mycolor2}{rgb}{0.85000,0.32500,0.09800}%
\definecolor{mycolor3}{rgb}{0.92900,0.69400,0.12500}%
\definecolor{mycolor4}{rgb}{0.49400,0.18400,0.55600}%
\begin{tikzpicture}

\begin{axis}[%
width=2.286in,
height=1.8in,
at={(0.758in,0.481in)},
scale only axis,
xmin=0,
xmax=5,
tick label style={font=\tiny},
xtick={0,1,2,3,4,5},
xticklabels={{0},{1000},{2000},{4000},{8000},{16000}},
xlabel style={font=\color{white!15!black}},
xlabel style={font=\small},  
xlabel={False Alarm Period},
ymin=0,
ymax=16000,
ytick={   0, 4000, 8000, 12000, 16000},
ylabel style={font=\color{white!15!black}},
ylabel style={font=\small,yshift=-1ex},
ylabel={Average Detection Delay},
axis background/.style={fill=white},
xmajorgrids,
ymajorgrids,
grid style={dotted, black, opacity=1},
legend style={at={(0.03,0.97)}, anchor=north west, legend cell align=left, align=left, draw=white!15!black}
]
\addplot [color=mycolor1, line width=1.0pt, mark=o, mark options={solid, mycolor1}]
  table[row sep=crcr]{%
0	0\\
1	829.4\\
2	1651.3\\
3	3314.5\\
4	6627.6\\
5	13259.7\\
};
\addlegendentry{\small RGCUSUM}

\addplot [color=mycolor2, line width=1.0pt, mark=x, mark options={solid, mycolor2}]
  table[row sep=crcr]{%
0	0\\
1	891.3\\
2	1785.1\\
3	3560.1\\
4	7129.2\\
5	14256.5\\
};
\addlegendentry{\small WCD}

\addplot [color=mycolor3, line width=1.0pt, mark=+, mark options={solid, mycolor3}]
  table[row sep=crcr]{%
0	0\\
1	976.2\\
2	1955.4\\
3	3914.1\\
4	7802.7\\
5	15619.2\\
};
\addlegendentry{\small Adaptive CUSUM}

\addplot [color=mycolor4, line width=1.0pt, mark=triangle, mark options={solid, mycolor4}]
  table[row sep=crcr]{%
0	0\\
1	672.2\\
2	1346.3\\
3	2681.4\\
4	5375.2\\
5	10745.6\\
};
\addlegendentry{\small Triadic-OCD}

\end{axis}
\end{tikzpicture}%

%% file: Figures/ICML_2.tex
%
%
\definecolor{mycolor1}{rgb}{0.00000,0.44700,0.74100}%
\definecolor{mycolor2}{rgb}{0.85000,0.32500,0.09800}%
\definecolor{mycolor3}{rgb}{0.92900,0.69400,0.12500}%
\begin{tikzpicture}

\begin{axis}[%
width=2.286in,
height=1.8in,
at={(0.758in,0.481in)},
scale only axis,
xmin=0,
xmax=5,
tick label style={font=\tiny},
xtick={0,1,2,3,4,5},
xticklabels={{0},{1000},{2000},{4000},{8000},{16000}},
xlabel style={font=\color{white!15!black}},
xlabel style={font=\small},  
xlabel={False Alarm Period},
ymin=0,
ymax=12000,
ytick={   0, 2500, 5000, 7500, 10000},
ylabel style={font=\color{white!15!black}},
ylabel style={font=\small,yshift=-1ex},
ylabel={Average Detection Delay},
axis background/.style={fill=white},
xmajorgrids,
ymajorgrids,
grid style={dotted, black, opacity=1},
legend style={at={(0.03,0.97)}, anchor=north west, legend cell align=left, align=left, draw=white!15!black}
]
\addplot [color=mycolor1, line width=1.0pt, mark=o, mark options={solid, mycolor1}]
  table[row sep=crcr]{%
0	0\\
1	672.2\\
2	1346.3\\
3	2681.4\\
4	5375.2\\
5	10745.6\\
};
\addlegendentry{Polytope}

\addplot [color=mycolor2, line width=1.0pt, mark=x, mark options={solid, mycolor2}]
  table[row sep=crcr]{%
0	0\\
1	696.2\\
2	1392.4\\
3	2784.8\\
4	5569.6\\
5   11139.2\\
};
\addlegendentry{D-norm}

\addplot [color=mycolor3, line width=1.0pt, mark=triangle, mark options={solid, mycolor3}]
  table[row sep=crcr]{%
0	0\\
1	712.2\\
2	1424.4\\
3	2848.8\\
4	5697.6\\
5	11395.2\\
};
\addlegendentry{Ellipsoid}

\end{axis}
\end{tikzpicture}%

%% file: Figures/robust_valid.tex
%
%
\definecolor{mycolor1}{rgb}{0.06270,0.50190,0.73330}%
\definecolor{mycolor2}{rgb}{0.91760,0.50196,0.10980}%
\definecolor{mycolor3}{rgb}{0.72160,0.72160,0.72160}%
\begin{tikzpicture}

\begin{axis}[%
width=2.286in,
height=1.8in,
at={(0.758in,0.481in)},
scale only axis,
bar shift auto,
xmin=0.516666666666667,
xmax=10.4833333333333,
tick label style={font=\tiny},
xtick={1,2,3,4,5,6,7,8,9,10},
xticklabels={{310},{320},{330},{340},{350},{360},{370},{380},{390},{400}},
xlabel style={font=\color{white!15!black}},
xlabel={Upper Bound on the Detection Delay},
ymin=0,
ymax=1.5,
ylabel style={font=\color{white!15!black}},
ylabel={Successful Detection Rate},
axis background/.style={fill=white},
xmajorgrids,
ymajorgrids,
legend style={legend cell align=left, align=left, draw=white!15!black}
]
\addplot[ybar, bar width=0.167, fill=mycolor1, draw=black, area legend] table[row sep=crcr] {%
1	0.439\\
2	0.549\\
3	0.644\\
4	0.736\\
5	0.812\\
6	0.87\\
7	0.914\\
8	0.936\\
9	0.961\\
10	0.974\\
};
\addplot[forget plot, color=white!15!black] table[row sep=crcr] {%
0.516666666666667	0\\
10.4833333333333	0\\
};
\addlegendentry{ \small RGCUSUM}

\addplot[ybar, bar width=0.167, fill=mycolor2, draw=black, area legend] table[row sep=crcr] {%
1	0.852\\
2	0.921\\
3	0.964\\
4	0.979\\
5	0.99\\
6	0.993\\
7	0.997\\
8	1\\
9	1\\
10	1\\
};
\addplot[forget plot, color=white!15!black] table[row sep=crcr] {%
0.516666666666667	0\\
10.4833333333333	0\\
};
\addlegendentry{\small Triadic-OCD}

\addplot[ybar, bar width=0.167, fill=mycolor3, draw=black, area legend] table[row sep=crcr] {%
1	0.439\\
2	0.549\\
3	0.644\\
4	0.736\\
5	0.812\\
6	0.87\\
7	0.914\\
8	0.936\\
9	0.961\\
10	0.974\\
};
\addplot[forget plot, color=white!15!black] table[row sep=crcr] {%
0.516666666666667	0\\
10.4833333333333	0\\
};
\addlegendentry{\small WCD}

\end{axis}
\end{tikzpicture}%

%% file: Figures/asy_sy_3.tex
%
%
\definecolor{mycolor1}{rgb}{0.00000,0.44700,0.74100}%
\definecolor{mycolor2}{rgb}{0.85000,0.32500,0.09800}%
\begin{tikzpicture}

\begin{axis}[%
width=2.286in,
height=1.8in,
at={(0.758in,0.481in)},
scale only axis,
tick label style={font=\tiny}, 
xmin=0,
xmax=1200,
xlabel style={font=\small},
xlabel={Time(s)},
ymin=-50,
ymax=300,
ylabel style={font=\small},
ylabel={Objective Function Value},
axis background/.style={fill=white},
xmajorgrids,
ymajorgrids,
grid style={dotted, black, opacity=1},
legend style={legend cell align=left, align=left, draw=white!15!black}
]
\addplot [color=mycolor1, dotted, line width=2.0pt]
  table[row sep=crcr]{%
1.8331738391229	286.616107393646\\
3.55860338414337	278.006101976776\\
5.12073905503848	274.726604695367\\
6.93208994222565	266.368931408117\\
8.67992535227286	263.192107813357\\
10.3202698438169	255.083266355201\\
11.3526055301932	252.007319019199\\
13.1266459046849	244.143675439905\\
14.9034783388186	241.166754363877\\
16.1596278812111	233.544576423928\\
17.6434235078848	230.664784933246\\
18.8139630209604	223.280249683358\\
20.716630919026	220.495635456359\\
22.3333563472983	213.344789171768\\
23.8633317450508	210.653408969197\\
25.2331296469709	203.732249491361\\
26.8384115108872	201.132085314839\\
28.4799843796739	194.436498393473\\
30.3399399608177	191.925522563333\\
31.7656200534876	185.451338418935\\
33.1260503908604	183.027530360156\\
34.7870807282921	176.770508758382\\
36.9484852024988	174.431804448145\\
38.6515671131537	168.387686104072\\
40.3992313394236	166.132030654014\\
41.277480341358	160.296458549043\\
43.2310817379018	158.121795986692\\
44.9905966441611	152.490429772675\\
46.1824716123639	150.394706719604\\
48.0564863103913	144.963164797165\\
49.08068745665	142.944339913979\\
51.1715495434157	137.708226894953\\
52.655694228072	135.764279280274\\
53.82977615307	130.719191637186\\
55.3552970450249	128.848117014674\\
57.1092986087336	123.989648562583\\
58.6841361379299	122.189461733171\\
60.0070770962829	117.513262020846\\
61.4624080132513	115.782020446311\\
62.8224982351687	111.283706748736\\
64.7515103571195	109.61949441677\\
66.308348929821	105.294762075424\\
68.2907114856801	103.695711090078\\
70.2484475956811	99.5402966193296\\
72.291499310487	98.0045809328897\\
73.709544089627	94.0142624217899\\
75.4274173488294	92.5401113989994\\
77.9592824898026	88.7107397275715\\
79.1766258735875	87.2964354287296\\
80.8011125152204	83.6239302727699\\
83.2683992556873	82.2678092861388\\
84.8795421379196	78.7481843629284\\
86.5456349137575	77.4486345788011\\
88.4106434101276	74.0779403551501\\
89.8159651480683	72.8334062445979\\
91.7339569118766	69.6077776448212\\
93.3318033778324	68.4167437222852\\
95.3716275802989	65.3323661773386\\
97.0760010578099	64.1933547620586\\
98.5477663705738	61.246467684367\\
99.9560342282414	60.1580367371573\\
101.820858514457	57.3449235180817\\
103.490276923855	56.3056571803475\\
105.016603365306	53.622642644591\\
106.888604759428	52.6311575920381\\
108.585003682533	50.074611476178\\
110.332301278583	49.1295480677966\\
111.822519017756	46.6958875381008\\
113.508762162182	45.7959157644314\\
114.659562918605	43.4816070433674\\
116.531196037158	42.6254214700648\\
117.780394111091	40.4269701217308\\
119.306689325239	39.6132886945861\\
121.28570005412	37.5272618299732\\
121.873111841241	36.754830657851\\
124.067268943255	34.7778348726299\\
125.889149495311	34.0454188607171\\
127.366727978075	32.1741149693276\\
128.656208136079	31.480505297976\\
129.96883123633	29.7116057000947\\
131.882726356726	29.0556148359249\\
133.593040011846	27.3858761421255\\
135.271664921971	26.7663361353564\\
137.08416628085	25.1925651531922\\
138.028896746422	24.6083284317979\\
139.725318784531	23.1273738370294\\
141.495483682783	22.5773111514957\\
142.978677707772	21.1860522282179\\
145.410147662921	20.669050713027\\
146.994642889183	19.3644029509071\\
148.058345898993	18.8793653521396\\
149.98948636461	17.658255664886\\
151.245504143132	17.2041003827173\\
152.613798862431	16.0634719965897\\
154.418967335683	15.639128880714\\
156.074432940959	14.5759257450577\\
157.447669188247	14.1803380330498\\
158.731469444191	13.1914905820418\\
160.295257581066	12.8236098807361\\
162.086744985354	11.9060242990757\\
164.009082952451	11.5648106335767\\
166.083667903598	10.7153616672491\\
167.631464895286	10.3997816325256\\
169.362765602401	9.61530062932417\\
171.497747606056	9.32432503334007\\
173.061409434681	8.60159483745252\\
174.718389244927	8.334197489161\\
176.326408592927	7.66995323549423\\
177.512135402122	7.42510806102774\\
178.836293547838	6.81603566758213\\
180.233888738109	6.59271771466251\\
182.04145569687	6.03546737718739\\
183.045731113711	5.83264916111995\\
185.208520996243	5.32384649128772\\
187.01317961715	5.14049800232334\\
188.479277547999	4.67676688811827\\
190.031882717987	4.51185595669291\\
191.161244932231	4.08984415458561\\
192.466130899589	3.94233598369711\\
193.825843003003	3.55874454969889\\
195.565854218758	3.42760359181585\\
197.190563577323	3.07921695030439\\
198.653730517781	2.96340805895427\\
199.858862315321	2.64712498128826\\
201.408979471963	2.54561694715245\\
203.027187084789	2.2584783674359\\
204.348183425288	2.17024737080756\\
206.345816591437	1.90946071932895\\
207.670436772147	1.83349489875464\\
209.238843625839	1.59645259176951\\
210.946841376246	1.53175743353643\\
212.365582369943	1.31604812020652\\
214.608072087767	1.26165258623179\\
216.328727842305	1.06507110872511\\
217.709032613759	1.0200335456166\\
219.457418400406	0.840577437317762\\
220.754192645195	0.803992139161451\\
222.120265452539	0.639860008815063\\
223.618004874726	0.610863316774982\\
225.759511856285	0.460445576983787\\
227.547809594338	0.438221479307984\\
228.943383526299	0.300088433219599\\
230.133938639448	0.283873503933435\\
231.974257966652	0.156762453205559\\
233.414719029525	0.145849495181125\\
235.440577828105	0.0286499248968681\\
237.28426757972	0.0223908139745532\\
239.322216868459	-0.0858708604799716\\
241.423860583071	-0.0880638433385315\\
243.299651577074	-0.188238142191139\\
244.767558319922	-0.186892150431035\\
245.909243981737	-0.279721291663904\\
247.685849267025	-0.275303973963446\\
249.069016729416	-0.361435227160143\\
250.477907744535	-0.354356851989834\\
252.517963214031	-0.434353971247532\\
254.346690913942	-0.42497042211674\\
255.889806347062	-0.499324700874578\\
257.74397797222	-0.487941238507428\\
258.789669737176	-0.557081473610506\\
260.324855224136	-0.543956930795471\\
262.190159844265	-0.608257819668292\\
263.595335400901	-0.593609095342671\\
265.528152592278	-0.653398848003789\\
267.710274705527	-0.637405495976381\\
268.895866102814	-0.692972199785842\\
269.927649886644	-0.67578110329384\\
272.193698153749	-0.727378834994337\\
273.956846086911	-0.709108496322724\\
275.553337780097	-0.756961736389339\\
276.868894584482	-0.737706518613432\\
278.867398651572	-0.78201468618283\\
280.713816915826	-0.761848718075836\\
282.649076319109	-0.802789973783074\\
283.940965518112	-0.78177052469881\\
285.394654731593	-0.819504461109783\\
286.631313603652	-0.797675138020821\\
288.030835366199	-0.832345868581353\\
288.853137265518	-0.809739491618124\\
290.832563595378	-0.841477689342575\\
292.189629538136	-0.818118661467757\\
294.024151629697	-0.847043706822568\\
295.96698220821	-0.822950147205314\\
297.114226462122	-0.849171787701063\\
298.969720904186	-0.824357343835746\\
300.719424796061	-0.847977113919742\\
302.714401216091	-0.822452552435489\\
304.263643514549	-0.843564992043293\\
305.757156364272	-0.81733959384152\\
307.419221148961	-0.83603339751448\\
308.665306802166	-0.809115961666542\\
310.55252412543	-0.825474926193134\\
312.809180666997	-0.797874794939261\\
314.41980200926	-0.811978502553302\\
315.897540717709	-0.783706290931715\\
317.466281833438	-0.795630892670314\\
318.583186277116	-0.766699179104449\\
320.35400415762	-0.776517937368761\\
322.230906162662	-0.746941745756916\\
323.646016213452	-0.754725589920699\\
325.730952638415	-0.724522866640593\\
327.190544376077	-0.730340765830575\\
329.22071522267	-0.699532621093994\\
330.568446577125	-0.703451989541758\\
331.935017244973	-0.672062997976308\\
333.740704583839	-0.674150025625153\\
335.218629638245	-0.642208462613075\\
337.203448563434	-0.642528357034782\\
338.822120083195	-0.610066234749039\\
340.598165015762	-0.608683468713304\\
341.987769628371	-0.575736646018458\\
343.457904047719	-0.572715137152772\\
345.263334838323	-0.539323309734913\\
347.146429750108	-0.534726584073956\\
348.363066364688	-0.500933275928353\\
349.988394485706	-0.494824580759032\\
352.047633256459	-0.460677098328595\\
354.319719998629	-0.453119521102501\\
356.614084885259	-0.418668821526108\\
358.142305418826	-0.40972533855107\\
360.443850937576	-0.375025918435711\\
361.895271183887	-0.364759465963679\\
363.464859344665	-0.329869171973601\\
365.691401048102	-0.318342681800462\\
366.630609807793	-0.283322554003006\\
368.546427397582	-0.270598951102958\\
370.83999220869	-0.235512966613338\\
372.324223979118	-0.221655169601347\\
374.496086399045	-0.186570024041464\\
375.789389386885	-0.171640934207984\\
377.821565572132	-0.136625763981312\\
379.891480123475	-0.120688243350048\\
380.790043766171	-0.0858143598329648\\
382.608378482061	-0.0689311560852723\\
384.016823986036	-0.0342717328886141\\
386.091236365118	-0.0165054566568583\\
387.39843667175	0.0178647948651671\\
389.452159464792	0.0364517492150578\\
390.639691960405	0.0704569084290452\\
392.593350168519	0.0898024400087048\\
394.234908694665	0.123365652443237\\
395.329645343963	0.143408014142307\\
397.133304661977	0.176451964251862\\
398.486552978928	0.197129873738579\\
400.691110835901	0.229577036411475\\
402.217012580308	0.250829761401084\\
403.82979048579	0.282602795996428\\
404.998736490519	0.304370246566253\\
407.292704622417	0.335392372840265\\
409.358127580391	0.357615209527439\\
410.923413376894	0.387810577786216\\
412.996450376002	0.41043032159205\\
414.560223155447	0.439724321242056\\
416.850976794461	0.462683437947599\\
418.172516004178	0.491003131800799\\
419.707883216205	0.514245163059446\\
421.73335242342	0.541519586779865\\
423.478501291306	0.56498923765794\\
424.906646377501	0.591149768623028\\
426.453560760755	0.614793044256907\\
428.507105084837	0.639773770199084\\
429.68108077481	0.663538049980305\\
431.398409546001	0.687275952352704\\
433.065494649781	0.711110087946278\\
434.3756838403	0.733545507725254\\
436.368349055867	0.757399960790331\\
438.184090867544	0.778476840947747\\
440.21869533161	0.802303761920116\\
442.149392914956	0.821969877595246\\
443.666197706048	0.845723179028755\\
445.448545055147	0.863930392016751\\
447.124891911125	0.887565860601399\\
449.107423957059	0.904270436391759\\
449.737974453458	0.92774580007211\\
451.936334969224	0.942908560968736\\
453.140873602974	0.966183573211343\\
454.691265700647	0.979770033481319\\
456.58039955064	1.00280650752731\\
458.257496826227	1.0147871431827\\
459.939530361104	1.03754907318427\\
461.601398155598	1.04789926905328\\
462.628368841378	1.07035288015526\\
464.751084395193	1.07905331606606\\
465.950664051001	1.10116710825682\\
467.901314519773	1.10820366745934\\
469.849957066339	1.12994835305217\\
471.411251452098	1.13531208485832\\
473.620879460557	1.15666071439685\\
475.687040171909	1.16034816244959\\
477.065457109452	1.18127625417814\\
479.051658714649	1.18328906992866\\
480.517360483271	1.20377437584122\\
482.649016358718	1.20411968343353\\
484.245828085425	1.22414240540133\\
485.647910480106	1.22283245085376\\
488.124035824717	1.242375146325\\
490.426758242457	1.23942731322014\\
492.265759313611	1.25847489625213\\
494.146551909147	1.25391162485759\\
495.482494185303	1.2724514172895\\
497.255870730959	1.26630010476348\\
498.619767933195	1.28432166463718\\
500.665963715859	1.27661431632458\\
501.85473974102	1.29410958756944\\
503.982596576292	1.28488291405639\\
505.582157468445	1.30184598675366\\
507.142501891651	1.29114102263019\\
508.853963128733	1.30756833310386\\
510.418229388708	1.29543019654174\\
512.580700618015	1.3113202188821\\
514.108220073299	1.29779808629203\\
516.078411089186	1.31315141211357\\
517.712549777485	1.29829794778761\\
519.284887960329	1.31311725267603\\
520.666481542535	1.29698856218703\\
522.242731967893	1.31127845217527\\
523.523891422725	1.29393372260162\\
525.734608550512	1.30770069235565\\
527.473579425089	1.28920193934048\\
529.122798324745	1.3024542695374\\
530.877934296569	1.28286600412618\\
532.572849787339	1.29561371576906\\
534.311798302632	1.27500257236921\\
535.778946500501	1.28725738072996\\
537.845947281639	1.26569192794379\\
539.455982245865	1.2774668924503\\
542.065762698457	1.25501722238249\\
543.480173460929	1.26632699783239\\
545.032654790488	1.24306439290245\\
546.967703762425	1.25392506157405\\
548.425371282521	1.22992170135979\\
549.525678238886	1.24035047856586\\
550.75614183404	1.21567915093145\\
552.533388278296	1.22569433289227\\
554.428823846268	1.20042811585068\\
556.149906536437	1.2100491542633\\
557.966272873631	1.18426103761638\\
559.67371556122	1.19350836938425\\
561.165680397665	1.16727102740438\\
563.173231847011	1.17616590333003\\
564.527879812467	1.14955137353269\\
566.52381621246	1.15811581265854\\
568.313660183313	1.13119522289475\\
570.041785100042	1.13945184313094\\
571.317564358497	1.11229516915124\\
573.445957252333	1.12026727213411\\
574.547315022418	1.09294305979782\\
576.311223964864	1.10065433921533\\
577.728637759505	1.07322946083219\\
579.616533042741	1.08070396443671\\
581.47550355483	1.05324344077673\\
583.088779497685	1.06050540756208\\
584.55360244943	1.03307218151287\\
586.347788488492	1.04014626585269\\
587.776432247726	1.01280101162952\\
589.284885190812	1.01971174128713\\
590.660425347305	0.992512708286544\\
592.093090632371	0.999284842395501\\
593.218157135237	0.972287724983977\\
595.420023382589	0.978945843544753\\
597.208059776025	0.95220363300681\\
599.185801518584	0.958772243491392\\
600.354174161499	0.932335137712739\\
601.573236031356	0.93883839558327\\
603.831834726688	0.91275371352184\\
605.192844555398	0.91921566483845\\
606.856964329132	0.893527764936527\\
608.370967757956	0.899972071331016\\
610.231747856575	0.874722326514306\\
611.607686395364	0.881172120148936\\
612.84745783228	0.856398876850492\\
614.150420868885	0.862877011400951\\
615.875377725149	0.838615578401941\\
617.719468897857	0.845144093109111\\
619.596393669158	0.821426760022117\\
621.25378849167	0.828027397984372\\
622.551831007716	0.804883437466708\\
624.828493336369	0.811576977699451\\
626.434426064603	0.789032733175446\\
627.678003560366	0.795839513167994\\
629.450685124694	0.773918416821037\\
631.116593405948	0.780857929258631\\
632.817506497802	0.759580573408108\\
634.768619826802	0.766671459398733\\
636.86955904273	0.746055703537497\\
638.196420697401	0.753315880864182\\
640.203713247448	0.733376890226009\\
641.644723870929	0.740823279796907\\
643.392518287241	0.721573724806616\\
644.982637726876	0.729222378192232\\
646.762236983677	0.710672436731731\\
647.926888467483	0.718538394600518\\
649.972110480025	0.700695997182934\\
651.291466712545	0.708793280897141\\
653.207574128478	0.691664111178604\\
655.150676287703	0.700005850155982\\
656.469042344724	0.683593568810764\\
658.560461189559	0.692191581831555\\
660.254798499005	0.676497927821633\\
662.209735197251	0.68536311958765\\
663.371090803602	0.670388036070865\\
664.814018349023	0.679529973321031\\
665.967712746708	0.665271754849823\\
667.914342356579	0.67469906270592\\
669.638687832761	0.661154520657432\\
671.136924071711	0.670874313951133\\
673.195856261418	0.658038868955959\\
674.807120521801	0.668057284829644\\
676.377241527381	0.655925165708255\\
678.220953091635	0.666246927182713\\
679.853076799154	0.654811228775331\\
681.538023943234	0.665439899234532\\
683.349803062579	0.654692744010976\\
684.781848283845	0.665630422124551\\
687.206841103597	0.655563025790529\\
689.023931385802	0.66681045649304\\
690.863425968482	0.657413334091487\\
692.221518975201	0.668969878686383\\
694.472430199762	0.660232910119292\\
696.566400387116	0.672096514413515\\
698.775406033534	0.664009042600004\\
700.463443638411	0.676176346468953\\
701.899910495311	0.66872722309876\\
703.831787023363	0.681193384042572\\
705.234120362922	0.674371046176639\\
706.930512997274	0.687129702166481\\
708.723662725687	0.680922398381728\\
710.313161609625	0.693965629143669\\
712.047715382622	0.68836126587124\\
713.733218548893	0.701679844506717\\
714.738570329835	0.69666614789934\\
716.515610452595	0.710249066580123\\
718.170156477922	0.705813688436582\\
719.613710327714	0.719648701710316\\
720.897729534621	0.715779074380632\\
722.862253605027	0.729852222704293\\
723.890654762417	0.726535761970286\\
725.392295154868	0.740831476974009\\
726.398836703662	0.738055494202443\\
728.740830443362	0.752556857073103\\
729.733676641037	0.750308619705223\\
731.608619622414	0.764997178196356\\
732.955778143905	0.763263841765969\\
735.012846149548	0.778119574179506\\
736.723337074279	0.77688826188885\\
738.732764195993	0.791889541771656\\
740.852001295464	0.791147258264477\\
743.248140691701	0.80627112423851\\
744.545652455994	0.806004778942659\\
746.376768857995	0.82122675602657\\
748.002934645973	0.821423140102315\\
750.007935883421	0.836717365368976\\
751.544981743316	0.837363083427224\\
753.349902753868	0.852702129409924\\
755.621133800835	0.853783603934577\\
756.885706022683	0.869138764644422\\
758.447018325505	0.870642197039533\\
760.214375281391	0.885983473673523\\
762.259604603309	0.887894801931118\\
764.04399407444	0.903191047513054\\
765.618615491886	0.905495862033963\\
767.323188480683	0.920714462097198\\
769.228397407709	0.92339803770258\\
770.734438634945	0.938505531801607\\
772.137915612775	0.941552781175989\\
773.653278127962	0.956514553945897\\
775.415042304506	0.959909941685044\\
777.59337362714	0.974690227082608\\
778.957602839616	0.978417763833171\\
780.862601780117	0.992980131929617\\
782.900958255133	0.997023411480402\\
784.468731684246	1.01133046902844\\
785.924574419951	1.01567250877006\\
787.09281735651	1.02968607836928\\
788.906539781647	1.03430944349443\\
790.383680506143	1.04799107326974\\
792.160680006678	1.05287772473448\\
793.369663703248	1.06618807759384\\
794.930945836959	1.07131946685748\\
796.429571856222	1.08421900046047\\
797.426362960926	1.0895760086929\\
799.116360125187	1.1020250198188\\
800.935545182093	1.10758798445609\\
802.467879900243	1.11954649261487\\
804.423753986967	1.12529519986909\\
806.141622423016	1.13672315929839\\
808.091526184255	1.14263688516551\\
809.7106985893	1.15349440289617\\
811.146617363012	1.15955192031329\\
812.568181665847	1.16979947972361\\
814.057370068732	1.17597907275997\\
815.560571295907	1.1855775964115\\
817.164744095024	1.19185712429064\\
818.876947667804	1.20076787262737\\
820.385770667916	1.20712484248671\\
822.558435082108	1.21530981089866\\
824.2306565436	1.22172137280575\\
825.795036326706	1.22914336129373\\
827.162040865232	1.23558639570804\\
828.858243867182	1.24220924626869\\
830.048436536711	1.24866041616829\\
831.42283831619	1.25444908518633\\
832.949650140657	1.26088489479923\\
834.108403090651	1.2658056926729\\
836.439418520204	1.27220264107797\\
837.965402139203	1.27622297003658\\
839.395173004997	1.28255753563423\\
841.387336109925	1.28564663063179\\
842.764010215178	1.29189539241076\\
844.764507927104	1.29402405766856\\
846.529275712569	1.30016380052887\\
847.854753427698	1.30130464594488\\
849.535102305974	1.3073124260127\\
851.626988343721	1.30743993393209\\
853.320862640916	1.31329319263405\\
854.977610890775	1.31238394957212\\
856.685522976601	1.31806060895342\\
857.807380630128	1.31609336612276\\
859.624973677597	1.32157190397974\\
860.896219665271	1.31852780137166\\
862.619618308335	1.32378735098463\\
864.082039863752	1.31964963411617\\
865.742163860738	1.32467006239214\\
866.796602783783	1.31942479930326\\
868.244415869961	1.32418685504031\\
869.541464611766	1.31782230794304\\
871.306160416022	1.32230765476742\\
873.000899620136	1.3148149845503\\
875.364921282877	1.31900627874099\\
876.776422320639	1.31037923920811\\
878.893744531925	1.31426025313616\\
880.573182362896	1.30449537996697\\
881.540031178146	1.30805105377656\\
883.505009080655	1.29714767914273\\
885.127572675245	1.30036416946774\\
887.041810114692	1.2883242647264\\
889.397965560604	1.29118903742057\\
890.832595383071	1.27801770869875\\
892.310602380179	1.28051958103102\\
894.567586611889	1.26622456554779\\
896.432693731477	1.26835372046769\\
898.571379497927	1.25294570041462\\
899.885473204052	1.25469381512975\\
901.656138735825	1.23818639171111\\
903.844677284029	1.23954654795211\\
905.252059439457	1.22195604563144\\
906.902153766643	1.22292284866694\\
908.046319036312	1.20426843712963\\
910.055698145024	1.2048379632869\\
911.967070421125	1.18514160682457\\
913.211497411768	1.18531144093573\\
914.769228384958	1.16459752191324\\
916.347587325575	1.16436671803463\\
917.629203190957	1.14266245285419\\
919.608605942752	1.14203153745014\\
921.193615279341	1.11936658345882\\
923.165909876474	1.11833755344482\\
924.787514143047	1.09474393136803\\
926.341876191712	1.09332015187772\\
928.217714575978	1.06883247029578\\
930.023563160828	1.06701866318113\\
931.746203393247	1.04167340143024\\
933.118954497669	1.03947560547046\\
935.116155837726	1.01331148301168\\
937.012667189721	1.01073693893617\\
938.506681025394	0.983794653056332\\
940.085850476276	0.980851756456815\\
941.853240874287	0.953173720814704\\
943.744329485733	0.949871948181025\\
945.324813689946	0.921502478295354\\
947.112142219986	0.917852263471458\\
948.77030553747	0.888836859027789\\
950.289388600008	0.884849491788977\\
951.6610005005	0.855235311012486\\
953.353070532455	0.850922881220789\\
954.627920972232	0.820758168719156\\
956.293881134564	0.816133369533535\\
957.738082466492	0.785467333487055\\
959.609484210049	0.78054340551099\\
961.604820066795	0.749426091675268\\
963.217783682346	0.74421662842517\\
964.860568368233	0.712698930558107\\
966.341907134164	0.707217802174491\\
968.088791743824	0.675350909520184\\
970.129759596904	0.669612096380542\\
971.816908764218	0.637447784416688\\
973.217920656395	0.631465192321455\\
975.077291874571	0.599055276412672\\
976.595891922562	0.592842647804449\\
978.260773613078	0.5602387312743\\
979.61029241628	0.5538094251511\\
981.221669817	0.521063180972613\\
983.167136747962	0.514430050513348\\
984.834825221095	0.481592697465087\\
986.52933839505	0.474767951449909\\
988.178746725281	0.441890144251389\\
989.772799827071	0.434885105452254\\
991.003925560555	0.402016700891399\\
992.663793883986	0.394841664034775\\
994.594488902272	0.362031993629987\\
995.503274123027	0.354696092587798\\
997.719012680368	0.321993109517316\\
999.620239616474	0.3145040861843\\
1001.33243748626	0.281954898499478\\
1003.18390365568	0.274318972273542\\
1004.31867331587	0.241969203704494\\
1005.61906180075	0.234190825160477\\
1007.23552716084	0.202084725407965\\
1008.98943117607	0.194166559528483\\
1010.06362432877	0.162346816251404\\
1011.8874001804	0.154289441927293\\
1012.83481850733	0.122797175870127\\
1014.74016453362	0.114598976139029\\
1016.58308792221	0.0834735146831216\\
1018.01611702474	0.0751305849453079\\
1019.55339345667	0.044409311603073\\
1021.27140928774	0.0359151832995011\\
1023.14947306609	0.005633668016789\\
1024.82905993177	-0.00302057109201392\\
1026.4971307421	-0.0328288702197721\\
1027.80449582801	-0.0416548009695372\\
1029.55505379692	-0.0709584452317156\\
1031.23160517058	-0.0799703049644105\\
1033.12048725927	-0.1087399894723\\
1034.4736191343	-0.117954667952599\\
1036.38718026785	-0.14616331935597\\
1037.65021931005	-0.155600200129107\\
1039.36833837025	-0.183222914214682\\
1041.03715265465	-0.192903756927506\\
1042.79416705495	-0.21991787823741\\
1044.55407736502	-0.22986665414572\\
1046.39887695079	-0.256251793333302\\
1047.82893074571	-0.266494167562982\\
1049.86918906029	-0.29223216376431\\
1051.12159218728	-0.30279516073468\\
1052.62404783006	-0.327870061864471\\
1054.87616850125	-0.338781376838793\\
1056.58392042985	-0.363179498617411\\
1057.63510644717	-0.374466692012936\\
1059.14699601185	-0.398176554083999\\
1060.27007956494	-0.409866131128346\\
1061.8067478485	-0.432878390031117\\
1063.17128024184	-0.444994770178891\\
1064.76360488033	-0.467302229000331\\
1066.12864418177	-0.47986630631312\\
1068.29389284657	-0.501463851769582\\
1069.553003055	-0.514491632030282\\
1070.88657130944	-0.535376128655809\\
1072.56448472726	-0.548877151169915\\
1074.02139251448	-0.56904744723624\\
1075.49224268649	-0.583023137094641\\
1076.76108406344	-0.60247993794052\\
1078.80374894796	-0.616921779512997\\
1080.46015358044	-0.635667594179007\\
1082.26382624389	-0.650555632627224\\
1084.66295730727	-0.668594687273578\\
1085.78289018817	-0.683895904575706\\
1087.44285439514	-0.701234061243439\\
1088.16220813518	-0.716901141140591\\
1090.6260116908	-0.733545796529051\\
1092.6655515772	-0.749516100670712\\
1093.90023587071	-0.765476105793286\\
1095.71956338661	-0.781671272856445\\
1096.93772911656	-0.796956825918927\\
1098.50399105161	-0.81328300567438\\
1100.15883080625	-0.827905549194299\\
1101.65235144451	-0.844253867004614\\
1103.98831050625	-0.85822600496092\\
1105.29682291303	-0.874474199888544\\
1106.9198590396	-0.88780958762967\\
1108.83761280448	-0.903824042894252\\
1110.23103211201	-0.916537284217929\\
1111.96362440762	-0.932175797210429\\
1113.24477614507	-0.944282346223215\\
1115.65834721853	-0.959397428834301\\
1117.25611146753	-0.97091349146978\\
1119.00355754312	-0.98535607558906\\
1120.1571863918	-0.996298511502912\\
1121.78882411903	-1.00992200529068\\
1123.47759296818	-1.02030823979918\\
1124.667741475	-1.0329724242473\\
1126.37599821428	-1.04282033268667\\
1128.61691452441	-1.0543953902703\\
1130.301980648	-1.06372323185885\\
1131.51414065188	-1.07409324389451\\
1133.59814723479	-1.08291957305885\\
1135.39931343967	-1.09198579646253\\
1137.0429898841	-1.10032936231271\\
1138.3002698768	-1.10801270361713\\
1139.62277321556	-1.11589238729404\\
1140.72080290494	-1.12213537663004\\
1142.21356813927	-1.12957009523755\\
1143.86272231218	-1.13433818585404\\
1145.73304206601	-1.14134685258238\\
1147.17812856533	-1.14462901535351\\
1149.14302824397	-1.15123044742218\\
1150.72399814687	-1.15303920071824\\
1152.71755182729	-1.15925208147399\\
1154.85631810436	-1.15962296487124\\
1156.33896364395	-1.16546576271717\\
1157.98770621662	-1.16445639549505\\
1159.21980408082	-1.16994732552194\\
1160.79593244804	-1.16763608887623\\
1161.94458492129	-1.172793076573\\
1163.82164987551	-1.16927757622445\\
1165.6131560457	-1.17411819216835\\
1167.25676256678	-1.16951352525541\\
1169.08670544973	-1.17405497395514\\
1170.1624624345	-1.1684919884883\\
1172.54805319456	-1.17275106704247\\
1173.82104764209	-1.1663744960398\\
1175.73811441959	-1.17036757938413\\
1177.49331864205	-1.16333422913491\\
1179.20159728137	-1.16707721512044\\
1180.49118933136	-1.15955416328379\\
1181.96122454301	-1.16306251274869\\
1184.57096183286	-1.15522544956223\\
1186.19341227376	-1.15851410262086\\
1187.5445443569	-1.15054532564237\\
1189.19588084702	-1.15362874413533\\
1190.68816997797	-1.14571544897357\\
1191.88773807408	-1.14860761418563\\
1194.28336080295	-1.1409399163621\\
1195.91121937559	-1.14365427994343\\
1197.07280801459	-1.13642273600485\\
1199.238552422	-1.13897225679393\\
};
\addlegendentry{Asynchronous}

\addplot [color=mycolor2, dotted, line width=2.0pt]
  table[row sep=crcr]{%
7.14061008190073	286.616107393646\\
12.6848571634962	278.006101976776\\
18.0307734636552	274.726604695367\\
24.4054243914712	266.368931408117\\
31.5017732410358	263.192107813357\\
37.7051848543163	255.083266355201\\
44.4080770345563	252.007319019199\\
50.265068990787	244.143675439905\\
56.8135056308066	241.166754363877\\
62.3192991327582	233.544576423928\\
72.2966980452018	230.664784933246\\
79.0925286172847	223.280249683358\\
83.7925524968093	220.495635456359\\
93.1082312134061	213.344789171768\\
99.7305849904921	210.653408969197\\
108.033006934577	203.732249491361\\
120.928373064351	201.132085314839\\
127.465304911415	194.436498393473\\
135.835788724607	191.925522563333\\
142.769534541299	185.451338418935\\
149.825515351352	183.027530360156\\
158.796667871878	176.770508758382\\
168.664257171709	174.431804448145\\
180.367976667651	168.387686104072\\
186.233012480013	166.132030654014\\
192.975558762159	160.296458549043\\
198.904159422809	158.121795986692\\
204.2394959771	152.490429772675\\
209.192876067321	150.394706719604\\
221.518875252007	144.963164797165\\
228.502438004184	142.944339913979\\
236.214268455649	137.708226894953\\
248.819333399688	135.764279280274\\
255.016711889701	130.719191637186\\
261.542024930301	128.848117014674\\
267.51728645521	123.989648562583\\
275.984566096475	122.189461733171\\
280.768174210901	117.513262020846\\
289.463421743194	115.782020446311\\
295.975574954625	111.283706748736\\
303.315844316451	109.61949441677\\
311.441693333402	105.294762075424\\
322.633751627178	103.695711090078\\
328.821733683747	99.5402966193296\\
335.501434596921	98.0045809328897\\
342.390841945341	94.0142624217899\\
348.258288331614	92.5401113989994\\
353.922012186231	88.7107397275715\\
363.460937800283	87.2964354287296\\
375.563455499297	83.6239302727699\\
381.524847946368	82.2678092861388\\
389.011696790392	78.7481843629284\\
394.79448812566	77.4486345788011\\
399.8945368127	74.0779403551501\\
407.191552262441	72.8334062445979\\
415.959670760656	69.6077776448212\\
421.895364321245	68.4167437222852\\
430.86678652202	65.3323661773386\\
435.907470831965	64.1933547620586\\
441.495687905863	61.246467684367\\
447.817611581659	60.1580367371573\\
453.661620876027	57.3449235180817\\
458.103813574548	56.3056571803475\\
462.975201946659	53.622642644591\\
468.211779514721	52.6311575920381\\
474.739012455275	50.074611476178\\
480.32217873462	49.1295480677966\\
488.395810828891	46.6958875381008\\
495.609206817571	45.7959157644314\\
504.289007471423	43.4816070433674\\
509.544751952297	42.6254214700648\\
515.259092692097	40.4269701217308\\
521.403337128824	39.6132886945861\\
531.817939525403	37.5272618299732\\
547.818507152454	36.754830657851\\
556.689190234142	34.7778348726299\\
562.006998359387	34.0454188607171\\
568.341346680181	32.1741149693276\\
575.977792953409	31.480505297976\\
581.41050863403	29.7116057000947\\
589.178092336646	29.0556148359249\\
595.826552771616	27.3858761421255\\
608.506607755043	26.7663361353564\\
613.943442690396	25.1925651531922\\
621.149060943369	24.6083284317979\\
628.058114289096	23.1273738370294\\
633.406904778525	22.5773111514957\\
641.881975290212	21.1860522282179\\
650.048372542495	20.669050713027\\
657.499609170175	19.3644029509071\\
663.726412538486	18.8793653521396\\
670.772217423205	17.658255664886\\
676.552802900965	17.2041003827173\\
683.361298561725	16.0634719965897\\
689.604540252684	15.639128880714\\
699.840175274928	14.5759257450577\\
714.906560844281	14.1803380330498\\
720.29136828543	13.1914905820418\\
726.583319638025	12.8236098807361\\
735.4128890717	11.9060242990757\\
755.757711423564	11.5648106335767\\
762.938976469407	10.7153616672491\\
767.38532888311	10.3997816325256\\
773.007172408315	9.61530062932417\\
779.104112600577	9.32432503334007\\
784.297793214053	8.60159483745252\\
790.003612277628	8.334197489161\\
797.495162234299	7.66995323549423\\
803.773777849927	7.42510806102774\\
809.357040605723	6.81603566758213\\
815.594482344829	6.59271771466251\\
822.516878914754	6.03546737718739\\
826.905731628243	5.83264916111995\\
832.391725944011	5.32384649128772\\
838.316374593998	5.14049800232334\\
844.599049160524	4.67676688811827\\
852.512647598118	4.51185595669291\\
861.72352459246	4.08984415458561\\
869.483658383757	3.94233598369711\\
875.846631477503	3.55874454969889\\
882.719344440525	3.42760359181585\\
890.812747882282	3.07921695030439\\
900.571875112112	2.96340805895427\\
909.180613328836	2.64712498128826\\
914.216206753616	2.54561694715245\\
920.882140597372	2.2584783674359\\
925.705610920277	2.17024737080756\\
935.975866521522	1.90946071932895\\
941.841438714317	1.83349489875464\\
950.830753795476	1.59645259176951\\
956.72171115594	1.53175743353643\\
966.715464774879	1.31604812020652\\
975.122239782859	1.26165258623179\\
983.219957201363	1.06507110872511\\
993.818125052856	1.0200335456166\\
1001.77763585806	0.840577437317762\\
1007.74523858605	0.803992139161451\\
1021.32501149732	0.639860008815063\\
1028.8999638401	0.610863316774982\\
1034.84531936888	0.460445576983787\\
1041.17852485673	0.438221479307984\\
1048.37476674776	0.300088433219599\\
1052.94580898063	0.283873503933435\\
1059.32739663765	0.156762453205559\\
1065.24112186868	0.145849495181125\\
1072.58592984391	0.0286499248968681\\
1083.86976285657	0.0223908139745532\\
1089.38308729285	-0.0858708604799716\\
1094.3055830589	-0.0880638433385315\\
1104.03895065304	-1.18823814219114\\
1109.91965955033	-1.18689215043104\\
1117.73184250995	-1.2797212916639\\
1124.82178338123	-1.27530397396345\\
1135.69136163729	-1.36143522716014\\
1141.7844622716	-1.35435685198983\\
1148.67612278659	-1.43435397124753\\
1153.9818494946	-1.42497042211674\\
1162.96409736773	-1.49932470087458\\
1170.62284586049	-1.48794123850743\\
1176.1378065614	-1.55708147361051\\
1190.68172194382	-1.54395693079547\\
1195.99857677714	-1.60825781966829\\
};
\addlegendentry{Synchronous}

\end{axis}
\end{tikzpicture}%